\documentclass{article}
\pdfoutput=1

\usepackage{Definitions}

\usepackage[utf8]{inputenc} % allow utf-8 input
\usepackage[T1]{fontenc}    % use 8-bit T1 fonts
\usepackage{hyperref}       % hyperlinks
\usepackage{url}            % simple URL typesetting
\usepackage{booktabs}       % professional-quality tables
\usepackage{amsfonts}       % blackboard math symbols
\usepackage{nicefrac}       % compact symbols for 1/2, etc.
\usepackage{microtype}      % microtypography
\usepackage{xcolor}         % colors
\usepackage{algorithm}
\usepackage{algorithmic}
\usepackage{graphicx}
\usepackage{enumerate}
\usepackage{subfig}
\usepackage{comment}
\usepackage[export]{adjustbox}

\usepackage{authblk}
\usepackage{fullpage}

\newcommand{\given}{\,|\,}

\newcommand{\Rbf}{\text{Range($f_{\Bcal}$)}}
\newcommand{\Rf}{\text{Range($f$)}}
\newcommand{\xhdr}[1]{\vspace{1mm} \noindent {\bf #1}}

\newcommand{\R}[2]{\cbr{#1, \ldots, #2}}

\title{On the Within-Group Fairness of Screening Classifiers}

\author{Nastaran Okati}
\author{Stratis Tsirtsis}
\author{Manuel Gomez Rodriguez}

\affil{Max Planck Institute for Software Systems \\ \{nastaran, stsirtsis, manuel\}@mpi-sws.org}

\date{}

\begin{document}

\maketitle

\begin{abstract}
Screening classifiers are increasingly used to identify qualified candidates in a variety of selection processes.
In this context, it has been recently shown that if a classifier is calibrated, one can identify the smallest set of candidates which contains, in expectation, a desired number of qua\-li\-fied candidates using a threshold decision rule.
This lends support to focusing on calibration as the only requirement for screening classifiers.
In this paper, we argue that screening policies that use calibrated classifiers may suffer from an understudied type of within-group unfairness---they may unfairly treat qualified members \emph{within} demographic groups of interest.
Further, we argue that this type of unfairness can be avoided if classifiers satisfy within-group monotonicity, a natural monotonicity property within each group.
Then, we introduce an efficient post-processing algorithm based on dynamic programming to minimally modify a given calibrated classifier so that its probability estimates satisfy within-group monotonicity. 
We validate our algorithm using US Census survey data and show that within-group monotonicity can often be achieved at a small cost in terms of prediction granularity and shortlist size. 
\end{abstract}

\section{Introduction}
\label{sec:introduction}
%screening
As many selection processes receive hundreds or even thousands of applications, it has 
become increasingly common to rely on automated screening tools to shortlist a tractable 
set of promising candidates.
These shortlisted candidates then move forward in the selection process and are evaluated 
in detail, possibly multiple times, until one or more qualified candidates are selected.
%
%examples of screening
The benefits and harms posed by auto\-ma\-ted screening have been investigated in many 
high-stakes domains, including medicine~\cite{Etzioni2003,shen2019}, recruiting~\cite{cowgill2018bias,raghavan2020mitigating} and content moderation~\cite{gorwa2020algorithmic}.
%
%screening classifiers
In the machine learning literature, algorithmic screening has been studied together with other 
high-stakes decision making problems as a supervised learning problem~\cite{corbett2017algorithmic, kilbertus2020fair, sahoo2021reliable}.
%
%   threshold rule
Under this view, algorithmic screening consists of designing both a screening classifier, which 
estimates the pro\-ba\-bi\-li\-ty that a candidate is qualified, and a screening policy, which shortlists 
candidates using the candidates'{} probability values estimated by the screening classifier.
%
% quality of shortlists and calibration
Only very recently, a line of work has focused specifically on algorithmic screening~\cite{Wang2022ImprovingSP, jin2022selection, wang2023fairness}. 
Therein,~\cite{Wang2022ImprovingSP} argue that, to increase the efficiency of the selection process without decreasing the quality of the shortlisted candidates, 
the focus should be on screening policies that find the smallest shortlist of candidates 
containing a desired average number of qualified candidates with high probability without making any distributional 
assumptions on the candidates.
Further, this work has shown that, if the screening classifier is calibrated~\cite{Dawid1982TheWB}, such distribution-free guarantees can be achieved using threshold decision rules as screening policies, and the more granular the predictions of the classifier, the smaller the shortlists provided by such policies.

% there may still be discrimination
In this work, our starting point is the realization that any threshold decision rule that uses calibrated 
screening classifiers may be biased against qualified candidates \emph{within} demographic groups of interest.
More specifically, it may shortlist one or more candidates from a group who are less likely to be qualified
than one or more rejected candidates from the same group.
%   why is it bad
Unfortunately, this type of within-group unfairness may perpetuate historical biases against minority groups 
since it may preclude the \emph{best} candidates from the groups---the candidates who are more likely to be 
qualified---to move forward in the selection process and have a chance to be selected~\cite{yang2019balanced}.

\xhdr{Our contributions.}
%within-group monotonicity
We first show that to avoid such within-group unfairness, screening classifiers need to satisfy a natural monotonicity property within each of the groups of interest, which we refer to as within-group monotonicity.
Then, we develop a set partitioning post-processing framework to minimally modify any calibrated classifier such that it
% its probability estimates 
satisfies within-group monotonicity. 
Along the way, we make the following contributions:
%
% \vspace{-3mm}
\begin{enumerate}
    \item[I.] We show that the problem is NP-hard using a re\-duc\-tion from a variation of the partition problem~\cite{karp1972reducibility}, which we refer to as the equal average partition problem and prove it is NP-complete. 
    However, we identify a natural class of partitions---contiguous partitions---under which the problem is tractable.
    %
    %\vspace{-1mm}
    \item[II.] While the structure of our problem for contiguous partitions resembles isotonic regression~\cite{barlow1972isotonic}, we show that the classical Pool Adjacent Violators (PAV) algorithm may fail even to find a locally optimal solution. 
    %\vspace{-1mm}
    \item[III.] We derive a dynamic programming algorithm for contiguous partitions that is guaranteed to find an optimal solution to our problem in polynomial time.
    %\vspace{-1mm}
    \item[IV.] We show that within-group calibration~\cite{pleiss2017fairness} 
    implies within-group monotonicity.
    However, we show that it is often impossible to modify a classifier to satisfy the former and, whenever possible, the predictions of the resulting classifier are coarse.
\end{enumerate}
%\vspace{-3mm}
%
Finally, we create multiple instances of a simulated scree\-ning process using US Census survey data to validate and complement our methodological contributions and theo\-re\-ti\-cal results. 
The results show that the probability that an individual from a minority group suffers from within-group unfairness may be significant and within-group monotonicity can be achieved at a small cost in terms of prediction granularity and shortlist size.

\xhdr{Related work.} 
%fairness in automated decision
There is an extensive and rapidly growing line of work addressing group bias and discrimination in the machine 
learning literature~\cite{hardt2016equality,friedler2016possibility,zafar2017fairness,kim2019multiaccuracy,beutel2019putting,lahoti2020fairness}. 
This line of work has applications in a variety of important domains, 
including ranking~\cite{celis2017ranking,yang2017measuring,biega2018equity,singh2018fairness,sing2019policy},
health care~\cite{garb1997race,williams2009discrimination}, 
criminal justice~\cite{dieterich2016compas,Flores2016,angwin2016machine,feller2016computer,chouldechova2017fair,dressel2018accuracy}
and recommender systems~\cite{sweeney2013discrimination,datta2014automated,beutel2019fairness,wang2021practical,prost2022simpson}.
However, it has predominantly focused on preventing discrimination \emph{across} groups of interest, \eg, designing machine learning models whose predictive performance (\eg, accuracy, false positive 
rate) is invariant across groups. 
In contrast, we focus on preventing unfairness \emph{within} groups.

% within-group calibration
Within the above machine learning literature, there are a few notable exceptions~\cite{zehlike2017fa,speicher2018unified, yang2019balanced,garcia2021maxmin,zehlike2022fair}, 
which stu\-died similar notions to within-group mo\-no\-to\-ni\-ci\-ty (in the context
of ranking) and within-group unfairness.
Among them, the works by~\cite{zehlike2017fa, zehlike2022fair} and~\cite{speicher2018unified} are the most related to ours.
~\cite{zehlike2017fa, zehlike2022fair} introduces a notion of in-group monotonicity
that is similar to ours.
However, it comprises only the top-$k$ 
% the notion of in-group monotonicity
ranked candidates in a specific pool of candidates (\ie, in our work, the shortlisted candidates), rather than every candidate in a population of
interest, and unconditional qua\-li\-ty scores, rather than group conditional 
quality scores.
Moreover, their formulation is fundamentally different and their technical contributions are orthogonal to ours.
\cite{speicher2018unified} addresses within-group unfairness as a measure of how unequally members within a group benefit from algorithmic decisions. In contrast, our notion of within-group monotonicity asks for accurately ranking individuals belonging to a group in terms of how worthy they are of receiving a beneficial decision rather than equally benefiting them.
In this context, it is also worth high\-lighting the notion of within-group calibration~\cite{pleiss2017fairness,kleinberg2018inherent}, 
which implies within-group monotonicity, as discussed previously. 
Within-group calibration asks for equally well-calibrated probability estimates \emph{across} groups so that a decision maker cannot use group membership to interpret these estimates.
However, in the context of screening, our results show that within-group calibration may be an unne\-cessa\-ri\-ly strong requirement.
% improve calibrated classifiers in terms of ECE
Our work also relates to a line of work devoted to the study of calibration in supervised learning~\cite{zadrozny2001obtaining,Zadrozny2002TransformingCS,guo2017calibration,kumar2018trainable,Krishnan2020ImprovingMC,Karandikar2021SoftCO}.
%   minimize ECE
Here, the main focus has been the design of classifiers with low ca\-li\-bra\-tion error using calibration-aware training or post-hoc re-calibration. % techniques. 
However, there have been also very recent efforts to ensure calibration errors are bias-free~\cite{Broecker11,Ferro2012ABD,roelofs2022mitigating}.
%
%   debiasing them
Here, we do not aim to minimize calibration error but ensure a calibrated classifier satisfies within-group monotonicity.

\section{Screening, Calibration and Within-Group Fairness}
\label{sec:formulation}
Given a candidate with a feature vector $x \in \Xcal$, we assume the candidate belongs to one demographic group of interest $z \in \Zcal$ and can be qualified ($y = 1$) or un\-qua\-li\-fied ($y = 0$) for the selection objective\footnote{\scriptsize We do not require a candidate'{}s group membership $z$ to be included in or be inferable from their feature vector $x$.}\textsuperscript{,}\footnote{\scriptsize In practice, one measures qualification using proxy variables, which need to be chosen carefully not to perpetuate historical biases~\cite{bogen2018help,garr2019diversity,tambe2019artificial}.}.
Next, let $f : \Xcal\rightarrow \Rf \subseteq [0, 1]$ be a screening classifier that maps a candidate'{}s feature vector $x \in \Xcal$ to a quality score $f(x)$, where the higher the quality score $f(x)$, the more the classifier believes the candidate is qualified. 
%
% Here, note that the classifier $f$ does not have access to 
% a candidate'{}s group membership $z$ but only their feature
% vector $x$, however, that does not mean that 
% $f$ is incapable of exhibiting bias across 
% groups~\cite{kleinberg_et_al:LIPIcs:2017:8156}.
%
Then, given a pool of $m$ candidates, a screening policy 
$\pi \,:\, [0, 1]^{m} \rightarrow \Pcal(\{0, 1\}^{m})$ maps the candidates'{} quality scores 
% and group membership $\{(f(x_i), z_i)\}_{i \in [m]}$ 
to a pro\-ba\-bi\-li\-ty distribution over shortlisting decisions $\{s_i\}_{i \in [m]}$.
Here, each decision $s_i$ specifies whether the corresponding candidate is shortlisted ($s_i = 1$) or is not shortlisted ($s_i = 0$).
%
% Moreover, note that, in contrast to the screening 
% classifier $f$, the screening policy $\pi$ does have access
% to the candidates'{} group memberships to account for 
% diversity requirements such as, \eg, Rooney'{}s 
% rule~\cite{collins2007tackling, kleinberg2018selection, 
% celis2021effect}, in the shortlisting decisions.

In high-stakes applications, screening classifiers $f$ are usually demanded to provide calibrated quality scores~\cite{brier1950verification, Gneiting2007, gupta2020distribution}, \ie, $f$ is calibrated iff, for every $a \in \Rf$, it holds that $\Pr(Y = 1 \given f(X) = a) = a$. 
%
% \begin{definition}
%     A classifier $f$ is calibrated iff, for every $a \in \Rf$, it holds that $\Pr(Y = 1 \given f(X) = a) = a$.
% \end{definition}
%
In this context, Wang et al.~\cite{Wang2022ImprovingSP} have recently shown that, if the classifier $f$ is calibrated, the optimal screening policy $\pi^{*}_f$ that is guaranteed to shortlist, in expectation, the smallest set of candidates with a desired number of qualified candidates with high probability is given by a simple threshold decision rule 
% that shortlists those candidates whose quality score fall above a $t_f$,
that take shortlisting decisions as
\begin{equation} \label{eq:decision-threshold-rule-algorithm}
s_i =
\begin{cases}
1 & \textnormal{if} \, f(x_i) > t_f, \\
\text{Bernoulli}(\theta_f) & \textnormal{if} \, f(x_i) = t_f \\
0 & \textnormal{otherwise},
\end{cases}
\end{equation}
where $t_f$ and $\theta_f$ depend on the classifier and data distribution.
% \footnote{\scriptsize Refer to (Theorem 4.1,~\citet{Wang2022ImprovingSP}) for the exact expression for the decision policy.}.
%
% Moreover, the same authors have also shown that, if $x$ 
% includes $z$, one can increase the diversity of the 
% shortlists by thresholding the quality scores $f(x)$ of 
% each candidate using a group-dependent threshold $t_{f, 
% z}$.
%
%, where $z \in \Zcal$ denotes each demographic group of 
% interest.
%
These results lend support to focusing on calibration as the only requirement for screening classifiers. 
% , especially since there is empirical evidence that well-performing classifiers 
% learned from data are (approximately) monotone, but not 
% calibrated~\cite{pmlr-v70-guo17a,pmlr-v80-kuleshov18a,Zadrozny2002TransformingCS,
% pmlr-v151-roelofs22a,chen2016,DBLP:conf/icml/GuptaR21,10.5555/1036843.1036912}. 
% Additionally, it has been proven impossible to find a nonatomic perfectly 
% calibrated classifier from 
% data~\cite{DBLP:conf/icml/GuptaR21,10.1214/20-EJS1749}. 
%
In this work, we argue that scree\-ning policies given by threshold decision rules using calibrated classifiers
% , even those with diversity guarantees, 
may suffer from an understudied type of unfairness---they may be biased against qualified members \emph{within} demographic groups.
More formally, the following proposition shows that any threshold decision rule may be biased against qualified members within demographic groups\footnote{\scriptsize All proofs can be found in the Appendix~\ref{app:proofs}.}:
\begin{proposition}\label{prop:within-group-discrimination}
Let $\pi$ be a screening policy given by a threshold decision rule using a calibrated classifier $f$ with threshold $t$.
Assume there exist $a, b \in \Rf$, with $a < t < b$, and $z \in \Zcal$ such 
that $P(Y = 1 \given f(X) = a, Z = z) > P(Y = 1 \given f(X) = b, Z = z)$.
Then, it holds that
\begin{align*}
    \EE_{Y \sim P_{Y \given X, Z}, \, S \sim \pi} \left[ Y (1-S) \given f(X) = a, Z = z \right]
    > \EE_{Y \sim P_{Y \given X, Z}, \, S \sim \pi} \left[ Y S \given f(X) = b, Z = z \right].
\end{align*}
\end{proposition}
The above result implies that there exist pools of applicants for which an optimal policy using a calibrated classifier may shortlist a candidate from a group who is less likely to be qualified than a rejected candidate from the same group.
Importantly, the assumption under which the above within-group unfairness appears is not just a theoretical construct---it has been observed empirically in multiple real-world domains
whenever the group membership $Z$ is a spurious confounding factor that causes both $X$ and $Y$~\cite{wagner1982simpson, pearl2000models}.
The case in which the assumption holds for \emph{every} group $z \in \Zcal$ and \emph{any} threshold decision rule is known as Simpson's paradox~\cite{blyth1972simpson}. Refer to Figure~\ref{fig:illustrative} for an illustrative example.

\begin{figure}
    \centering
    \subfloat[Overall]{\includegraphics[width=0.48\textwidth]{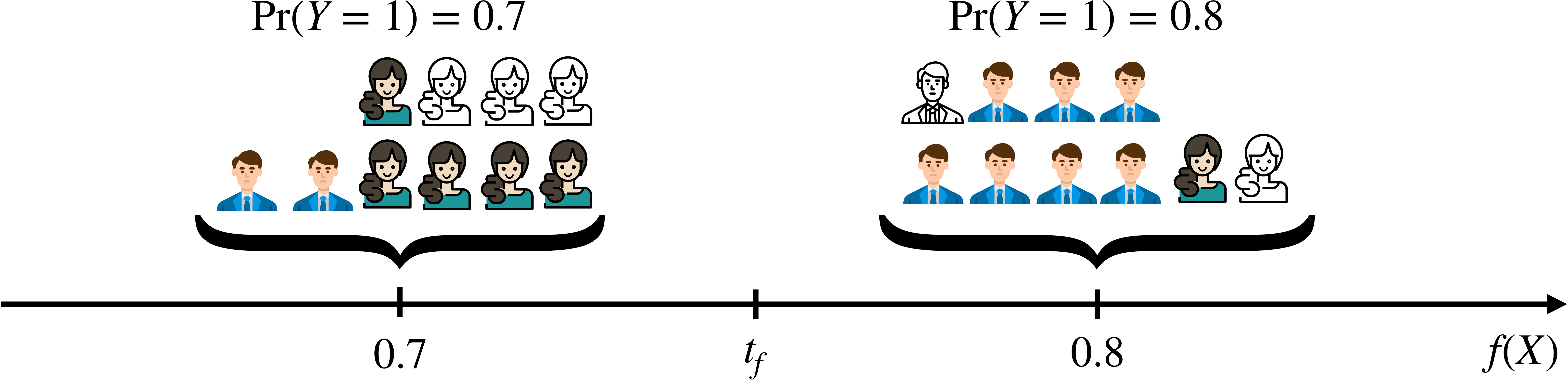}} \\
    \subfloat[Female]{\includegraphics[width=0.48\textwidth]{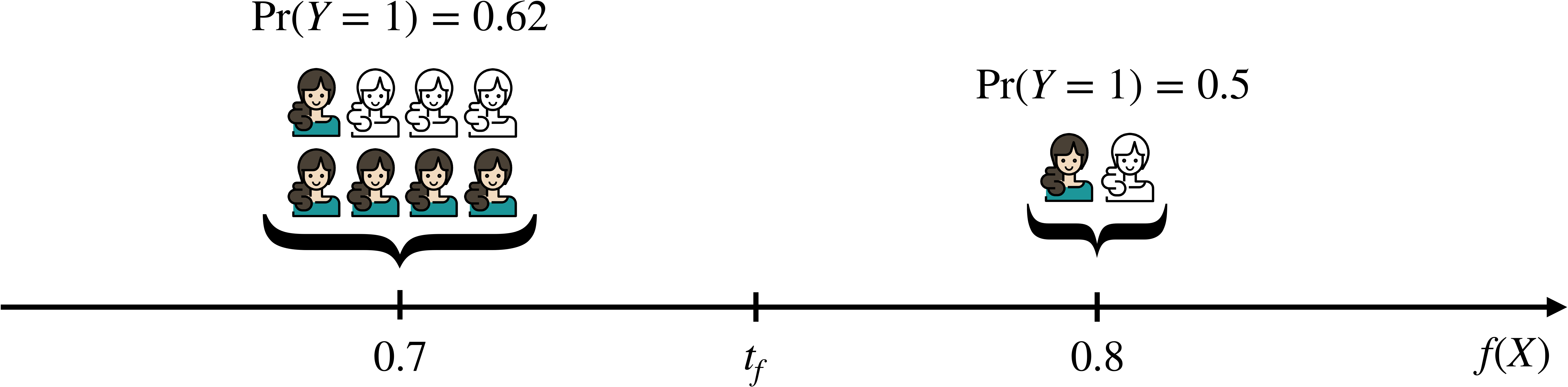}}
    \\
    \subfloat[Male]{\includegraphics[width=0.48\textwidth]{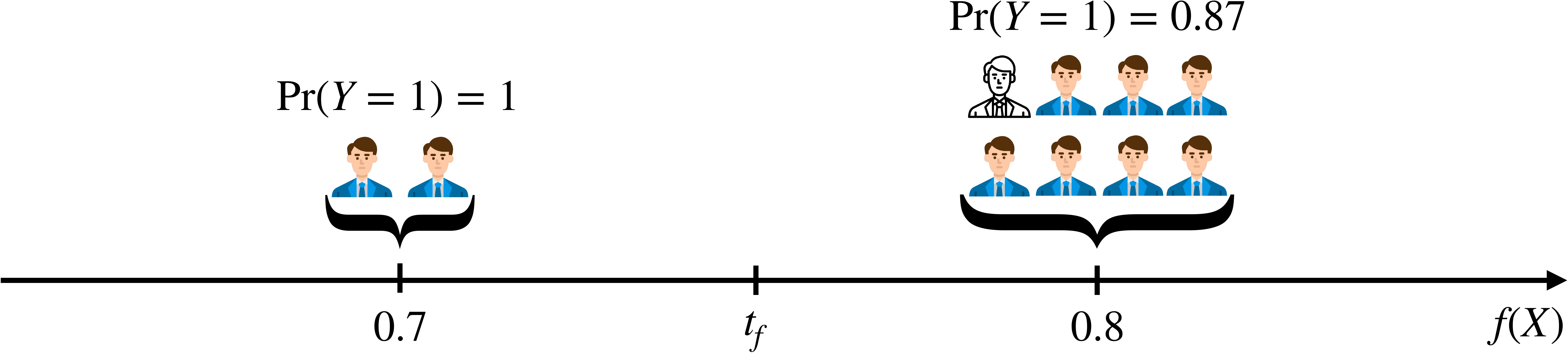}}
    %\vspace{-2mm}
    \caption{An illustrative example of within-group unfairness. Panel (a) shows that candidates who are shortlisted ($f(X) > t_f$) are more likely to be qualified ($Y = 1$) than those who are rejected ($f(X) < t_f$). 
    However, panels (b) and (c) show that, after conditioning on their gender, candidates who are rejected ($f(X) < t_f$) are more likely to be qualified than those who are short listed ($f(X) > t_f$).
    Qualified candidates are shown in color.
    }
    \label{fig:illustrative}
    %\vspace{-3mm}
\end{figure}
%
% In what follows, our goal is to design a post-processing
% framework that, given a monotone classifier, modifies it
% minimally so that, for each group $z \in \Zcal$, the modified classifier $f'$ is monotone with 
% respect to the data distribution 
% $P_{X,Y \given Z}$ and any threshold decision rule using $f'$ 
% cannot suffer from within-group discrimination.

To avoid the above within-group unfairness, we introduce and study within-group monotonicity:
%
% manuel: i removed the "which may be of independent interest" because we may like not to
% write that we are the "first" to talk about within-group monotonicity. In the related work, 
% we clarify the connection to FA*IR...
%
% , a property that may be of independent interest beyond screening classifiers:
%
\begin{definition}\label{def:within-group-monotonicity}
Given a set of groups $\Zcal$, a classifier $f$ is within-group monotone if, for any $z \in \Zcal$ and $a, b \in \Rf$ such that $a < b$, $\Pr(Z=z\given f(X) = a) > 0$ and $\Pr(Z=z\given f(X) = b) > 0$, it holds that
\begin{align*}
    \Pr\rbr{Y=1\given f(X) = a, Z=z} \leq \Pr\rbr{Y=1\given f(X) = b, Z=z}.
\end{align*}
\end{definition}
In what follows, we will design a post-processing framework that, given a calibrated classifier, modifies it minimally so that it is within-group monotone, as shown in Figure~\ref{fig:wgm_illustration}. As a result, any screening policy given by a threshold decision rule using the modified classifier will not suffer from within-group unfairness.
Here note that we favor a post-processing approach, rather than an in-processing one, because post-processing approaches can be applied to any black-box classifier without asking for retraining or introducing training overhead~\cite{hardt2016equality}. Furthermore, in-processing approaches commonly need access to the feature defining group membership to ensure group-level fairness~\cite{woodworth2017learning}, which may not be available to the classifier due to privacy, legal or regulatory reasons.
Whenever it is clear from the context, we do not specify the set of groups $\Zcal$ with respect to which a classifier is within-group calibrated or monotone.
%
%
% manuel: since we already mention this as contributions in the intro, 
% no need to include this here, there is an entire section discussing this.
%
% \xhdr{Remarks.} Within-group monotonicity relates to within-group calibration, a 
% well-known property that has been recently argued for in the fairness literature 
% in classification~\cite{kleinberg2018inherent, hebert2018multicalibration}. 
%
% More specifically, if a classifier satisfies within-group calibration, then it 
% also satisfies within-group monotonicity. 
%
% However, in Section~\ref{sec:group-calibration}, we argue that within-group 
% calibration is an unnecessarily strong requirement for screening classifiers or, 
% more broadly, classifiers used to perform subset selection.
%
%
\begin{figure}
    \centering
    \includegraphics[width=0.6\textwidth]{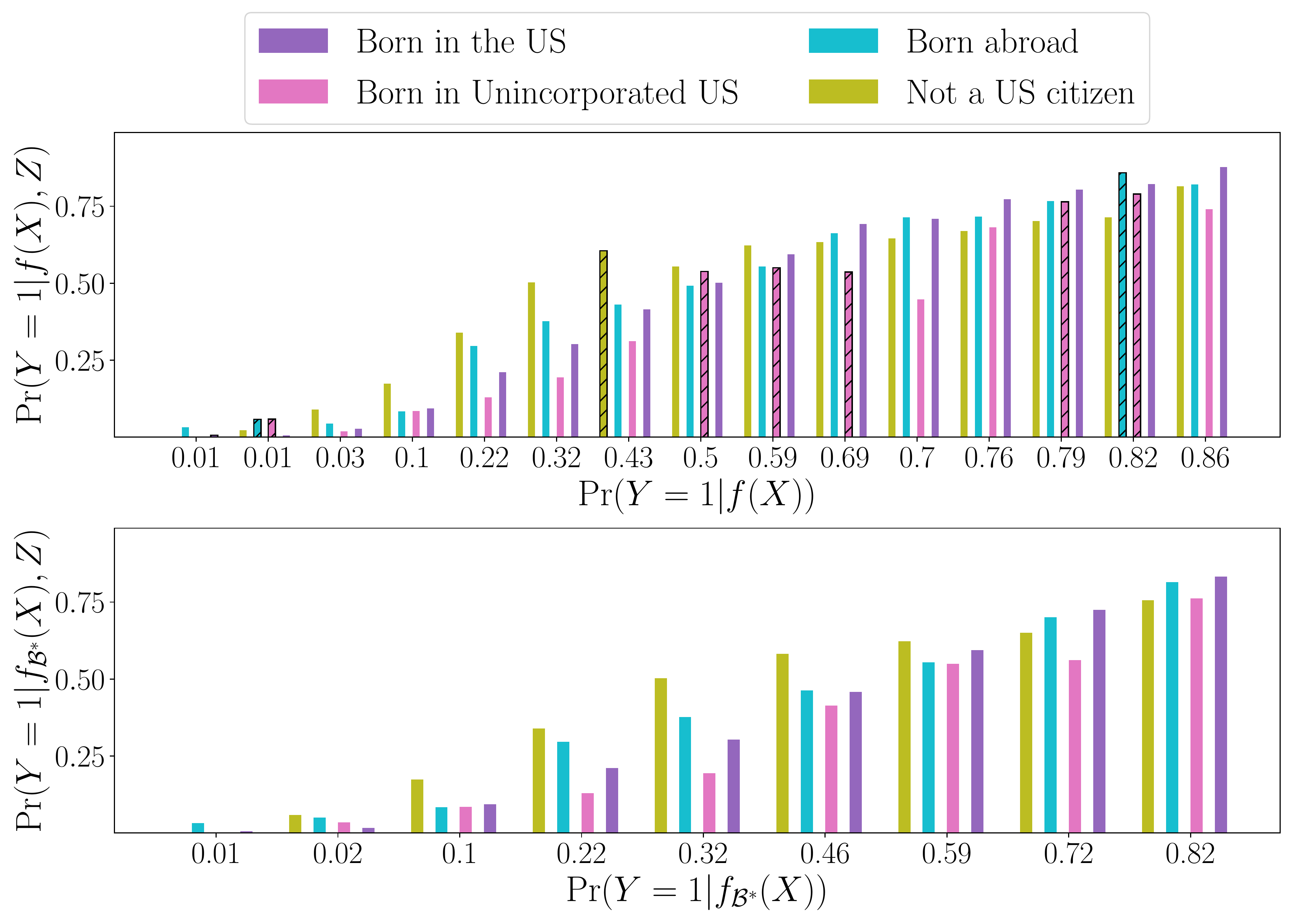}
    % \vspace{-2mm}
    \caption{Quality score values $a = P(Y = 1 \given f(X) = a)$ and group conditional quality score values $a_z = P(Y = 1 \given f(X) = a, Z = z)$ of a (approximately) calibrated screening classifier $f$ with finite range trained on US Census survey data and its within-group monotone counterpart $f_{\Bcal^{*}}$ found by our post-processing framework.
    The demographic groups of interest $\Zcal$ are defined using US citizen status and the hatched bars indicate within-group monotonicity violations. Note that there exist no such violations in $f_{\Bcal^{*}}$ (second row).
    }
    \label{fig:wgm_illustration}
    % \vspace{-3mm}
\end{figure}

\section{A Set Partitioning Post-Processing Framework}
\label{sec:framework}
Let $f$ be a calibrated classifier with $\Rf = \cbr{a_1, \ldots, a_n}$ and $\Pr\rbr{f(X) = a_i} = \rho_i$.
Here, note that we focus on calibrated classifiers with finite range, \ie, $\abr{\Rf} = n < \infty$, since it is \-impossible~to~find non-atomic calibrated classifiers from data\footnote{\scriptsize Given a (non-atomic) classifier $f$, there exists a variety of methods to discretize and calibrate its predictions~\cite{zadrozny2001obtaining,Zadrozny2002TransformingCS,gupta2020distribution}.
However, this is out of the scope of our work.
Moreover, for ease of exposition, we assume that $f$ is \emph{perfectly} calibrated and we have access to true value of the relevant probabilities $\rho_i$, $a_i$, $\rho_{z \given i}$ and $a_{i, z}$. 
However, our methodology can be adapted to work with approximately calibrated $f$ and noisy probability estimates as long as the estimation errors can be bounded (with high probability).
}, even asymptotically~\cite{10.5555/1036843.1036912,10.1214/20-EJS1749}.
Here, assume that $a_i < a_j$ for any $i < j$ without~loss of ge\-ne\-ra\-li\-ty.
Further, for every demographic group of interest $z \in \Zcal$, let $\Pr\rbr{Y = 1 \given f(X) = a_i, Z = z} = a_{i,z}$ and $\Pr\rbr{Z =z \given f(X) = a_i} = \rho_{z \given i}$, and note that, by definition, we have that $a_i = \sum_{z\in\Zcal}\rho_{z\given i}a_{i,z}$.
%
% Further, for every group of interest $z \in \Zcal$, let $\Pr\rbr{Y = 1 \given f(X) = a_i, Z = z} = a_{i,z}$ 
% and $\Pr\rbr{Z = z \given f(X) = a_i} = \rho_{z \given i}$, 
% and
% Without loss of generality, we assume $\rho_{z\given i}>0$ for all 
% $z\in\Zcal$ and $a_i\in\Rf$.
%
Then, our goal is to modify $f$ minimally so that it is within-group monotone.
% with respect to $P_{X, Y \given Z}$.

To this end, we first note that the classifier $f$ induces a partition of $\Xcal$ 
into $n$ disjoint regions or bins $\{\Xcal_1, \ldots, \Xcal_n\}$, 
where each bin $\Xcal_i$ is characterized by $a_i$ and $\rho_i$.
Building upon this observation, we look at the problem from the perspective of set 
partitioning and seek to \emph{merge} a small number of these induced bins to achieve within-group monotonicity. 
More formally, let $\Pscr$ be the set of all partitions of the bin indices $\{1, \ldots, n\}$. Every $\Bcal\in\Pscr$ is a partition of the bin indices into a collection of nonempty and disjoint equivalence classes $\{ \Acal_1, \ldots, \Acal_{|\Bcal|} \}$, which we call cells. For each $x \in \Xcal$, denote the index of the bin it belongs to as $i(x) = \{ i \given f(x) = a_i\}$
%
% Since $\Bcal$ is a partition, the equivalence classes are disjoint and they cover 
% $\{1, \ldots, n\}$. 
%
and represent a cell in $\Bcal$ containing index $i(x)$ by $\sbr{i(x)}_\Bcal$, where we drop the subscript $\Bcal$ whenever it is clear from the context. 
Further, we know that the equivalence relation $\sim_\Bcal$ implies that, for all $i(x') \in \sbr{i(x)}$, we have that $i(x) \sim_{\Bcal} i(x')$.
Then, we can use the partition\footnote{\scriptsize We use partition instead of partition on the bin indices whenever it is clear from the context.} $\Bcal$ to define the modified classifier $f_{\Bcal} : \Xcal \rightarrow \Rbf = \{a_{\Acal}\}_{\Acal \in \Bcal}$, where 
\begin{equation*} %\label{tilde_a}
    a_{\Acal} = \frac{\sum_{j \in \Acal} a_j\rho_j}{\sum_{j \in \Acal} \rho_j} \quad \text{and} \quad f_{\Bcal}(x) = a_{\sbr{i(x)}}.
\end{equation*}
%
% \begin{align}\label{tilde_rho}
%     \rho_{\Acal} = \Pr\rbr{f(X)\in  \Acal} = \sum_{a_j \in \Acal} \rho_j,
% \end{align}
% and the cell value $a_{\Acal}$ as
%
Without loss of generality, we keep the cells induced by the partition $\Bcal$ in increasing order with respect to $a_{\Acal}$, \ie, $a_{\Acal_i} \leq a_{\Acal_j}$ for any $i < j$.
Next, note that, by definition, $f_{\Bcal}$ is calibrated, \ie,
\begin{equation*}% \label{eq:cell_mean}
\Pr\rbr{Y = 1 \given f_{\Bcal}(X) = a_{\Acal}} = \frac{ \sum_{j \in \Acal} a_j \rho_{j}}{\sum_{j \in \Acal} \rho_{j}} = a_{\Acal},
\end{equation*}
and we have that
\begin{align*} % \label{eq:conditional-cell-mean}
    % \Pr\rbr{Z =z \given f_{\Bcal}(X) = a_{\Acal}} = \frac{\sum_{j \in \Acal} \rho_j\rho_{z\given j}}{\sum_{i\in\Acal}\rho_j} = \rho_{z\given \Acal} 
    % \quad \text{and} \quad
 \Pr\rbr{Y = 1 \given f_{\Bcal}(X) = a_{\Acal}, Z = z} = \frac{ \sum_{j \in \Acal} \rho_j \rho_{z \given j} a_{j,z}}{\sum_{j \in \Acal} \rho_j \rho_{z \given j}} 
 := a_{\Acal,z}. 
\end{align*}
Moreover, the larger the size of the partition $\Bcal$, the more fine-grained the predictions of the classifier $f_{\Bcal}$~\cite{Gneiting2007,Wang2022ImprovingSP}.
Therefore, we can naturally think of reducing the problem to finding a partition $\Bcal$ of maximum size such that $f_\Bcal$ is within-group monotone\footnote{\scriptsize Maximizing the size of the partition $|\Bcal|$ is equivalent to minimizing the distance $d(f, f_{\Bcal}) = n - |\Bcal|$. Thus, $f_{\Bcal^{*}}$ can be viewed as the \emph{closest} within-group monotone classifier $f_{\Bcal}$ under a prediction-only access model~\cite{blasiok2022unifying}.}, \ie,
\begin{align*} % \label{eq:set-partitioning-problem}
% \begin{split}
    \underset{\Bcal\in\Pscr}{\text{maximize}} \,\, |\Bcal| \quad \text{subject to} \quad a_{\Acal_i,z}\leq a_{\Acal_j,z}
    \forall \Acal_i, \Acal_j \in \Bcal 
    \text{ such that } a_{\Acal_i} < a_{\Acal_j}, \forall z \in \Zcal. 
% \end{split}
\end{align*}
%-------------------------
%
% not all within-group monotone classifiers are equally useful. As an example, consider a partition 
% $\Bcal = \cbr{\Acal}$ with a single cell $\Acal = \Rf$. By definition, the modified classifier 
% $f_{\Bcal}$ maps every feature vector to the same quality score $f(x) = \sum_{a_i \in \Rf} a_i\rho_i$. 
%
% While such a calibrated classifier is clearly within-group monotone, it is not an informative one as it does 
% not distinguish between different feature 
% vectors~\cite{pmlr-v80-kuleshov18a,Pakdaman_Naeini_Cooper_Hauskrecht_2015,Gneiting2007}. 
%
However, such a problem formulation presents difficulties both in terms of tractability and soundness. 
First, we cannot expect to find such a partition in polynomial time:
% , as formalized by the following Theorem:
%
\begin{theorem} \label{thm:np-hardness}
    Given a calibrated classifier $f$, the problem of finding the partition $\Bcal\in\Pscr$ of maximum size such that $f_{\Bcal}$ is within-group monotone is NP-hard.
\end{theorem}
To prove the above result in Appendix~\ref{app:np-hardness}, we first show 
that, by finding the partition $\Bcal$ of maximum size such that $f_{\Bcal}$ is within-group monotone, we can decide whether there exists a partition $\Bcal'$ of size $\abr{\Bcal'} = 2$ such that $f_{\Bcal'}$ is within-group monotone. 
Then, we show that the latter decision problem is NP-complete by a reduction from a variation of the partition problem~\cite{karp1972reducibility}, which we refer to as the equal average partition problem and prove it is NP-complete.

Second, even if the size of the partition $\Bcal$ is large, the shortlists provided by threshold decision rules using $f_{\Bcal}$ may differ greatly from those using $f$. 
The reason is that, in general, we may end up merging very different bins to ensure monotonicity within groups and, as a consequence, $f_{\Bcal}$ may rank (pairs of) candidates \emph{strictly differently}. 
More specifically, $f_{\Bcal}$ may \emph{not} satisfy the following monotonicity property with respect to $f$:
%
% To avoid this, we aimintroduce a notion of monotonicity between classifiers.
%
% However, in order to assure minimal changes to $f$, we wish to keep any ranking 
% of individuals provided by $f$ unswapped. 
% In other words, we want $f_\Bcal$ to be monotone with respect to $f$ as defined 
% formally below.
%
\begin{definition}
A classifier $f'$ is monotone with respect to $f$ if, for all $f(x_1), f(x_2) \in \Rf$ such that $f(x_1) < f(x_2)$, it holds that $f'(x_1) \leq f'(x_2)$.
\end{definition}
%
% The sentence below is redundant and confusing -- the monotonicity property 
% is a formal way of speciying what we mean by "strictly different" 
%
% Note that this monotonicity property ensures that $f_\Bcal$ does not rank 
% (pairs of) candidates \emph{strictly differently}.
%
To guarantee that $f_\Bcal$ is monotone with respect to $f$, we need to restrict our attention to the set of contiguous partitions $\Bscr\subseteq\Pscr$ of $\{1, \ldots, n\}$, \ie, for any $\Bcal \in \Bscr$, if $i(x_1)< i(x_2)< i(x_3)$ and $i(x_1) \sim_\Bcal i(x_3)$, then it also holds that $i(x_1) \sim_\Bcal i(x_2)$ and $i(x_2) \sim_\Bcal i(x_3)$. 
More formally, we have the following result:
\begin{proposition} \label{prop:monotone-contiguous}
Given a classifier $f$ with $\Rf = \{a_1, \ldots, a_n\}$, $f_{\Bcal}$ is monotone with respect to $f$ iff $\Bcal$ 
is a contiguous partition on $\{1, \ldots, n\}$.
\end{proposition}
Surprisingly, while $\abr{\Bscr} = 2^{n-1}$, we will show in the next section that it is possible to find the optimal contiguous partition $\Bcal^{*} = \argmax_{\Bcal \in \Bscr} |\Bcal|$ such that $f_{\Bcal^{*}}$ is within-group monotone in polynomial time using dynamic programming.

\section{Optimal Set Partitioning via Dynamic Programming}
\label{sec:algorithm}
% In this section, we first claim that finer 
% partitions $\Bcal$ result in sharper classifiers 
% $f_\Bcal$. We hence try to find a partition $\Bcal$
% such that $f_\Bcal$ is within-group monotone and no
% refinement applied to it results in a within-group monotone classifier. We 
% further present a polynomial-time algorithm based on dynamic 
% programming with which we are able to find the 
% optimal partition $\Bcal^*$, the one of the largest
% size, such that $f_{\Bcal^*}$ is within-group 
% monotone.

Since the structure of our problem resembles isotonic regression,
one may think of using a simple variation of the many times re-discovered Pool Adjacent Violators (PAV) algorithm~\cite{Ayer1955ANED,EedenConstancevan1958Taeo,10.1093/biomet/46.3-4.317,10.2307/2332806} to find the optimal (contiguous) partition.
However, in what follows, we first show that the PAV algorithm may not find
the optimal partition---it is not even guaranteed to find a partition 
satisfying an intuitive type of local optimality.
Then, building on the reasons why the PAV algorithm may not find the optimal 
partition, we derive an efficient algorithm based on dynamic 
programming that is guaranteed to find the optimal partition.
\begin{algorithm}[t]
\begin{algorithmic}[1]
    % \small
    \STATE{{\bf Input:} $\cbr{a_{1,z}, \ldots, a_{n, z}}_{z\in\Zcal}$}
    \STATE{{\bf Initialize:} $\Bcal_{\text{pav}}=\cbr{\cbr{1}, \ldots, \cbr{n}}$}
    \vspace{1mm}
    \WHILE{$\exists \Acal_{i-1}, \Acal_{i} \in \Bcal_{\text{pav}} \text{ and } z\in\Zcal \text{ such that } a_{\Acal_i,z}<a_{\Acal_{i-1},z}$}
        % \FOR{$i \in \cbr{2..\abr{\Bcal^t}}$}\label{inalg:forloop}
            % \STATE{$\Acal \leftarrow \cbr{\cbr{i}}$}
            % \FOR{$\Acal' \in \Bcal_d$}
            % \IF{$\exists i\in \cbr{2..\abr{\Bcal}}, z\in\Zcal \text{ such that }a_{\Acal_i,z}<a_{\Acal_{i-1},z}$}
            % \STATE{$\Acal = \Acal_i\cup \Acal_{i-1}$}
            \STATE{$\Bcal_{\text{pav}} = \Bcal_{\text{pav}} \setminus \cbr{\Acal_{i-1},\Acal_{i}}$}
            \STATE{$\Bcal_{\text{pav}} = \Bcal_{\text{pav}} \cup \cbr{\Acal_{i-1}\cup \Acal_{i}}$}
                % \STATE{$\Bcal^{t+1}\leftarrow\Bcal^t$}
                % \STATE{\textbf{Go to~\ref{inalg:forloop}}}
            %     \STATE{\textbf{Continue}}
            % \ELSE
            %     \STATE{\textbf{Break}}
            % \ELSE
            %     \STATE{$\Bcal^{t+1}\leftarrow\Bcal^t$}
                % \STATE{\textbf{Go to~\ref{inalg:forloop}}}
            % \ENDIF
            % \STATE{$t = t+1$}
            % \STATE{\textbf{Go to~\ref{inalg:forloop}}}
        % \ENDFOR
        % \STATE{$t = t+1$}
        % \STATE{\textbf{Go to~\ref{inalg:forloop}}}
    \ENDWHILE
    \STATE{{\bf return} $\Bcal_{\text{pav}}$}
\end{algorithmic}
\caption{It returns a partition $\Bcal_{\text{pav}}$ such that $f_{\Bcal_{\text{pav}}}$ is within-group monotone.} % and non-dominated.}
\label{alg:pav}
\end{algorithm}

\subsection{Pool Adjacent Violators (PAV) Algorithm}
%
% Note that, not all within-group monotone classifiers are equally useful. As 
% an example, consider the case where all the bins of $f$ are merged into one, 
% \ie, $\Bcal = \cbr{\Rf}$ and $f_\Bcal(X) = \Pr(Y=1)$. Such a classifier, 
% although monotone within groups, is not able to distinguish between different
% individuals. Furthermore, no matter what minimum number of qualified 
% candidates are required, such a classifier shortlists every candidate, 
% imposing unnecessary burden on the decision maker. As a result, 
%
% To find a non-dominated classifier which is within-group monotone, we use a 
% generalization of the pool adjacent violators (PAV) 
% algorithm~\cite{Ayer1955ANED,EedenConstancevan1958Taeo,10.1093/biomet/46.3-4.317,
% 10.2307/2332806} which is commonly used for solving monotonic regression problems
% in an iterative manner. 
%
In comparison with the original PAV algorithm, the only difference is that,
in our setting, one needs to check for monotonicity violations across multiple 
sets of conditional predictors, one per group $z \in \Zcal$, rather than 
only one set of predictors. 
However, the main idea underpinning the PAV algorithm remains the same, \ie, 
as long as there are monotonicity violations between two adjacent cells, the 
algorithm merges the corresponding cells into one.
%
% Let $\{i, \ldots, j\} = \cbr{l\in[n]\given i\leq l\leq j}$ be the interval of all
% integers between $i$ and $j$ included and further define $\Bcal^t = 
% \cbr{\Acal_i}_{i\in[\abr{\Bcal^t}]}$ as the partition at iteration $t$ with its 
% cells being index sorted based on their mean value $a_{\Acal_i}$. Starting from 
% $\Bcal^0=\cbr{\cbr{a_i}_{i\in[n]}}$, in which every cell is a singleton 
% containing one bin of $\Rf$, we do the following:
%
% \begin{itemize}
%    \item[1.] For the smallest $i\in\cbr{2..\abr{\Bcal^t}}$ such that there exists
% a $z\in\Zcal$ such that $a_{\Acal_{i},z}<a_{\Acal_{i-1},z}$, merge cells $i$ and 
% $i-1$ of $\Bcal^t$ (Adjacent pooling). If no such index exists, stop.
%    \item[2.] Build $\Bcal^{t+1}$ based on $\Bcal^t$ by replacing the cells
% $\Acal_i,\Acal_{i-1}\in\Bcal^t$ with $\Acal_i \cup \Acal_{i-1}$.
%    \item[3.] $t = t+1$ and go to step $1$. 
% \end{itemize}
%
Algorithm~\ref{alg:pav} % in Appendix~\ref{app:pav} 
summarizes the overall procedure, which has complexity
$\Ocal(n^2\times \abr{\Zcal})$ and is guaranteed to return a partition $\Bcal_{\text{pav}}$ such that $ f_{\Bcal_{\text{pav}}}$ is within-group monotone, as formalized by the following 
Proposition:
\begin{proposition} \label{prop:pav-within-group-monotonicity}
Algorithm~\ref{alg:pav} returns a partition $\Bcal_{\text{pav}} \in \Bscr$ such that the classifier $f_{\Bcal_{\text{pav}}}$ is within-group monotone.
\end{proposition}
Unfortunately, while the original PAV algorithm does enjoy global optimality guarantees for the isotonic regression problem\footnote{\scriptsize In 
the isotonic regression problem~\cite{barlow1972isotonic}, given a set of response variables $\{ y_i \}_{i \in [n]}$, the goal is to find 
a set of predictor values $\{ x_i \}_{i \in [n]}$, with $x_i \leq x_{i+1}$ for 
all $i \in [n]$, such that $\sum_{i} \ell(x_i, y_i)$ is minimized, where 
$\ell(x_i, y_i)$ is a loss measuring how well $x_i$ approximates $y_i$.} under multiple choices of loss functions~\cite{yu2016exact, jordan2019optimal}, this is not true for our problem. There exist many instances for which Algorithm~\ref{alg:pav} fails to find the optimal 
partition $\Bcal^{*}$, \eg, refer to Figure~\ref{fig:exp_violations_1} in Appendix~\ref{app:exp-partitions}.
%
% That being said
%
In fact, Algorithm~\ref{alg:pav} does not even enjoy a type of intuitive local optimality guarantee 
based on the notion of dominance~\cite{Wang2022ImprovingSP}:
\begin{definition}
Let $f$ and $f'$ be calibrated classifiers. Classifier $f$ dominates $f'$ if, 
for any $x_1,x_2\in \Xcal$ such that $f(x_1) = f(x_2)$, it holds that 
$f'(x_1) = f'(x_2)$.
% ~\cite{Wang2022ImprovingSP}.
\end{definition}
More specifically, if $f_{\Bcal}$ dominates $f_{\Bcal'}$, it can be shown that the expected size of the shortlists provided by the optimal screening policies using $f_{\Bcal}$ are not larger than those using $f_{\Bcal'}$~\cite[Corollary 4.3]{Wang2022ImprovingSP}) and it clearly holds that $|\Bcal| \geq |\Bcal'|$.
%
% In particular, we have the following negative result:
%
For example, 
% \begin{proposition} \label{prop:pav-non-dominance}
let $\Rf = \{a_1, a_2, a_3\}$, $\Zcal = \cbr{z_1,z_2}$ and 
$\rho_i\rho_{z\given i}=\frac{1}{6}$ for all $i \in \{1, 2, 3\}$ and 
$z\in\Zcal$.
Further, let $a_{1,z_2}=a_{2,z_1}=a_{3,z_2}=\alpha$, 
$a_{1,z_1}=2\alpha$, $a_{2,z_2}=3\alpha$ and $a_{3,z_1}=4\alpha$, 
where $\alpha\in[0,0.25]$. 
Then, 
Algorithm~\ref{alg:pav} returns $\Bcal_{\text{pav}} = \cbr{\cbr{1,2,3}}$, however, $f_{\Bcal_{\text{pav}}}$ is dominated by $f_{\Bcal}$, with $\Bcal = \cbr{\cbr{1},\cbr{2,3}}$, which is also within-group monotone. Refer to Appendix~\ref{app:pav-suboptimality} for details.
%
% \end{proposition}
%
% In the above, note that there may exist multiple partitions $\Bcal \in \Bscr$ such that $f_{\Bcal}$ is 
% not dominated by 
% any other classifier $f_{\Bcal'}$ with $\Bcal' \in \Bscr$, including $\Bcal^{*}$, and thus Algorithm~\ref{alg:pav} may 
% return a partition $\Bcal \neq \Bcal^{*}$.

The reason why Algorithm~\ref{alg:pav} may fail to find the optimal partition is that, 
whenever it tries to fix a monotonicity violation between two adjacent cells $\Acal_{i-1}$ and $\Acal_i$, 
it does so by mer\-ging them. However, in our problem, the optimal fix may require 
merging cells $\Acal_{i}$ and $\Acal_{i+1}$.
Building on this insight, we will design an efficient algorithm based on dynamic programming that provably finds the optimal~partition\-. 
\begin{algorithm}[t]
\begin{algorithmic}[1]
    % \small
    \STATE{{\bf Input:} $\cbr{a_{1,z}, \ldots, a_{n, z}}_{z\in\Zcal}$}
    \STATE{{\bf Initialize:} $\Bcal_{l,r} = \cbr{}$ $\forall l,r\in\cbr{2, \ldots, n}$, $\Bcal_{1,r} = \cbr{1, \ldots, r}$ $\forall r \in \cbr{1, \ldots, n}$}
    % \vspace{1mm}
    \FOR{$l \in \cbr{2,\ldots,n}$}
        \FOR{$r \in \cbr{l,\ldots,n}$}
            \STATE{$\Scal_{l,r} = \cbr{k | k<l, a_{\cbr{k, \ldots, l-1},z}\leq a_{\cbr{l, \ldots, r},z}~\forall{z\in\Zcal}}$}\quad \COMMENT{Refer to Lemma.~\ref{lem:recursive-approach}}
            \IF{$\Scal_{l,r}=\emptyset$}
                \STATE{\textbf{Continue}}\quad \COMMENT{\textit{In this case $\Bscr_{l,r} = \emptyset$}}
            \ENDIF
            \STATE{$k^* = \argmax_{k\in\Scal_{l,r}}\abr{\Bcal_{k,l-1}}$}
            \STATE{$\Bcal_{l,r} = \Bcal_{k^*,l-1} \cup \cbr{\cbr{l, \ldots, r}}$} 
            % \IF{$r = n \text{ and } \abr{\Bcal_{l,r}^*}>\abr{\Bcal^*}$}
            %     \STATE{$\Bcal^* = \Bcal_{l,r}^*$} \quad \COMMENT{\textit{Since $\Bcal^* = \Bcal_{l,n}^* \text{ if } l = \argmax_{l'\in\cbr{2, \ldots, n}} \abr{\Bcal_{l',n}^*}$}}
            % \ENDIF
        \ENDFOR
    \ENDFOR
    \STATE{$l^* = \argmax_{i\in\cbr{1, \ldots, n}} \abr{\Bcal_{i,n}}$}
    % \Bcal^* = 
    \STATE{{\bf return} $\Bcal_{l^*,n}$}
\end{algorithmic}
\caption{It returns the optimal partition $\Bcal^{*}$ such that $f_{\Bcal^{*}}$ is within-group monotone.}
\label{alg:optimal}
\end{algorithm}

\subsection{An Optimal Dynamic Programming Algorithm}
% \subsection{Optimal partition}
%
% Let $\Bscr_r$ be the set of contiguous partitions of the bin indices $\cbr{1, \ldots, r}$, 
% with $r\leq n$. 
%
Our starting point is the following observation, which allows us to break down the problem
of finding the optimal partition $\Bcal^{*}$ into several subproblems.
Let $\Bscr_r$ be the set of con\-ti\-guous partitions of the bin indices $\cbr{1, \ldots, r}$, with $r \leq n$, and
%
% The optimal partition in $\Bscr_{l,r}$ is furthermore the partition in it of largest size.
%
% The key idea behind our approach is as follows. 
%
$\Bscr_{l,r} \subseteq \Bscr_r$ be the subset of those partitions such that, for any $\Bcal = \{ \Acal_{1}, \ldots, \Acal_{|\Bcal|}\} \in \Bscr_{l,r}$, it holds that $\Acal_{|\Bcal|} = \{l, \ldots, r\}$ and $f_{\Bcal \cup \Bcal'}$ is within-group monotone on the region of the feature space defined by $\cup_{i \leq r} \Xcal_i$, where $\Bcal'$ is any partition of the bin indices $\{r+1, \ldots, n\}$\footnote{\scriptsize Note that it may be impossible to satisfy both conditions simultaneously if, for example, the Simpon'{}s paradox~\cite{Simpson1951TheIO} holds, \ie, for every group $z \in \Zcal$ and every pair of indices $i < j$, we have that $a_{i,z} > a_{j,z}$. In those cases, we may have that $\Bscr_{l,r}=\emptyset$ for all $1<l\leq r$.}. 
Then, it clearly holds that the optimal partition $\Bcal^{*} \in \cup_{l=1}^{n}\Bscr_{l,n}$ and thus we can break the problem of finding $\Bcal^{*}$ into $n$ subproblems, \ie, finding the optimal partition $\Bcal^{*}_{l, n} = \argmax_{\Bcal \in \Bscr_{l,n}} |\Bcal|$ within in each subset $\Bscr_{l,n}$.
From now on, with a slight abuse of notation, we will write $f_{\Bcal}$ instead of $f_{\Bcal \cup \Bcal'}$ whenever $\Bcal'$ refers to any partition of the bin indices not in $\Bcal$ and it is clear from the context.

% We further let $\Bcal_{l,r}^*$ be the optimal partition among the partitions in $\Bscr_{l,r}$, \ie, $\Bcal_{l,r}^* = \argmax_{\Bcal\in\Bscr_{l,r}} \abr{\Bcal}$.
% \begin{align}\label{def:intv-optimal}
%     \Bcal_{l,r}^* = \argmax_{\Bcal\in\Bscr_{l,r}} \abr{\Bcal} \quad\text{ subject to }
%     \quad a_{\Acal_i,z}\leq a_{\Acal_j,z} \quad
%     \forall \Acal_i, \Acal_j \in \Bcal \text{ such that } i < j,
%     % a_{\Acal,z}\leq a_{\Acal',z}
%     % ~\forall a_\Acal<a_{\Acal'}\leq a_{[r]}\in\Rbf,
%     \forall z \in \Zcal.
% \end{align}
% It is clear that $\Bcal^*$ is the largest partition among the partitions $\Bcal_{k,n}^*$ with $k\in[n]$.
% \ie, 
% \begin{align}\label{eq:optimal-partition-recursive}
%     \Bcal^* = \Bcal_{k^*,n}^* \text{ such that } k^* = \argmax_{k\in[n]} \abr{\Bcal_{k,n}^*}.
% \end{align}
% As a result, the original problem of finding the optimal partition can be broken into the subproblems of finding optimal partitions $\Bcal_{k,n}^*$. 

Next, we realize that we can efficiently find the optimal partition $\Bcal^{*}_{l, n}$ in each subset $\Bscr_{l,n}$ recursively using dynamic programming. 
%
% The key idea of the recursion is that for any partition $\Bcal = \Bcal' \cup \cbr{\R{l}{r}} \in \Bscr_{l,r}$ such that $\Bcal' \in \Bscr_{k,l-1}$ with $k < l$, the following necessary and sufficient condition must hold:
%
The key idea of the recursion is that any partition $\Bcal \in \Bscr_{l,r}$ 
% it holds that $\Bcal = \Bcal' \cup \cbr{\R{l}{r}}$ such that $\Bcal' \in \Bscr_{k,l-1}$ for some $k < l$ and 
needs to satisfy the following necessary and sufficient conditions:
%
% Formally, we have the following proposition:
%
\begin{lemma}\label{lem:recursive-approach}
Given any $\Bcal \in \Bscr_r$, it holds that $\Bcal \in \Bscr_{l,r}$ if and only if $\exists k<l$ such that $\Bcal\setminus\cbr{\R{l}{r}} \in\Bscr_{k,l-1}$ and $a_{\cbr{k, \ldots, l-1},z}\leq a_{\cbr{l, \ldots, r},z}$ $\forall z \in \Zcal$. 
\end{lemma}
% Let $\Scal_{i,j}$ for all $i,j\in\cbr{2..n}$ such that $i\leq j$ be the set of indices $k$ such that the group conditional cell mean in the cell $\cbr{k..i-1}$ is not higher than its counterpart in the cell $\cbr{i..j}$ for any group $z\in\Zcal$, \ie,
%
% \begin{align}
%     \Scal_{i,j} = \cbr{k<i\given a_{\cbr{k..i-1},z}\leq a_{\cbr{i..j},z}~\forall{z\in\Zcal}}.
% \end{align}
% Note that if $\Scal_{i,j}=\emptyset$ we can directly conclude that $\Bscr_{i,j}=\emptyset$. We now prove that any partition $\Bcal\in\Bscr_{i,j}$ is in the form of $\Bcal' \cup \cbr{i..j}$ for some $k\in\Scal_{i,j}$ and $\Bcal'\in\Bscr_{k,i-1}$. Formally we have the following theorem.
%
% \begin{theorem}\label{th:no-violation}
% A classifier $f_\Bcal$ with $\cbr{i..j}\in\Bcal$ is within-group monotone on $\cup_{l\leq j}\Xcal_l$ iff for some $k\in\Scal_{i,j}$ and $\Bcal'\in\Bscr_{k,i-1}$ it holds that $\Bcal = \Bcal'\cup \cbr{\cbr{i..j}}$.
% \end{theorem}
%
Consequently, we can efficiently find all the partitions in the subsets $\Bscr_{l,r}$ 
% in increasing order with respect to 
iterating through $l$ using the partitions in the subsets $\Bscr_{k,l-1}$ with $k < l$. 
Finally, by construction, it clearly holds that, if $\Bcal^{*}_{l, r} = \Bcal' \cup \cbr{\R{l}{r}}$, with $\Bcal' \in \Bscr_{k,l-1}$, is the optimal partition in $\Bscr_{l, r}$ then $\Bcal' = \Bcal^{*}_{k, l-1}$ is the optimal partition in $\Bscr_{k, l-1}$.
As a result, at each step of the recursion, we only need to store the optimal partition $\Bcal^{*}_{l,r}$, not all partitions in $\Bscr_{l,r}$.
%
% if $\Bcal_{l,r}^* = \Bcal \cup \cbr{\cbr{l, \ldots, r}}$ for a $\Bcal\in\Bscr_{k,l-1}$ then it should hold that $\Bcal = \Bcal_{k,l-1}^*$.
% As a result, there exists $\sum_{k\in\Scal_{i,j}}\abr{\Bscr_{k,j}}$ potential partitions among which we are interested in the one with the maximum size. It is clear that among the potential partitions, the one of the largest size with $\cbr{i..j}$ appended to it results in the optimal partition $\Bcal_{i,j}^*$, \ie, we have the following theorem.
% \begin{theorem}\label{th:recursive}
% If $k^* = \argmax_{k\in\Scal_{i,j}} \abr{\Bcal_{k,i-1}^*}$ then we have that $\Bcal_{i,j}^* = \Bcal_{k^*,i-1}^* \cup \cbr{\cbr{i..j}}$.
% \end{theorem}

Algorithm~\ref{alg:optimal} summarizes the overall procedure, which has complexity 
$\Ocal(n^3 \times \abr{\Zcal})$ and is guaranteed to find the optimal partition $\Bcal^{*}$, as formalized by the following theorem:
%
% and for any $l,r\in\R{1}{n}$ such that $l<r$, only stores the optimal partition in 
% $\Bscr_{l,r}$ and reuses it for finding optimal partitions in $\Bscr_{l',r'}$ with 
% $l'>l$ and $r'>r$.
% Firstly, it breaks down the problem of finding the optimal partition to the subproblems of finding the partitions $\Bcal_{k,n}^*$. It then solves those subproblems recursively based on Lemma.~\ref{lem:recursive-approach}. 
% The algorithm is further guaranteed to return the optimal partition $\Bcal^*$. In 
% particular, we have the following theorem:
%
\begin{theorem}\label{thm:optimal}
Algorithm~\ref{alg:optimal} returns $\Bcal^{*} = \argmax_{\Bcal \in \Bscr} |\Bcal|$ such that $f_{\Bcal^{*}}$ is within-group monotone.
\end{theorem}

\xhdr{Remark}
In many domains, allowing for a pre-specified, application-dependent level of within-group monotonicity violations may be acceptable. Such tolerance levels, whether global or group-specific, can easily be integrated into our algorithm without introducing any computational overhead. More specifically, let $\tau_z\in[0,1]$ be the pre-specified maximum level of within-group monotonicity violations for each group $z\in\mathcal{Z}$, i.e., a classifier $f$ needs to satisfy that $\Pr(Y=1|f(X) = a, Z=z) \leq \Pr(Y=1|f(X) = b, Z=z) + \tau_z$ for all $z\in\mathcal{Z}$ and $a<b$. Then, one only needs to modify the condition in line 5 in Algorithm~\ref{alg:optimal} to $a_{\{k, \ldots, l-1\},z}\leq a_{\{l, \ldots, r\},z} + \tau_z~\forall{z\in\mathcal{Z}}$. Here, note that this modification does not add to the time or space complexity of our algorithm. At the same time, a similar proof as the proof of Theorem 4.4 shows that the algorithm can return the partition of maximum size such that the classifier induced by this partition is within-group monotone with a slack of $\tau_z$ for all groups $z\in\mathcal{Z}$. Note that such relaxations will result in partitions of larger sizes or, equivalently, more fine-grained classifiers and may be imposed by the domain expert to tradeoff the prediction power and within-group fairness.

\section{Within-Group Monotonicity vs Within-Group Calibration}
\label{sec:group-calibration}
\begin{algorithm}[t]
\begin{algorithmic}[1]
    % \small
    \STATE{{\bf Input:} $\cbr{a_{1,z}, \ldots, a_{n, z}}_{z\in\Zcal}$}
    \STATE{{\bf Initialize:} $\Bcal_{\text{cal},i}=\cbr{}$ $\forall i\in\R{1}{n}$}
    \IF{$a_{1,z} = a_1~\forall z\in\Zcal$}
        \STATE{$\Bcal_{\text{cal},1} = \cbr{\cbr{a_1}}$}
    \ENDIF
    \vspace{1mm}
    \FOR{$r \in \cbr{2, \ldots, n}$}\label{inalg:first_for}
        \STATE{$\Scal_r = \cbr{i\in\R{2}{r}\given a_{\R{i}{r},z} = a_{\R{i}{r}}~\forall z\in\Zcal}$}
        \STATE{$k^* = \argmax_{k\in\Scal_r}\abr{\Bcal_{\text{cal}, k-1}}$}
        \IF{$\Bcal_{\text{cal}, k^*-1} \neq \emptyset$}
            \STATE{$\Bcal_{\text{cal},r} = \Bcal_{\text{cal},k^*-1}\cup \cbr{\R{k^*}{r}}$}
        \ELSIF{$a_{\R{1}{r}} = a_{\R{1}{r},z}~\forall z\in\Zcal$}
                \STATE{$\Bcal_{\text{cal},r} = \cbr{\R{1}{r}}$}
            % \ELSE
            %     \STATE{\textbf{Continue}}
            % \ENDIF
        \ENDIF
        % \STATE{$\Bcal_{\text{cal},r} = \Bcal_{\text{cal},k^*-1}\cup \cbr{\R{k^*}{r}}$}
        % \STATE{$\Acal = \Acal \cup \cbr{a_i}$}
        % % \FOR{$z \in \Zcal$}
        %     \IF{$\exists z\in\Zcal \text{ such that }\rho_{z\given \Acal}>0$ \text{and} $a_{\Acal,z}\neq a_{\Acal}$ }% \COMMENT{Based on Eq.~\ref{tilde_a} and Eq.~\ref{tilde_a_g}}
        %         \IF{$i<n$}
        %             \STATE{\textbf{Continue}}
        %             \ELSE
        %                 \STATE{{\bf return} $\cbr{}$ } %\COMMENT{\textbf{No multimonotone classifier exists}}
        %         \ENDIF
        %     \ENDIF
        % \ENDFOR
        % \STATE{$\Bcal_{\text{\text{cal}}}^* = \Bcal_{\text{\text{cal}}}^*\cup \cbr{\Acal}$}
        % \STATE{$\Acal = \{\}$}
    \ENDFOR
    % \Bcal_{\text{cal}}^* = 
    \STATE{{\bf return} $\Bcal_{\text{cal},n}$} %\COMMENT{Based on Eq.~\ref{tilde_a}}
\end{algorithmic}
\caption{It returns the optimal partition $\Bcal_{\text{\text{cal}}}^{*}$ such that $f_{\Bcal_{\text{\text{cal}}}^{*}}$ within-group calibrated.}
\label{alg:multicalibrated}
\end{algorithm}

Within-group calibration, or calibration within groups\footnote{\scriptsize There also exists a generalized, stronger notion of within-group calibration called multicalibration~\cite{hebert2018multicalibration,jung2021moment}, which requires predictions to be calibrated within every group that can be identified within a specified class of computations.},
%
% is a rather strong property since it 
requires that the probability that a candidate is qualified is independent of their group 
membership conditioned on their quality score.
More specifically, it is defined as fo\-llows~\cite{pleiss2017fairness,kleinberg2018inherent}:
\begin{definition}\label{def:within-group-calibration}
    Given a set of groups $\Zcal$, a classifier $f$ is within-group calibrated iff, for 
every $z \in \Zcal$ and $a \in \Rf$ such that $\Pr(Z=z\given f(X) = a) > 0$, it holds that $\Pr(Y = 1 \given f(X) = a, Z = z) = a$.
\end{definition}
As discussed previously, within-group calibration implies within-group monotonicity. 
% if a classifier is within-group calibrated, then, by definition, it 
% is also within-group monotone.
%
% Note that within-group calibration is a very strong assumption given 
% that the group 
% membership is not available to the 
% classifier~\cite{kleinberg_et_al:LIPIcs:2017:8156}. 
%
Then, to minimally modify a calibrated classifier $f$ so that it becomes within-group monotone, one may think of finding the optimal partition $\Bcal^{*}_{\text{\text{cal}}} = \argmax_{\Bcal \in \Bscr} |\Bcal|$ such that $f_{\Bcal}$ is within-group calibrated.
In what follows, we will first show that, perhaps surprisingly, finding $\Bcal^{*}_{\text{\text{cal}}}$ is computationally \emph{easier}\footnote{\scriptsize Using a similar proof technique as in Theorem~\ref{thm:np-hardness}, it can be proven that the problem of finding the partition $\Bcal \in \Pscr$ of maximum size such that $f_{\Bcal}$ is within-group calibrated is NP-hard. 
Therefore, in general, the computational complexity is not lower.} than finding $\Bcal^{*}$.
% , \ie, $\Ocal(n^2\times \abr{\Zcal})$ vs. $\Ocal(n^3\times \abr{\Zcal})$. 
%
%
% We first show that if such partitions exist, we are able to find the optimal one among 
% them using a polynomial-time algorithm. 
%
However, we will further show that, in many cases, $\Bcal_{\text{\text{cal}}}^{*}$ may not exist and, when it does exist, the size of $\Bcal_{\text{\text{cal}}}^{*}$ may be much smaller than the size of $\Bcal^{*}$, leading to less fine-grained predictions. 
%
% However, we also show that satisfying within-group calibration, which is a much % stronger assumption compared to within-group monotonicity, comes at the cost of modifying the initial classifier to a great extent, 
% and having 
% As a result, the predictions of the classifier $f_{$ even the optimal within-group 
% calibrated classifier may not be useful as a screening classifier. 
%
% Our recursive algorithm 
% finds the partition $\Bcal_{\text{cal}}^*$ which is the partition in 
% $\Bscr$ of maximum size such that $f_{\Bcal_{\text{\text{cal}}}^*}$ is 
% within-group calibrated. 
%

To find the optimal $\Bcal^{*}_{\text{\text{cal}}}$, we proceed recursively.
Let $\Bscr_r$ be the set of contiguous partitions of the bin indices $\cbr{1, \ldots, r}$, with $r \leq n$.
Then, iterating through $r$, we find the optimal partitions $\Bcal^{*}_{\text{\text{cal}}, r} = \argmax_{\Bcal \in \Bscr_r} |\Bcal|$ such that $f_{\Bcal^{*}_{ \text{\text{cal}}, r}}$ is within-group calibrated in $\cup_{i \leq r} \Xcal_i$. % in increasing order with respect to $r$. 
%
% The algorithm first considers the subproblems of finding within-
% group calibrated classifiers on $\cup_{i\leq r}\Xcal_i$ and solves 
% each subproblem recursively. 
%
In this case, the key idea of the recursion is that any partition $\Bcal \in \Bscr_r$ such that $f_{\Bcal}$ is within-calibrated on $\cup_{i \leq r}\Xcal_i$ needs to satisfy the following
necessary and sufficient condition: 
%
% which results in a within-group calibrated classifier can be 
% decomposed as $\Bcal\cup\cbr{\R{l}{r}}$ for $l\leq r$, if no 
% violations of within-group calibration occur in the cell $\R{l}{r}$ 
% and $f_\Bcal$ is within-group calibrated on $\cup_{i<l}\Xcal_i$. In 
% particular we have the following lemma:
%
\begin{lemma}\label{lem:multical-recursive}
Given any $\Bcal \in \Bscr_r$, it holds that $f_{\Bcal}$ is within-calibrated on $\cup_{i \leq r}\Xcal_i$ if and only if $\exists l < r$ such that $\Bcal \backslash \cbr{\R{l}{r}} \in \Bscr_{l-1}$ and $f_{\Bcal \backslash \cbr{\R{l}{r}}}$ is within-group calibrated on $\cup_{i \leq l-1} \Xcal_i$ and $a_{\R{l}{r},z} = a_{\R{l}{r}}$ $\forall z\in\Zcal$.
%
% The classifier $f_{\Bcal\cup\cbr{\R{l}{r}}}$ with $\Bcal\in\Bscr_{l-
% 1}$ such that $f_\Bcal$ is within-group calibrated on 
% $\cup_{i<l}\Xcal_i$ is within-group calibrated on $\cup_{i\leq 
% r}\Xcal_i$ iff it holds that $a_{\R{l}{r},z} = a_{\R{l}{r}}$ for all 
% $z\in\Zcal$.
\end{lemma}
As a consequence, we can efficiently find all partitions $\Bcal$ in the subsets $\Bscr_{r}$ such that $f_{\Bcal}$ is within-group calibrated 
%
% in increasing order with respect to 
%
iterating through $r$ using the partitions $\Bcal'$ in the subsets $\Bscr_{l}$ with $l < r$ such that $f_{\Bcal'}$ is within-group calibrated.
%
% using The above lemma allows us to find 
% 
% the optimal within-group calibrated classifier in a recursive 
% manner, based on such partitions on smaller set of indices. 
%
% Furthermore, it is clear that if the optimal partition in 
% $\Bscr_{r}$ resulting in a within-group calibrated classifier is 
% decomposed as $\Bcal\cup\cbr{\R{l}{r}}$ for a $\Bcal\in\Bscr_{l-1}$, 
% then $\Bcal$ must be optimal partition in $\Bscr_{l-1}$ resulting in 
% a within-group calibrated classifier. 
%
% For any $r\in\R{1}{n}$ the algorithm only stores the optimal 
% partition in $\Bscr_r$ which results in a within-group calibrated 
% classifier and reuses it to find such optimal partitions in 
% $\Bscr_{r'}$ for $r'>r$.
%
Finally, by construction, it clearly holds that if the optimal partition $\Bcal^{*}_{\text{\text{cal}}, r} = \Bcal' \cup \cbr{\R{l}{r}}$, with $\Bcal' \in \Bscr_{l-1}$, is the optimal partition in $\Bscr_{r}$ then $\Bcal' = \Bcal^{*}_{\text{\text{cal}}, l-1}$ is the optimal partition in $\Bscr_{l-1}$.
As a result, at each step of the recursion, we only need to store the optimal partition $\Bcal^{*}_{r}$, not all partitions $\Bcal \in \Bscr_{r}$ such that $f_{\Bcal}$ is within-group calibrated, and reuse it to find all $\Bcal^{*}_{r'}$ with $r' > r$.

Algorithm~\ref{alg:multicalibrated} summarizes the overall procedure, which has complexity $\Ocal(n^2\times \abr{\Zcal})$ and is guaranteed to find the optimal partition $\Bcal_{\text{\text{cal}}}^{*}$, if such a partition exists, as formalized by the following theorem:
\begin{theorem}\label{thm:multicalibrated}
Algorithm~\ref{alg:multicalibrated} returns $\Bcal_{\text{\text{cal}}}^* = \argmax_{\Bcal\in\Bscr}\abr{\Bcal}$ such that $f_{\Bcal_{\text{\text{cal}}}^*}$ is within-group calibrated if such partition exists or $\emptyset$ otherwise.
\end{theorem}
Unfortunately, there are many cases in which $\Bcal_{\text{\text{cal}}}^*$ does not exist, \eg, % and Algorithm~\ref{alg:multicalibrated} returns an empty set. 
%
% As an example, despite within-group monotonicity, within-group 
% calibration is rarely satisfied by merging all the bins together 
% into one. More specifically, as soon as there exists a $z\in\Zcal$ 
% such that $\Pr(Y=1\given Z=z)\neq \Pr(Y=1)$, merging all the bins 
% into one does not result in a within-group calibrated classifier. In 
% other words, to have $\abr{\Bcal} = 1$ we already need the strong 
% assumption that every group contains, in expectation, equal number 
% of qualified candidates.
%
% We first show that there exists calibrated classifiers for which no such 
% partition exist. 
%
this will happen if $f$ systematically undervalues the probability that individuals from a group are qualified, in comparison with individuals from another group:
% , as formalized by the following proposition:
%
\begin{proposition}\label{prop:multicalibrated-bin-diff}
Let $\Zcal = \cbr{z,z'}$, $\rho_{z\given i} = \rho_{z'\given i}$ and $a_{i,z}<a_{i,z'}$ for all $i\in\R{1}{n}$. Then, there exists no $\Bcal \in \Bscr$ such that $f_{\Bcal}$ is within-group calibrated.
\end{proposition}
%
% Remarkably, 
In the above situation, $f$ may actually be within-group monotone and thus $\abr{\Bcal^*} = n$.
%
% Furthermore, the conditions in the above Proposition can hold even 
% for a within-group monotone classifier $f$. As a result, even when 
% $\abr{\Bcal^*}$ is as large as $\abr{\R{\cbr{1}}{\cbr{n}}} = n$, it 
% can happen that no partition resulting in a within-group calibrated 
% classifier exits.
%
Even if $\Bcal_{\text{\text{cal}}}^*$ exists, there are
examples where % the gap in size between the optimal partitions is 
$|\Bcal^{*}| - |\Bcal_{\text{\text{cal}}}^*| = n-1$. 
\begin{figure*}[t]
\includegraphics[width=0.48\textwidth,right]{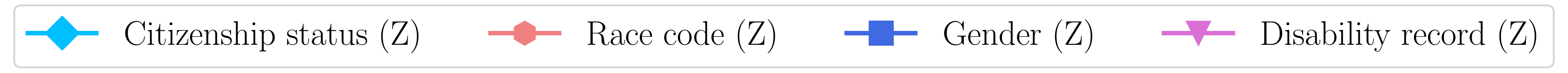}\\[-2.6ex]
\centering
 \subfloat[$p_{d \given z}$ vs. $\Pr(Z = z)$]{
\includegraphics[width=0.48\textwidth]{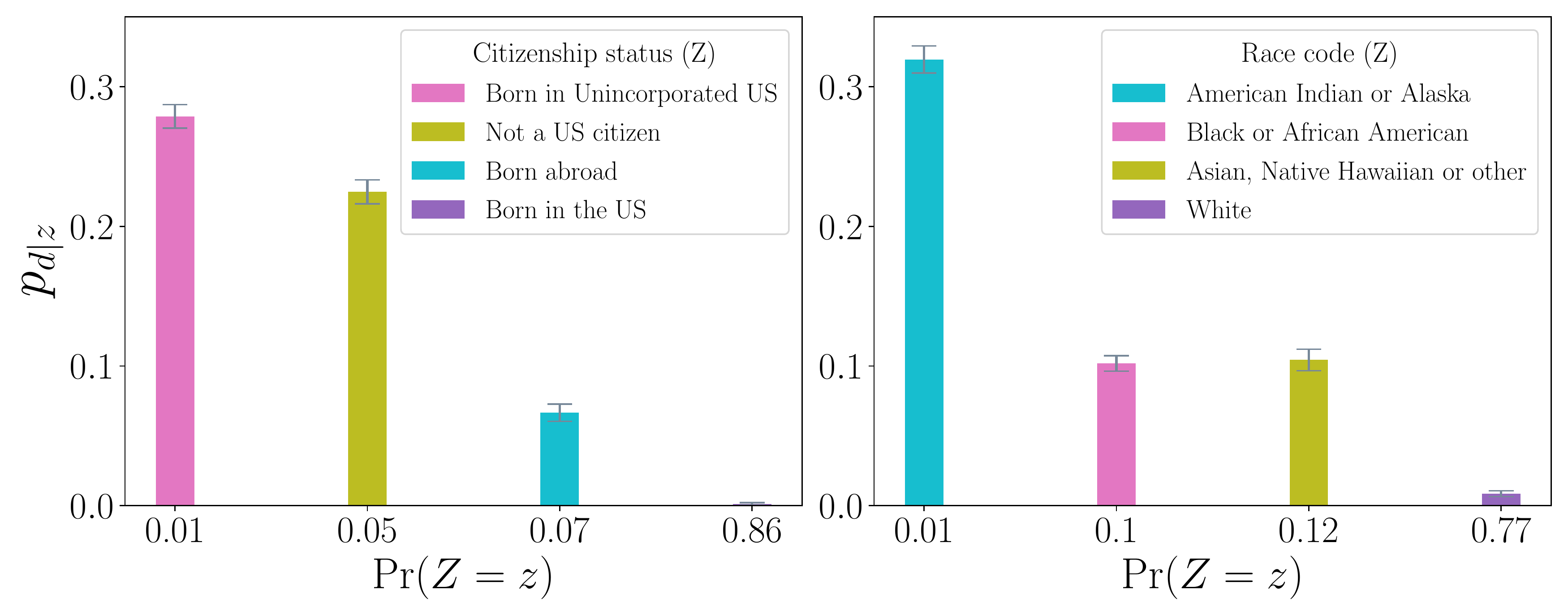}\label{fig:pdz}}
 \subfloat[$p_{d}$ vs. $n$]{
\includegraphics[width=0.24\textwidth]{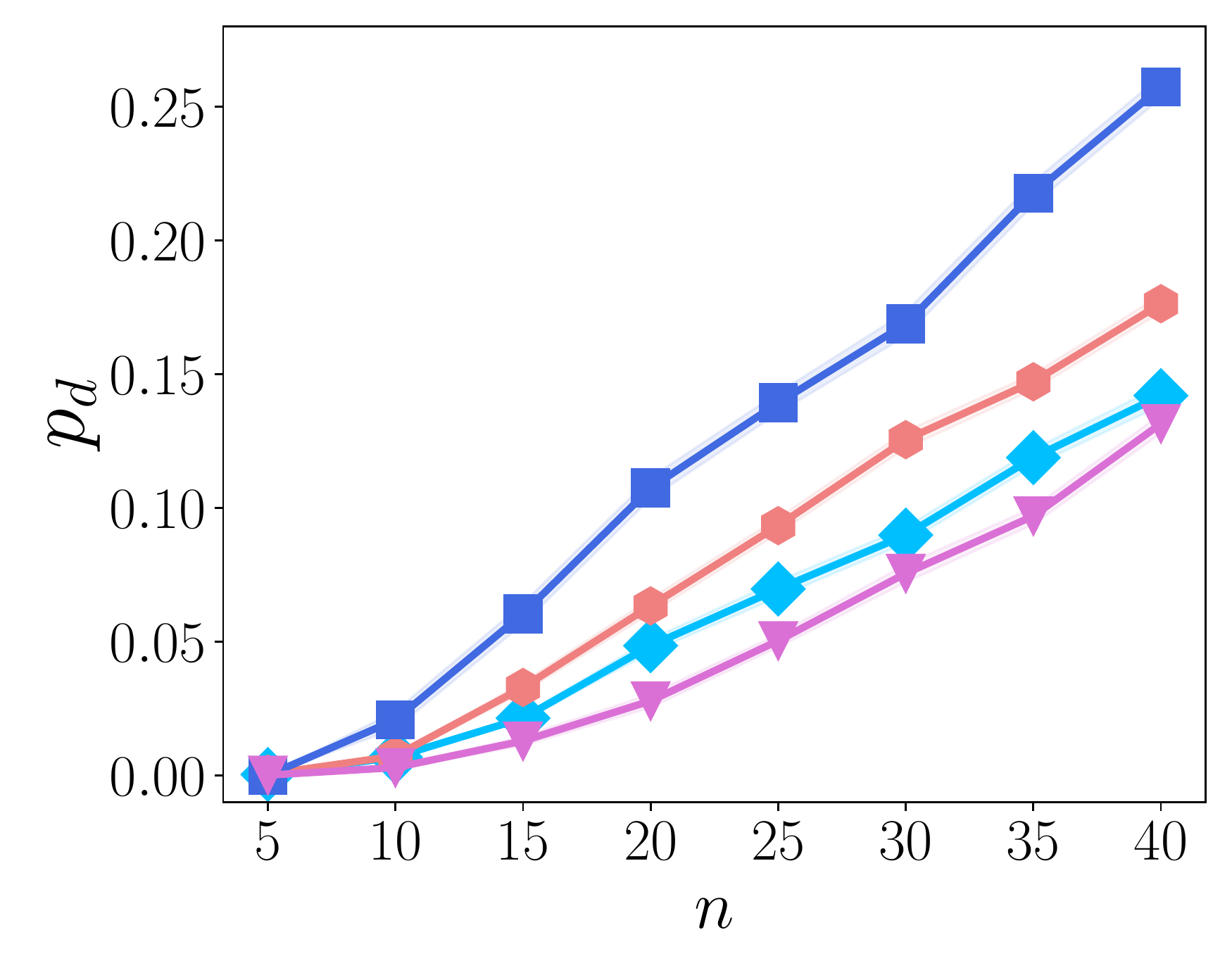} \label{fig:pd}}
 \subfloat[$p_{d \given \Dcal_{\text{pool}}}$ vs. $n$]{
\includegraphics[width=0.24\textwidth]{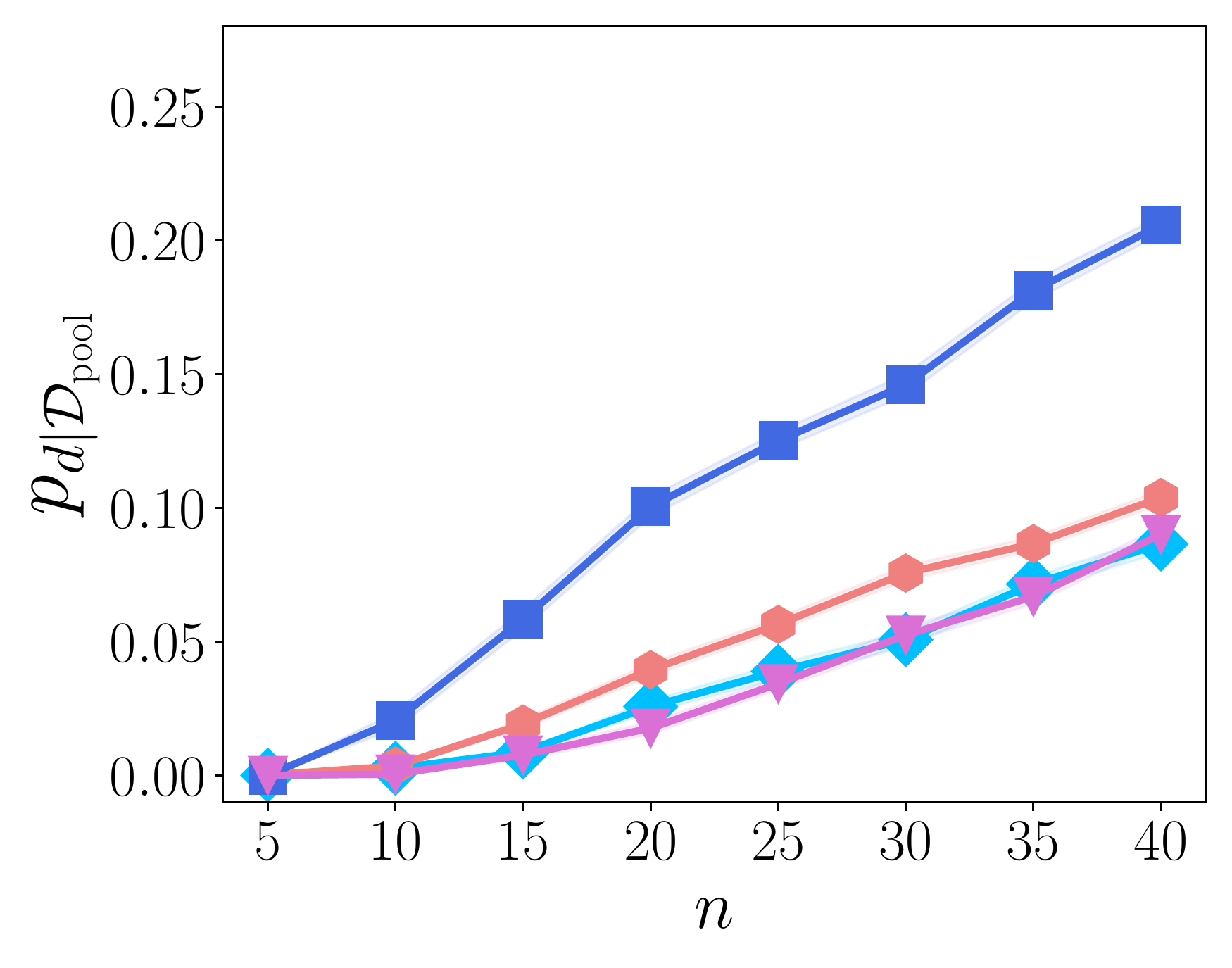} \label{fig:pdm}}
\caption{Probability that an individual suffers from within-group unfairness. 
Panel (a) shows the probability $p_{d \given z}$ that an individual from group $z$ may 
suffer from within-group unfairness against $\Pr(Z = z)$ for $n = 15$.
Panel (b) shows the probability $p_{d}$ that an individual may suffer from within-group 
unfairness. 
Panel (c) shows the probability $p_{d \given \Dcal_{\text{pool}}}$ that an individual suffers from 
within-group unfairness in a test pool $\Dcal_{\text{pool}}$ of size $m$, averaged
across all test pools, against $n = \abr{\Rf}$.}
\label{fig:exp_discriminations}
% \vspace{-3mm}
\end{figure*}

\section{Experiments Using Survey Data}
\label{sec:real}
In this section, we create multiple instances of a simulated screening process using US Census survey data to
first investigate how frequently within-group unfairness occurs and then compare the partitions, as well as induced scree\-ning classifiers, provided by Algorithms~\ref{alg:pav},~\ref{alg:optimal} and \ref{alg:multicalibrated}\footnote{\scriptsize We ran all experiments on a machine equipped with 48 Intel(R) Xeon(R) 2.50GHz CPU cores and 256GB memory.}. % both qualitatively and quantitatively.

\xhdr{Experimental setup.}
We use a dataset consisting of $\sim$$3.2$ million individuals from the US Census~\cite{ding2021retiring}. 
Each individual is represented by sixteen features and one label $y\in\cbr{0,1}$ indicating whether the individual is employed $(y=1)$ or not $(y=0)$. For our experiments, we think of employment as a (imperfect) proxy of qualification\footnote{\scriptsize Note that the label used as the proxy for qualification closely depends on the application domain. In an academic hiring scenario, the label “Educational Attainment” could serve as a proxy for qualification while “Years of Working Experience” might be a better proxy in hiring scenarios for craft professions.}.
The features contain demographic information such as age, marital status or 
gender~\cite[Appendix B4]{ding2021retiring}. % ~\cite[Appendix B4]{ding2021retiring}. 
We run four sets of experiments where, in each of them, we use a different feature 
(US citizen status, race, gender, or disability record) to define the demographic groups of interest $\Zcal$\footnote{\scriptsize For space reasons, in this section, we focus mainly on groups $z\in\Zcal$ based on US citizenship status and race. However, Appendix~\ref{app:exp-other-z} shows similar results for groups defined based on gender and disability record.}.
% the demographic groups of interest $\Zcal$. 

For the experiments, we randomly split the dataset into two equally-sized and disjoint 
subsets. 
We use the first subset for trai\-ning and calibration and the second subset for testing.
More specifically, for each experiment, we create the training and calibration sets 
$\Dcal_{\text{tr}}$ and $\Dcal_{\text{cal}}$ by picking $100{,}000$ and $50{,}000$ individuals
at random (without replacement) from the first subset. %(without replacement)
We use $\Dcal_{\text{tr}}$ to train a logistic regression model $f_{LR}$\footnote{\scriptsize The classifier $f_{LR}$ achieves a test accuracy of $\sim$$74$\% at predicting whether an individual is qualified.} and 
use $\Dcal_{\text{cal}}$ to both (approximately) calibrate $f_{LR}$ using uniform mass binning (UMB)~\cite{Wang2022ImprovingSP, zadrozny2001obtaining}, \ie, discretize its outputs to $n$ 
calibrated quality scores,
and estimate the relevant pro\-ba\-bi\-li\-ties $\rho_i$, $a_i$, $\rho_{z \given i}$ and $a_{i, z}$ needed 
by Algorithms~\ref{alg:pav},~\ref{alg:optimal} and \ref{alg:multicalibrated}. 
The resulting (approximately) calibrated classifier serves as our screening classifier $f$.
For testing, we create a set $\{ \Dcal_{\text{pool}}^i \}_{i=1}^{100}$ of $100$ pools, each with $m=100$ in\-di\-vi\-duals picked at random from the second subset, and create (the smallest) shortlists with at least $k$ qualified individuals using\- the screening classifiers $f_{\Bcal_{\text{pav}}}$, $f_{\Bcal^{*}}$ and $f_{\Bcal^{*}_{\text{cal}}}$ induced by the partitions found by Algorithms~\ref{alg:pav},~\ref{alg:optimal} and \ref{alg:multicalibrated}, respectively.
Here, since we find that, in most experiments, no within-group calibrated classifier exists,
we allow $f_{\Bcal^{*}_{\text{cal}}}$ to be within-group $\epsilon$-calibrated\footnote{\scriptsize Given a set of groups $\Zcal$, a classifier $f$ is within-group $\epsilon$-calibrated iff, for every $z \in \Zcal$ and $a \in \Rf$ such that $\Pr(Z=z\given f(X) = a) > 0$, it holds that $\abr{\Pr(Y = 1 \given f(X) = a, Z = z) - a}\leq \epsilon$.} within 
Algorithm~\ref{alg:multicalibrated} and use binary search to find the smallest $\epsilon \in (0, 1)$ such that $f_{\Bcal^{*}_{\text{cal}}}$ exists\footnote{\scriptsize Refer to Appendix~\ref{app:exp-epsilon-calibration} for additional experiments on within-group $\epsilon$-calibration.}.
Throughout the experiments, we estimate the ave\-rage and the standard error of the reported quantities by repeating each experiment $100$ times. 

\xhdr{Within-group unfairness occurs frequently between individuals from minority groups, especially with fine-grained classifiers.}
We start by estimating the pro\-ba\-bi\-li\-ty $p_{d \given z}$ that an individual from a demographic group of interest $z \in \Zcal$ may suffer from within-group unfairness, \ie, 
%
% \begin{equation*}
    % \varrho_z
    $p_{d \given z} = \frac{1}{\Pr(Z = z)} \sum_{i \in \R{1} {n}} \rho_i \rho_{z \given i} v_{i}$,
    % \frac{1}{\Pr(Z=z)}\sum_{i\in\R{1}{n}} \rho_i\rho_{z\given i}\II[\exists a_j\in\Rf \text{ such that 
    % } a_{i,z}>a_{j,z} \text{ and } a_i<a_j].
% \end{equation*}
% \begin{align*}
%     % \varrho_z
%     &p_{d \given z} = \frac{1}{\Pr(Z = z)} \sum_{i \in \R{1} {n}} \rho_i \rho_{z \given i} 
%     \\
%     &\II\left[ \exists a_j \in \Rf \given a_i < a_j \wedge a_{i,z} > a_{j,z} \right]
%     % \frac{1}{\Pr(Z=z)}\sum_{i\in\R{1}{n}} \rho_i\rho_{z\given i}\II[\exists a_j\in\Rf \text{ such that 
%     % } a_{i,z}>a_{j,z} \text{ and } a_i<a_j].
% \end{align*}
where $v_i = \II\left[ \exists a_j \in \Rf \given a_i < a_j \wedge a_{i,z} > a_{j,z} \right]$.
Figure~\ref{fig:pdz} summarizes the results for 
% groups $\Zcal$ based on US citizenship status and race and 
a screening classifier $f$ with $n = 15$ bins.
We find that individuals who belong to minority groups are much more likely to suffer from within-group unfairness than those who belong to a ma\-jo\-ri\-ty group. 
For example, the probability that an individual who is not a US citizen may suffer from within-group unfairness is $p_{d \given z} > 0.3$ while it is almost impossible that an individual born in the US is treated unfairly within their group.
Further, we investigate to what extent the probability $p_d = \sum_{z\in\Zcal} P(Z = z) p_{d \given z}$ that an individual may suffer from within-group unfairness depends on the number of bins $n$ of $f$.
Figure~\ref{fig:pd} shows that the more fine-grained a classifier is, the higher the probability that an 
individual may suffer from within-group unfairness, \eg, for $n \leq 10$, $p_d < 0.05$ while, for $n = 40$, 
$p_d > 0.12$ across all sets of groups $\Zcal$.
% while, for $n = 40$, $p_d = 0.25$.
%
Since the accuracy of a calibrated classifier is related to how fine-grained its predictions are~\cite{Wang2022ImprovingSP}, the above finding suggests that high accuracy may have a cost in terms of within-group unfairness.

Our results so far show that the probability that individuals \emph{may} suffer from within-group unfairness is significant. 
Next, we estimate the probability that in a test pool of size $m$, an individual \emph{does} suffer from within-group unfairness, \ie,
%
% \begin{align*}
    $p_{d \given \Dcal_{\text{pool}}} = \frac{1}{m} \sum_{x \in \Dcal_{\text{pool}}} v_{x}$,
% \end{align*}
%
where $v_{x} = \II\left[ \exists x' \in \Dcal_{\text{pool}} \given a_{i(x)} <a_{i(x')} \wedge a_{i(x),z}>a_{i(x'),z} \right]$.
Figure~\ref{fig:pdm} shows that, on average across all test pools, the probability $p_{d \given \Dcal_{\text{pool}}}$ follows the same trend as $p_d$, however, it is slightly lower in value because each of the test pools is not representative of the entire population.
% due to its limited size.
%
However, note that, as $m \rightarrow \infty$, one can readily conclude that $p_{d \given \Dcal_{\text{pool}}} \rightarrow p_d$.

\begin{figure}[t]
    \centering
    \hspace*{0.5cm}\includegraphics[width=0.6\textwidth]{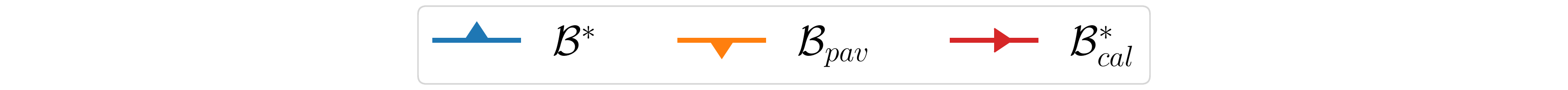}\\[-2.5ex]
    \subfloat[Citizenship status ($Z$)]{\includegraphics[width=0.35\textwidth]{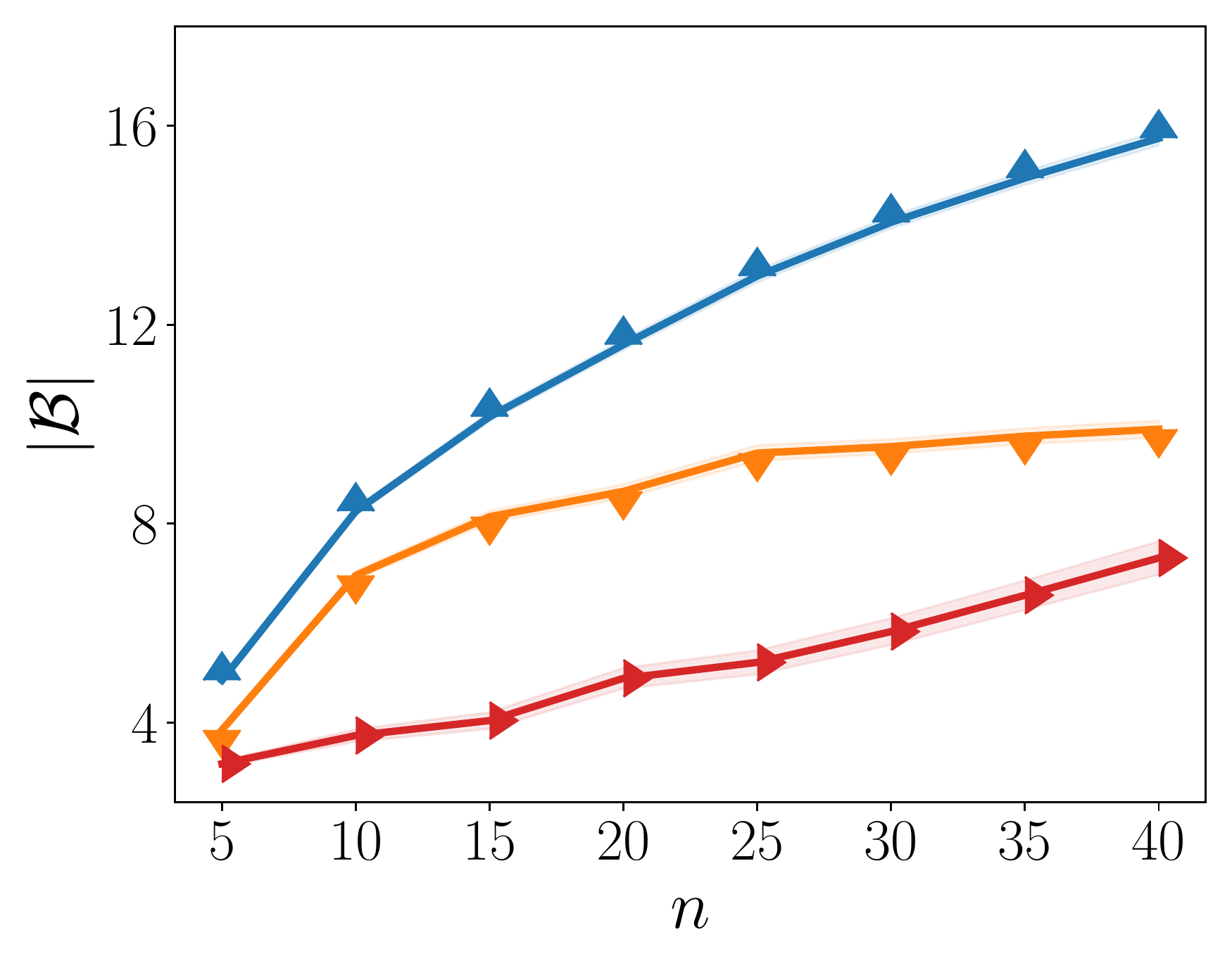}} \hspace{5mm}
    \subfloat[Race code ($Z$)]{\includegraphics[width=0.35\textwidth]{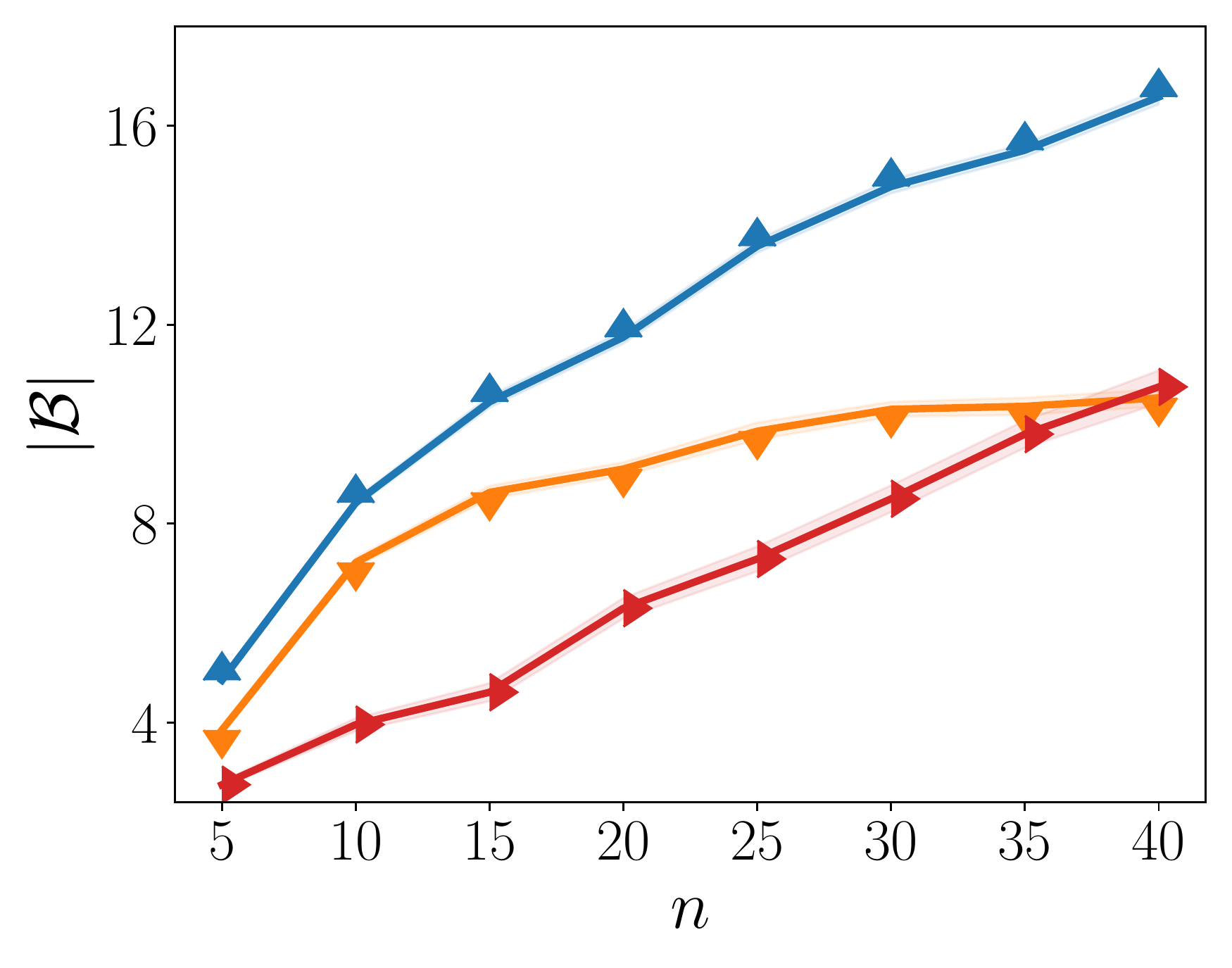}}
\caption{Size of the partitions $\Bcal_{\text{pav}}$, $\Bcal^{*}$ and $\Bcal^{*}_{\text{cal}}$
returned by Algorithms~\ref{alg:pav},~\ref{alg:optimal} and \ref{alg:multicalibrated}, respectively (higher is better).} \label{fig:partition-size}
% \vspace{-3mm}
\end{figure}

\xhdr{Algorithm~\ref{alg:optimal} consistently provides larger partitions, which result in more fine-grained classifiers and smaller shortlists, than Algorithms~\ref{alg:pav} and \ref{alg:multicalibrated}.}
We experiment with several screening classifiers $f$ with a varying number of bins $n$ and 
compare the size of the partitions $\Bcal$ provided by each of the algorithms, \ie, the number of bins of the 
modified classifiers $f_{\Bcal}$.
Figure~\ref{fig:partition-size} shows that the optimal partition $\Bcal^{*}$ 
% found by Algorithm~\ref{alg:optimal} 
is always greater in size than the partitions $\Bcal^{*}_{\text{cal}}$ and $\Bcal_{\text{pav}}$.
Moreover, it also shows that, as $n$ increases, the growth in the size of the partitions $\Bcal^*$ and $\Bcal_{pav}$ di\-mi\-ni\-shes because the occurrence of within-group unfairness increases, as shown in Figure~\ref{fig:exp_discriminations}.
Further, we use both the original classifier $f$ and the modified classifiers $f_{\Bcal^*}$, $f_{\Bcal_{\text{pav}}}$ and $f_{\Bcal^{*}_{\text{cal}}}$ to shortlist the minimum number of individuals among those in each of the simulated test pools $\{ \Bcal^{i}_{\text{pool}} \}$ such that, in expectation, there are at least $k$ qualified shortlisted individuals 
per pool.
To this end, for each test pool and classifier, we sort the candidates in decreasing order with respect to the corresponding quality score and, starting from the first, we keep shortlisting individuals in order until the 
sum of the quality scores reaches $k$ ~\cite[Appendix, A.3]{Wang2022ImprovingSP}). % ~\cite[Appendix A.3]{Wang2022ImprovingSP}.
Figure~\ref{fig:shortlist-size} shows that the shortlists created using $f_{\Bcal^{*}}$ are consistently smaller than those created using $f_{\Bcal_\text{pav}}$ and $f_{\Bcal^{*}_{\text{cal}}}$ for $k = 5$. Moreover, it also shows that the price to pay for achieving within-group monotonicity, \ie, the difference in size between the shortlists created using $f$ and $f_{\Bcal^{*}}$, is small. We found qualitatively similar results for other $k$ values.
Appendix~\ref{app:exp-partitions} takes a closer look at the (group conditional) score values of $f$, $f_{\Bcal^{*}}$, $f_{\Bcal_\text{pav}}$ and $f_{\Bcal^{*}_{\text{cal}}}$.

\xhdr{Remark.} Note the shortlists created using $f_{\Bcal^{*}}$ will be larger than those created using $f$ and this imposes more burden on the decision maker in selecting the desired number of qualified candidates (\eg, they have to interview more candidates). However, it ensures none of the members within demographic groups are unfairly treated. Therefore, it shifts the costs of using a poor screening classifier from the applicants to the decision-maker. If $\abr{\Bcal^{*}}$ is too small, it may be a sign that the decision maker has to reconsider using $f$ as the screening classifier.
\begin{figure}[t]
    \centering
    \hspace*{0.5cm}\includegraphics[width=0.6\textwidth]{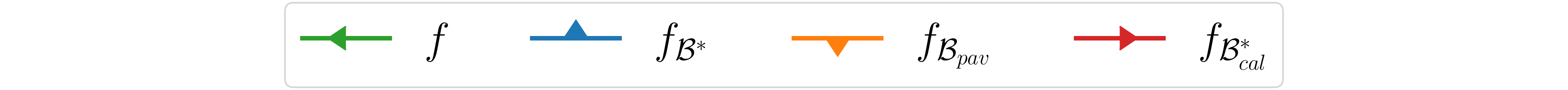}\\[-2.5ex]
    \subfloat[Citizenship status ($Z$)]{\includegraphics[width=0.35\textwidth]{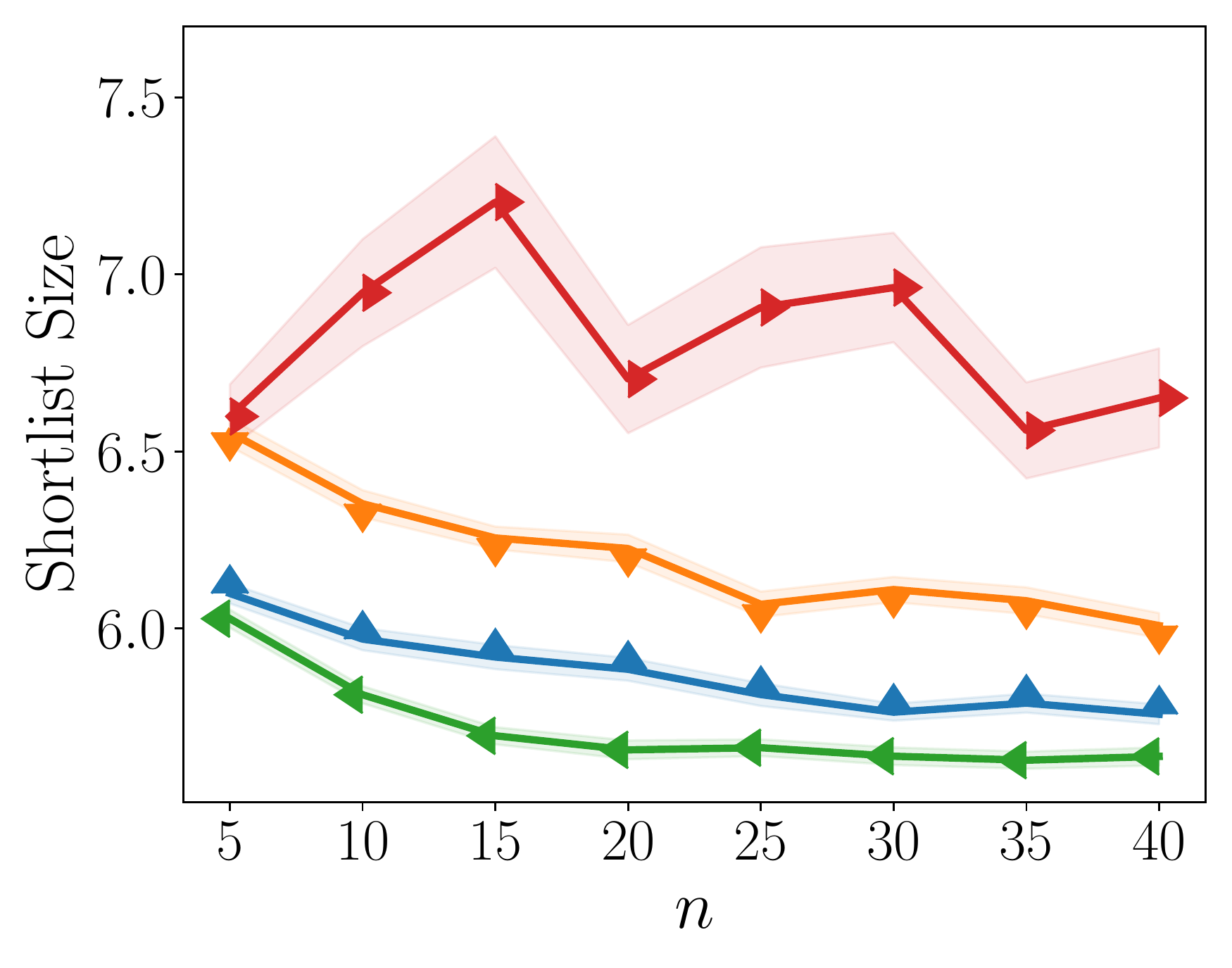}} \hspace{5mm}
    \subfloat[Race code ($Z$)]{\includegraphics[width=0.35\textwidth]{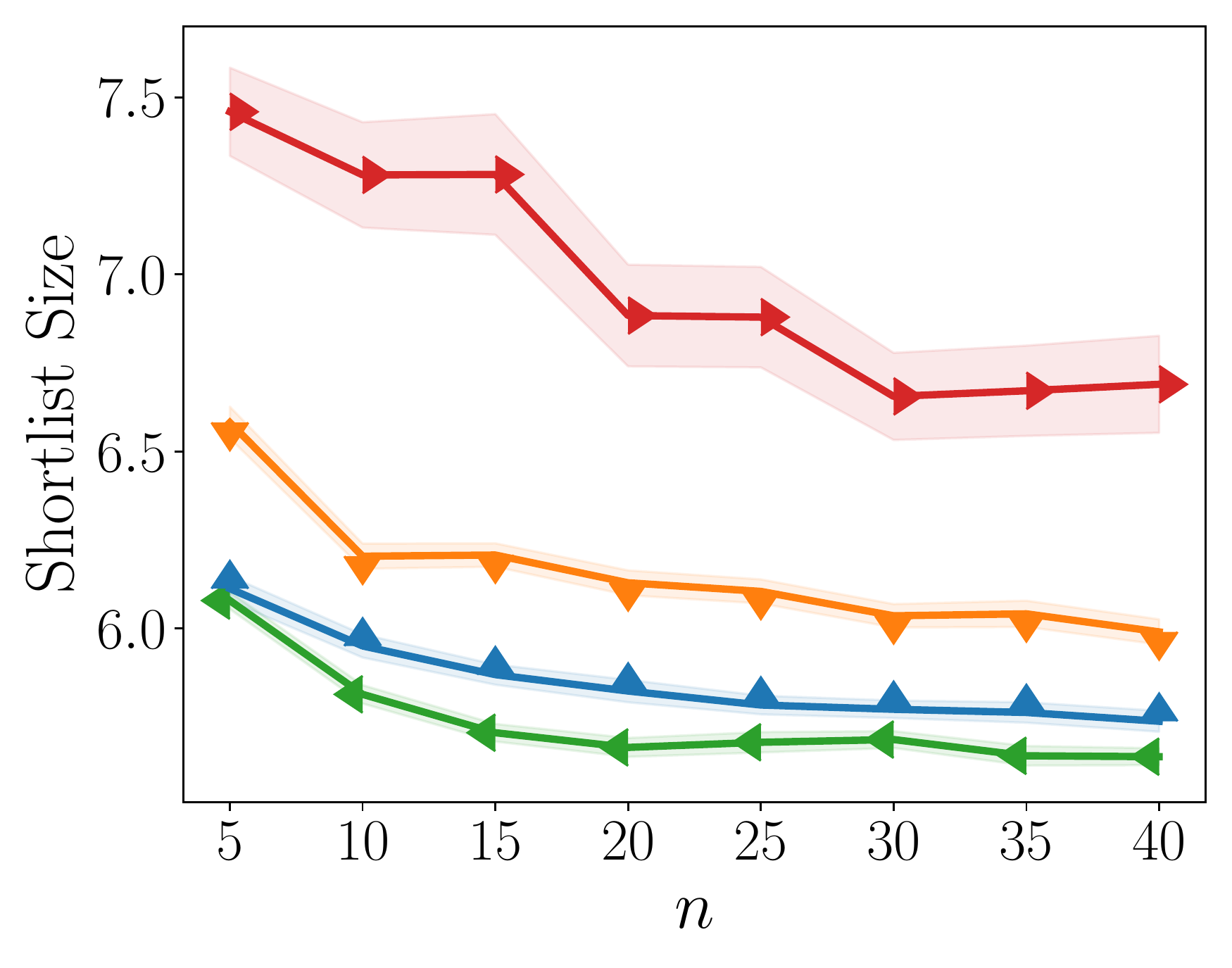}}
\caption{Size of the shortlists created using the original classifier $f$ and the modified classifiers $f_{\Bcal_{\text{pav}}}$, $f_{\Bcal^*}$ and $f_{\Bcal^{*}_{\text{cal}}}$ induced by the partitions found by Algorithms~\ref{alg:pav},~\ref{alg:optimal} and \ref{alg:multicalibrated}, respectively, for $k = 5$ (lower is better).} \label{fig:shortlist-size}
% \vspace{-3mm}
\end{figure}

\section{Conclusions}
\label{sec:conclusions}
In this work, we have first shown that optimal screening policies using calibrated classifiers may suffer from an understudied type of within-group unfairness. 
%, which we refer to as within-group discrimination.
%
Then, we have developed a polynomial time algorithm based on dynamic programming to minimally modify any given calibrated classifier so that it satisfies within-group monotonicity, a natural monotonicity property that prevents the occurrence of within-group unfairness.
Finally, we have shown that within-group monotonicity can be achieved at a small cost in terms of prediction granularity and shortlist size.
% , a well-known property that has been argued for in the fair machine learning literature.

Our work opens up many interesting avenues for future work. 
For example, it would be interesting to design classifiers that are within-group monotone with respect to every group that can be identified within a specified class of computations~\cite{hebert2018multicalibration}.
Moreover, in some scenarios, it might be sufficient to control the probability that an individual suffers from within-group unfairness.
Further, it would be important to investigate how within-group monotonicity interacts with group fairness~\cite{hardt2016equality,zafar2017fairness}.
Finally, it would be interesting to design post-processing algorithms using a sample access model~\cite{blasiok2022unifying}, rather than a prediction-only access model, and optimize other quality measures different from the partition size.

% section{Acknowledgements}
% \vspace{5mm}
%
\section*{Acknowledgements} 
We would like to thank Nina Corvelo Benz, Eleni Straitouri,
 and Luke Wang for fruitful discussions and constructive feedback during different stages of the project. Gomez-Rodriguez acknowledges support from the European Research Council (ERC) under the European Union'{}s Horizon 2020 research and innovation programme (grant agreement No. 945719).

{ 
\bibliographystyle{unsrt}
\bibliography{within-group-monotonicity}
}

\clearpage
\newpage

\appendix

% \section{Proofs} \label{app:proofs}
%

% \section{On multimonotonicity of calibrated and monotonic 
% classifiers}\label{multimononicity_calibrated_monotonic}

% \begin{proposition}
% A classifier $f$ which is perfectly calibrated on the distribution $P_{X,Y}$ 
% is not necessarily multimonotonic with respect to the distribution 
% $P_{X,Z,Y}$.
% \end{proposition}

% \begin{proposition}
% A multimonotonic classifier $f$, might be non-monotonic.
% \end{proposition}

\section{Proofs} \label{app:proofs}

\subsection{Proof of Proposition~\ref{prop:within-group-discrimination}}
By definition, the threshold decision rule $\pi$ outputs $S = 0$ if $f(X) = a$ and $S = 1$ if $f(X) = b$. As a result, it immediately follows that:
\begin{multline*}
    \EE_{Y \sim P_{Y \given X, Z}, \, S \sim \pi} \left[ Y (1-S) \given f(X) = a, Z = z \right] = \EE_{Y \sim P_{Y \given X, Z}} \left[ Y \given f(X) = a, Z = z \right] \\
    >  \EE_{Y \sim P_{Y \given X, Z}} \left[ Y \given f(X) = b, Z = z \right] =
    \EE_{Y \sim P_{Y \given X, Z}, \, S \sim \pi} \left[ Y S \given f(X) = b, Z = z \right].
\end{multline*}

\subsection{Proof of Theorem~\ref{thm:np-hardness}} \label{app:np-hardness}
We call a partition $\Bcal\in\Pscr$ valid if $f_\Bcal$ is within-group monotone.
We first show that, by finding a valid partition $\Bcal$ of maximum size, we can decide whether there exists a valid partition $\Bcal'$ of size $\abr{\Bcal'} = 2$.
Assume the valid partition $\Bcal$ of maximum size has size $\abr{\Bcal} = m$. Then, if 
$m \geq 2$, we can conclude that such a partition exists using Lemma~\ref{lem:dicision-to-original} and, if $m<2$, no such partition exists because $\Bcal$ is the valid partition 
of maximum size.
Now, since we prove in Lemma~\ref{lem:decision-np-hard} that this decision problem is
NP-complete, we can directly conclude that the problem of finding the valid partition of 
maximum size is NP-hard.

\begin{lemma}\label{lem:dicision-to-original}
    Assume the valid partition $\Bcal$ of maximum size has size $\abr{\Bcal}=k$. Then for every $k'\in\R{1}{k-1}$, there exist a valid partition $\Bcal'$ such that $\abr{\Bcal'} = k'$.
\end{lemma}
\begin{proof}
    By Proposition~\ref{prop:monotone-contiguous}, we have that any contiguous partition $\Bcal'$ on $\R{1}{\abr{\Bcal}}$ is monotone with respect to $f_\Bcal$. Furthermore, due to the same proposition, $\Bcal'$ is also monotone with respect to the set $\cbr{a_{\Acal_i,z}}_{i\in\R{1}{\abr{\Bcal}}}$ for all $z\in\Zcal$. Since $\Bcal$ is valid, we have that $\cbr{a_{\Acal_i,z}}_{i\in\R{1}{\abr{\Bcal}}}$ is increasing for all $z\in\Zcal$. As a result, $\Bcal'$ is a valid partition. Thus, for any $k'\in\R{1}{k-1}$, we have that the contiguous partition $\Bcal' = \cbr{\Acal_1, \Acal_2,\ldots, \Acal_{\abr{\Bcal}-k'-1},\cup_{j\in\R{0}{k'}}\Acal_{\abr{\Bcal}-j}}$ is valid and $\abr{\Bcal'} = k'$. This concludes the proof.
\end{proof}

\begin{lemma}\label{lem:decision-np-hard}
    The problem of deciding whether there exists a valid partition $\Bcal$ such that $\abr{\Bcal}=2$ is NP-complete.
\end{lemma}
\begin{proof}
    First it is easy to see that, given a partition $\Bcal$, we can check whether the partition is valid and has size $|\Bcal| = 2$ in polynomial time. Therefore, the problem belongs to NP.

    Now, to show the problem is NP-complete, we perform a reduction from a variation of the classical partition problem~\cite{karp1972reducibility}, which we refer to as the equal average partition problem.
    The equal average partition problem seeks to decide whether a set of $n$ positive integers $\Scal = \cbr{s_1, \dots, s_n}$ can be partitioned into two subsets of equal average.
    In Theorem~\ref{thm:avg-np-hard}, we prove that the equal average partition problem is NP-complete, a result which may be of independent interest\footnote{Given the similarity of the equal average partition problem to the classical partition problem, we would have expected to find a proof of NP-completeness elsewhere. However, we failed to find such a proof in previous work.}.
    
    Without loss of generality, we assume $s_i\in[0,1]$ for all $s_i\in\Scal$\footnote{We can always divide every element in $\Scal$ by the largest member of $\Scal$ to ensure elements fall in $[0,1]$.} and, $s_i \leq s_j$ if $i < j$. For every $s_i\in\Scal$, we set $a_{i,z_1}=s_i$, $a_{i,z_2}=1-s_i$, $\rho_i = \frac{1}{n}$, $\rho_{z_1\given i} = \alpha$, $\rho_{z_2\given i} =1-\alpha$ for $\alpha\in(0.5,0.75]$. Note that we will have that $a_i = \alpha s_i + (1-\alpha) (1-s_i) = (2\alpha -1)s_i + (1-\alpha)\in [0,1]$. Note first that for any $\Acal\in\Bcal$
    \begin{align}\label{eq:in-np-proof}
        a_{\Acal,z_1} = \frac{ \sum_{j \in \Acal} \rho_j \rho_{z_1 \given j} a_{j,z_1}}{\sum_{j \in \Acal} \rho_j \rho_{z_1 \given j}}
        =\frac{ \sum_{j \in \Acal} \frac{\alpha}{n} a_{j,z_1}}{\sum_{j \in \Acal} \frac{\alpha}{n}}
        =\frac{ \sum_{j \in \Acal}  a_{j,z_1}}{\abr{\Acal}}
        =1 -\frac{ \sum_{j \in \Acal}  (1 - a_{j,z_1})}{\abr{\Acal}}
        =1 - a_{\Acal,z_2}.
    \end{align}
    , and
    \begin{align}\label{eq:in-np-proof-2}
        a_{\Acal} = \frac{\sum_{j\in\Acal}((2\alpha -1)a_{j,z_1} + 1 - \alpha)}{\abr{\Acal}}
        = (2\alpha -1)\frac{\sum_{j\in\Acal}a_{j,z_1}}{\abr{\Acal}} + 1-\alpha
        = (2\alpha -1)a_{\Acal,z_1} + 1-\alpha
    \end{align}
    Note that, whenever we have that $a_{\Acal,z_1}\leq a_{\Acal',z_1}$, it will also hold that $a_{\Acal}<a_{\Acal'}$ as $2\alpha -1>0$.
    
    Now, assume a valid partition $\Bcal$ with $\abr{\Bcal}=2$ exists and $\Bcal = \cbr{\Acal_1,\Acal_2}$. Without loss of generality, assume $a_{\Acal_1,z_1}\leq a_{\Acal_2,z_1}$. 
    Since $\Bcal$ is a valid partition, we should have also that $a_{\Acal_1,z_2}\leq a_{\Acal_2,z_2}$, furthermore,
    \begin{align}
        a_{\Acal_1,z_1}\leq a_{\Acal_2,z_1} \Rightarrow 1 - a_{\Acal_1,z_1}\geq 1 - a_{\Acal_2,z_1} \Rightarrow a_{\Acal_1,z_2}\geq a_{\Acal_2,z_2}
    \end{align}

    Since it simultaneously holds that $a_{\Acal_1,z_2}\geq a_{\Acal_2,z_2}$ and $a_{\Acal_1,z_2}\leq a_{\Acal_2,z_2}$, a valid partition $\Bcal$ with $|\Bcal| = 2$ exists if and only if $a_{\Acal_1,z_2}= a_{\Acal_2,z_2}$ and hence $a_{\Acal_1,z_1}= a_{\Acal_2,z_1}$. As $a_{\Acal_1,z_1}$ is the average of $s_j$ for $j\in\Acal_1$ and $a_{\Acal_2,z_1}$ is the average of $s_j$ for $j\in\Acal_2$ the partition $\Bcal$ can partition $\Scal$ into two subsets of equal average.

    We now prove that if no valid partition $\Bcal$ with $\abr{\Bcal}=2$ exists, there is no way of partitioning $\Scal$ into two subsets of equal average. For the sake of contradiction, assume $\Scal$ can be partitioned into $\Scal_1$ and $\Scal_2$ with equal averages $\kappa$. Define $\Acal_1 = \cbr{i\given s_i\in\Scal_1}$ and $\Acal_2 = \cbr{j\given s_j\in\Scal_2}$. Now if we build an instance of our problem based on $\Scal$ as described before and set $\Bcal=\cbr{\Acal_1,\Acal_2}$ (clearly we have that $\Bcal$ is a partition of $\R{1}{n}$) we have that $a_{\Acal_1,z_1} = a_{\Acal_2,z_1} = \kappa$, $a_{\Acal_1,z_2} = a_{\Acal_2,z_2} = 1-\kappa$ (refer to Eq.~\ref{eq:in-np-proof}) and $a_{\Acal_1} = a_{\Acal_2} = (2\alpha-1)\kappa + (1-\alpha)$ (refer to Eq.~\ref{eq:in-np-proof-2}). As a result, we have that $\Bcal$ is a valid partition of size $2$ which is a contradiction.
    This concludes the proof.
\end{proof}

\begin{theorem}\label{thm:avg-np-hard}
    Given a set of $n$ positive integers, the problem of deciding whether it can be partitioned into two non-empty subsets of equal average is NP-complete.
\end{theorem}

\begin{proof}
    First it is easy to see that, given two subsets, we can evaluate in polynomial time their averages and check whether they are equal or not.
    Therefore, the problem belongs to NP.

    In the remainder of the proof, we will perform a reduction from the equal cardinality partition problem, which is known to be NP-complete, to the equal average partition problem.
    In the original problem, we are given a set of $n$ positive integers $\Scal$, where $n$ is an even number.
    The objective is to decide whether there exist two subsets $\Scal_1, \Scal_2 \subseteq \Scal$ such that $\Scal_1 \cup \Scal_2 = \Scal$ and $\Scal_1 \cap \Scal_2 = \emptyset$, with $|\Scal_1|=|\Scal_2|$ and $\sum_{i \in \Scal_1} i = \sum_{j \in \Scal_2} j$.

    Now, we will transform an arbitrary instance of that problem into an instance of the equal average partition problem.
    Let the set of integers be $\Scal' = \Scal \cup \{n\sigma, n\sigma\}$, where $\sigma = \sum_{k\in\Scal} k$.
    It is easy to see that the average of $\Scal'$ is equal to $\frac{(2n+1)\sigma}{n+2}$.

    We will start by showing that, if we can decide positively about that instance of the equal average partition problem, we can also decide positively about the original instance of the equal cardinality partition problem.
    Assume there exists a partition of $\Scal'$ into two sets $\Scal'_1$, $\Scal'_2$, with equal averages.
    As an intermediate result, we will show that the two copies of the number $n\sigma$ cannot belong to the same set $\Scal'_1$ or $\Scal'_2$.
    For the sake of contradiction, and without loss of generality, assume that both copies belong to $\Scal'_1$.

    In the case where $\Scal'_1 = \{n\sigma, n\sigma\}$, it holds that $\frac{\sum_{i\in\Scal'_1} i}{|\Scal'_1|} = n\sigma$ and $\frac{\sum_{i\in\Scal'_2} j}{|\Scal'_2|} = \frac{\sigma}{n}$, which is a contradiction, since the two quantities cannot be equal because of $n\geq 2$.
    In cases where $\Scal'_1$ contains at least one more element, since $\Scal'_2 \neq \emptyset$, we get that $\frac{\sum_{i\in\Scal'_1} i}{|\Scal'_1|} = \frac{2n\sigma + \kappa}{2+l}$, with $0 < \kappa < \sigma$ and $1\leq l \leq n-1$, and $\frac{\sum_{j\in\Scal'_2} j}{|\Scal'_2|} = \frac{\sigma - \kappa}{n-l}$.
    It follows that
    \begin{equation*}
        \frac{1}{n-l} \leq 1 \Rightarrow 
        \frac{\sigma - \kappa}{n-l} \leq \sigma - \kappa \Rightarrow
        \frac{\sum_{j\in\Scal'_2} j}{|\Scal'_2|} < \sigma \stackrel{(*)}{\Rightarrow}
        \frac{\sum_{j\in\Scal'_2} j}{|\Scal'_2|} < \frac{(2n+1)\sigma}{n+2} \Rightarrow
        \frac{\sum_{j\in\Scal'_2} j}{|\Scal'_2|} < \frac{\sum_{k\in\Scal'} k}{|\Scal'|},
    \end{equation*}
    where $(*)$ holds because $n>1$.
    According to Lemma~\ref{lem:avg-eq-partition}, the last inequality leads to a contradiction.
    With that, we can conclude that one copy of $n\sigma$ belongs to $\Scal'_1$ and the other one belongs to $\Scal'_2$.

    Let $\Scal_1$, $\Scal_2$ be such that $\Scal'_1 = \{n\sigma\} \cup \Scal_1$ and $\Scal'_2 = \{n\sigma\} \cup \Scal_2$.
    We will now show that $\Scal_1$ and $\Scal_2$ are a solution to the original instance of the equal cardinality partition problem, \ie, $|\Scal_1|=|\Scal_2|$ and $\sum_{i\in\Scal_1} i = \sum_{j\in\Scal_2} j$.
    It is trivial to see that $\Scal_1, \Scal_2$ have to be non-empty, otherwise the averages of $\Scal'_1$ and $\Scal'_2$ would differ.
    Since $\Scal'_1$, $\Scal'_2$ are a partition of $\Scal'$ with equal averages and because of Lemma~\ref{lem:avg-eq-partition}, we know that
    \begin{equation}\label{eq:lem_helper}
        \frac{n\sigma + \sum_{i\in\Scal_1} i}{1 + |\Scal_1|} = \frac{n\sigma + \sum_{j\in\Scal_2} j}{1 + |\Scal_2|} = \frac{(2n+1)\sigma}{n+2}.
    \end{equation}
    For the sake of contradiction, assume that either $|\Scal_1|\neq|\Scal_2|$ or $\sum_{i\in\Scal_1} i \neq \sum_{j\in\Scal_2} j$.
    For brevity, we will focus only on the two following cases, as any other case leads easily to a contradiction:
    \begin{itemize}
        \item $|\Scal_1| < |\Scal_2|$ and $\sum_{i\in\Scal_1} i < \sum_{j\in\Scal_2} j$:
        Since $\Scal_1$, $\Scal_2$ are such that $\Scal_1 \cup \Scal_2 = \Scal$, it holds that
        \begin{align*}
            & \sum_{j\in\Scal_2} j - \sum_{i\in\Scal_1} i < \sigma \stackrel{(*)}{\Rightarrow} 
            \frac{(2n+1)\sigma}{n+2}(1+|\Scal_2|) -n\sigma - \frac{(2n+1)\sigma}{n+2}(1+|\Scal_1|) + n\sigma < \sigma \Rightarrow \\
            & \frac{(2n+1)\sigma}{n+2}(|\Scal_2| - |\Scal_1|) < \sigma \Rightarrow
            (2n+1)(|\Scal_2| - |\Scal_1|) < (n+2) \stackrel{(**)}{\Rightarrow}
            2n+1 < n+2 \Rightarrow
            n < 1,
        \end{align*}
        where $(*)$ follows from Equation~\ref{eq:lem_helper}, and $(**)$ holds because $|\Scal_2|-|\Scal_1|\geq 1$.
        The last inequality is clearly a contradiction.

        \item $|\Scal_1| > |\Scal_2|$ and $\sum_{i\in\Scal_1} i > \sum_{j\in\Scal_2} j$: The proof is the symmetric version of the proof in the previous case.
    \end{itemize}
    Therefore, we can conclude that $\Scal_1$ and $\Scal_2$ are a solution to the original problem, \ie, they are a partition of $\Scal$ with equal cardinality and equal sums.

    Lastly, we will show that, if there is no partition of $\Scal'$ with equal averages, there can be no equal cardinality partition of $\Scal$ with equal sums.
    For the sake of contradiction, assume there exist $\Scal_1$, $\Scal_2$ with $|\Scal_1| =  |\Scal_2|$ and $\sum_{i\in\Scal_1} i = \sum_{j\in\Scal_2} j$.
    Then, let $\Scal'_1 = \{n\sigma\} \cup \Scal_1$ and $\Scal'_2 = \{n\sigma\} \cup \Scal_2$.
    It is easy to see that 
    \begin{equation}
        \frac{\sum_{i\in\Scal'_1} i}{|\Scal'_1|} = \frac{n\sigma + \sum_{i\in\Scal_1} i}{1 + |\Scal_1|} = \frac{n\sigma + \sum_{j\in\Scal_2} j}{1 + |\Scal_2|} = \frac{\sum_{i\in\Scal'_2} i}{|\Scal'_2|},
    \end{equation}
    which is a contradiction, since it means that $\Scal'_1$ and $\Scal'_2$ are a partition of $\Scal'$ with equal averages.

    Following the above procedure, we can decide whether the original instance of the equal-cardinality problem has a solution or not.
    As a consequence, the problem of deciding whether a set of positive integers can be partitioned into two subsets of equal average is NP-complete.
    
\end{proof}

\begin{lemma}\label{lem:avg-eq-partition}
    A set of integers $\Scal$ can be partitioned into two non-empty sets $\Scal_1$, $\Scal_2$ with equal averages $\frac{\sum_{i\in\Scal_1} i}{|\Scal_1|}=\frac{\sum_{j\in\Scal_2} j}{|\Scal_2|}$, iff $\frac{\sum_{i\in\Scal_1} i}{|\Scal_1|}=\frac{\sum_{k\in\Scal} k}{|\Scal|}$, with $|\Scal_1|\subset |\Scal|$.
\end{lemma}

\begin{proof}
    First, assume there is such a partition of $\Scal$ into $\Scal_1$, $\Scal_2$, with equal averages.
    It holds that
    \begin{align*}
        \frac{\sum_{i\in \Scal_1} i }{|\Scal_1|} = \frac{\sum_{k\in \Scal} k - \sum_{i\in \Scal_1} i}{|\Scal| - |\Scal_1|} \Rightarrow % \\
        \left(|\Scal| - |\Scal_1|\right) \sum_{i\in \Scal_1} i = |\Scal_1| \left(\sum_{k\in \Scal} k - \sum_{i\in \Scal_1} i \right) & \Rightarrow 
        |\Scal|\sum_{i\in \Scal_1} i = |\Scal_1| \sum_{k\in \Scal} k  \\
        & \Rightarrow \frac{\sum_{i\in\Scal_1} i}{|\Scal_1|} = \frac{\sum_{k\in\Scal} k}{|\Scal|},
    \end{align*}
    where $\Scal_1\subset \Scal$ because $\Scal_2 \neq \emptyset$.

Now, assume there exists a set $\Scal_1 \subset \Scal$, such that $\frac{\sum_{i\in\Scal_1} i}{|\Scal_1|}=\frac{\sum_{k\in\Scal} k}{|\Scal|}$ and let $\Scal_2 = \Scal \setminus \Scal_1$.
It is easy to see that
\begin{equation*}
        \frac{\sum_{j\in \Scal_2} j }{|\Scal_2|} =
        \frac{\sum_{k\in \Scal} k - \sum_{i\in \Scal_1} i}{|\Scal| - |\Scal_1|} =
        \frac{\sum_{k\in \Scal} k - \frac{|\Scal_1|}{|\Scal|}\sum_{k\in \Scal} k}{|\Scal|\left(1 - \frac{|\Scal_1|}{|\Scal|}\right)} = 
        \frac{\sum_{k\in \Scal} k }{|\Scal|},
\end{equation*}
and therefore, the sets $\Scal_1$, $\Scal_2$ consist a partition of $\Scal$ with equal averages.
\end{proof}

\subsection{Proof of Proposition~\ref{prop:monotone-contiguous}}
% \begin{lemma}\label{lem:necessity-monotone-contiguous}
% If $f_\Bcal$ is monotone with respect to $f$, then $\Bcal$ is a contiguous partition on $\Rf$.
% \end{lemma}
% \begin{proof}
We first prove the sufficient condition, \ie, we prove that, if $f_\Bcal$ is monotone with respect to $f$, then $\Bcal$ is a contiguous partition on $\{1, \ldots n\}$. 
The proof is by contradiction. 
Assume $\Bcal$ is not a contiguous partition, \ie, there exists $x_1,x_2,x_3\in\Xcal$ such that $i(x_1)<i(x_2)<i(x_3)$ and $i(x_1)\sim_\Bcal i(x_3)$ while $i(x_1)\not\sim_\Bcal i(x_2)$. 
If $a_{[i(x_1)]}>a_{[i(x_2)]}$, then $f_\Bcal(x_1)>f_\Bcal(x_2)$, however, since $f(x_1)<f(x_2)$, this leads to a contradiction with the monotonicity assumption. 
On the other hand, if $a_{[i(x_1)]}<a_{[i(x_2)]}$, then $f_\Bcal(x_3)<f_\Bcal(x_2)$ since $i(x_1)\sim_\Bcal i(x_3)$ and thus $a_{[i(x_3)]}<a_{[i(x_2)]}$, however, this leads again to a contradiction with the monotonicity assumption. This proves that $\Bcal$ must be a contiguous partition.
% \end{proof}
%
% \begin{lemma}\label{lem:suff-monotone-contiguous}
% If $\Bcal$ is a contiguous partition on $\Rf$, then $f_\Bcal$ is monotone with respect to $f$. 
% \end{lemma}
% \begin{proof}

Next, we prove the necessary condition, \ie, we prove that, if $\Bcal$ is a contiguous partition on $\{1, \ldots n\}$, then $f_\Bcal$ is monotone with respect to $f$. 
For any $x_1, x_2 \in \Xcal$ such that $f(x_1) < f(x_2)$, we have that:
\begin{align*}
f_\Bcal(x_1) = a_{[i(x_1)]} = \frac{ \sum_{l \in 
[i(x_1)]} 
a_l \rho_{l}}{\sum_{l\in[i(x_1)]}\rho_{l}} 
\leq \frac{ \sum_{l \in [i(x_2)]} a_l
\rho_{l}}{\sum_{l\in[i(x_2)]}\rho_{l}} = a_{[i(x_2)]}=f_\Bcal(x_2).
\end{align*}
where the inequality is due to Lemma~\ref{lem:first-monotone-contiguous} below and the fact that the weighted average of a set of numbers is lower and upper bounded by the smallest and largest element of the set respectively.
%
% \end{proof}

\begin{lemma}\label{lem:first-monotone-contiguous}
Let $f$ be a classifier with $\Rf = \cbr{a_1, \ldots, a_n}$, $\Bcal$ be a contiguous partition on $\{1, \ldots, n\}$ and $x_1, x_2 \in \Xcal$. If $i(x_1) < i(x_2)$ and $i(x_1) \not \sim_{\Bcal} i(x_2)$, then, for every $k \in [i(x_1)]$ and $k' \in [i(x_2)]$, it holds that $k < k'$.
\end{lemma}
\begin{proof}
%
% We prove that the largest index in $[i(x_1)]$ is smaller than the smallest index in $[i(x_2)]$ which will immediately prove the proposition. The proof is by contradiction. Assume $l = \max \{ k \given k \in [i(x_1)] \}$ and $s = \min\{ k \given k \in [i(x_2)]\}$ and $l > s$. Note that in this case it cannot simultaneously hold that $i(x_1) = l$ and $i(x_2)=s$ since we have that $i(x_1)<i(x_2)$. We first assume $i(x_1)\neq l$. In this case for $x_3,x_4\in\Xcal$ such that $i(x_3)=l$ and $i(x_4) = s$, if $i(x_2)<i(x_3)$, then it holds that $i(x_1)<i(x_2)<i(x_3)$ and $i(x_1)\sim_\Bcal i(x_3)$ while $i(x_1)\not\sim_\Bcal i(x_2)$ which is in contradiction with $\Bcal$ being a contiguous partition. If $i(x_3)<i(x_2)$, then it holds that $i(x_4)<i(x_3)<i(x_2)$ and $i(x_2)\sim_\Bcal i(x_4)$ while $i(x_4)\not\sim_\Bcal i(x_3)$ which is a contradiction. If $i(x_1)=s$ a similar reasoning leads to the same contradiction. Hence, we should have that $l<s$ which completes the proof.
%
% \vspace{5mm}
To prove the lemma, we just need to prove that the largest index in $[i(x_1)]$ is smaller than the smallest index in $[i(x_2)]$. 
The proof is by contradiction. 
% a a'
Let $l = \max \{ k \given k \in [i(x_1)] \}$ and $s = \min\{ k \given k \in [i(x_2)]\}$ and assume that $l > s$. 
Then, it cannot simultaneously hold that $i(x_1) = l$ and $i(x_2) = s$ since we have that $i(x_1) < i(x_2)$. 
Assume first that $i(x_1) \neq l$, and 
take $x_3, x_4 \in \Xcal$ such that $i(x_3) = s$ and $i(x_4) = l$. %, $i(x_2)=a'$ and $i(x_1) = a$. 
% x_1, x_2, x_3...a_i a' a --> x_3, x_2, x_1...i s l
%
% a_i < a'
If $i(x_3) < i(x_1)$, then it holds that $i(x_3) < i(x_1) < i(x_2)$, however, since $i(x_2) \sim_\Bcal i(x_3)$ and $i(x_1) \not \sim_\Bcal i(x_2)$, this leads to a contradiction with the assumption that $\Bcal$ is contiguous.
If $i(x_3) > i(x_1)$, then it holds that $i(x_1) < i(x_3) < i(x_4)$, however, since $i(x_1) \sim_\Bcal i(x_4)$ while $i(x_3) \not \sim_\Bcal i(x_4)$, this also leads to a contradiction with the assumption that $\Bcal$ is contiguous. 
If one assumes instead that $i(x_1) = l$, a similar reasoning using $i(x_2)$ and $i(x_4)$ leads to a contradiction too. 
This completes the proof. 
% Hence, we should have that $a<a'$ which completes the proof.
\end{proof}

\subsection{Proof of Proposition~\ref{prop:pav-within-group-monotonicity}}
We prove by contradiction. Assume there exist violations of within-group monotonicity. We first define the nearest violating triplet, $(l,r,z)$, as:
\begin{align*}
    (l,r,z) = \argmin_{\cbr{(i,j,z)\given i,j\in\Rbf, i<j, z\in\Zcal}} \abr{j-i} \text{ such that } a_{\Acal_i,z}>a_{\Acal_{j},z}
\end{align*}
If $r = l+1$ then it contradicts with the assumption that no monotonicity violations occur between adjacent cells. If $r\neq l+1$, there exists $i\in\Rbf$ such that $l\leq i\leq r$ and it does not happen simultaneously that $i=l$ and $i=r$. Then it should hold that $a_{\Acal_l,z}\leq a_{\Acal_i,z}\leq a_{\Acal_r,z}$ since otherwise either of $(l,i,z)$ or $(i,r,z)$ is the nearest violating triplet. In this case however, $a_{\Acal_l,z}\leq a_{\Acal_r,z}$ which is a contradiction with it being a violating triplet. As a result, no such triplet can exist and $f_\Bcal$ is within-group monotone.

\subsection{Proof of Lack of Local Optimality of the Pool Adjacent Violators (PAV) Algorithm} \label{app:pav-suboptimality}
Let $\Rf = \{a_1, a_2, a_3\}$, $\Zcal = \cbr{z_1,z_2}$ and 
$\rho_i\rho_{z\given i}=\frac{1}{6}$ for all $i \in \{1, 2, 3\}$ and 
$z\in\Zcal$.
Further, let $a_{1,z_2}=a_{2,z_1}=a_{3,z_2}=\alpha$, 
$a_{1,z_1}=2\alpha$, $a_{2,z_2}=3\alpha$ and $a_{3,z_1}=4\alpha$, 
where $\alpha\in[0,0.25]$. 
First, we note that, by construction, it holds that $a_1 = \frac{3}{2}\alpha<a_2=2\alpha<a_3=\frac{5}{2}\alpha$.
Now, since $a_{1,z_1}>a_{2,z_1}$, Algorithm~\ref{alg:pav} first 
merges these two bins, % when it reaches the second bin. 
then, since $a_{\cbr{1,2},z_2}>a_{\cbr{3},z_2}$, it merges 
all the three bins together and finally it terminates, 
returning $\Bcal = \cbr{\cbr{1,2,3}}$.
%
% From this point on since no violations exist the algorithm does 
%  merge any other bins together and outputs 
% $\Bcal_{pav}=\cbr{\cbr{1,2,3},\cbr{4},\cbr{5}, \ldots, 
% \cbr{n}}$. 
%
However, since it holds that $a_{1,z_1}<a_{\cbr{2,3},z_1}$ and $a_{1,z_2}<a_{\cbr{2,3},z_2}$, it clearly holds that the partition $\Bcal' = \cbr{\cbr{1},\cbr{2,3}}$ induces a classifier $f_{\Bcal'}$ that is within-group monotone and it readily follows that $f_{\Bcal'}$ dominates $f_{\Bcal}$.

\subsection{Proof of Lemma~\ref{lem:recursive-approach}}
We first prove the sufficient condition, \ie, we prove, for any $\Bcal\in\Bscr_{l,r}$, $\exists k<l$ such that $\Bcal\setminus\cbr{\R{l}{r}} \in\Bscr_{k,l-1}$ and $a_{\cbr{k, \ldots, l-1},z}\leq a_{\cbr{l, \ldots, r},z}$ $\forall z \in \Zcal$.
Let $\Bcal' = \Bcal\setminus\cbr{\cbr{l, \ldots, r}}$. 
To this end, we start by proving by contradiction that $\exists k<l$ such that $\Bcal'\in\Bscr_{k,l-1}$. 
Since the partition $\Bcal$ covers $\R{1}{r}$, we have that the last cell of $\Bcal'$ contains bin $l-1$. 
Assume $\Bcal' \not\in \cup_{k=1}^{l-1}\Bscr_{k,l-1}$. 
Then, there must exist $\Acal,\Acal'\in\Bcal'$ and $z\in\Zcal$ such that $a_{\Acal}<a_{\Acal'}$ and $a_{\Acal,z}>a_{\Acal,z'}$. 
However, since $\Bcal' \subset \Bcal$, it also holds that $\Acal, \Acal' \in \Bcal$ and $f_{\Bcal}$ cannot be within-group monotone on $\cup_{i\leq r}\Xcal_i$, leading to a contradiction. 
Therefore, it must hold that $\Bcal'\in \cup_{k=1}^{l-1}\Bscr_{k,l-1}$. 
Now, to prove that, if $\Bcal'\in \cup_{k=1}^{l-1}\Bscr_{k,l-1}$ and $\Bcal \in \Bscr_{l,r}$, then it must hold that $a_{\cbr{k, \ldots, l-1},z}\leq a_{\cbr{l, \ldots, r},z}$ $\forall z \in \Zcal$, we resort to Lemma~\ref{lem:lem-recursive-approach}.

We next prove the necessary condition, \ie, we prove that, given any $\Bcal \in \Bscr_r$, if $\exists k<l$ such that $\Bcal\setminus\cbr{\R{l}{r}} \in\Bscr_{k,l-1}$ and $a_{\cbr{k, \ldots, l-1},z}\leq a_{\cbr{l, \ldots, r},z}$ $\forall z \in \Zcal$ then $\Bcal \in \Bscr_{l,r}$. 
Let $\Bcal' = \Bcal\setminus\cbr{\cbr{l, \ldots, r}}$. Since $\Bcal'\in\Bscr_{k,l-1}$, we know that no violations of within-group monotonicity occurs on $\cup_{i\leq l-1}\Xcal_i$. 
Now, we prove that there are no violations of within-group monotonicity between $\R{l}{r}$ and any $\Acal \in \Bcal'$. 
By assumption, we know that there are not violations of within-group monotonicity between $\R{l}{r}$ and $\R{k}{l-1}$.
Then, we prove by contradiction that there are not violations between $\R{l}{r}$ and any $\Acal \in \Bcal' \setminus \cbr{\R{k}{l-1}}$. 
For any $\Acal \in \Bcal' \setminus \cbr{\R{k}{l-1}}$, it follows from Proposition~\ref{prop:monotone-contiguous} that $a_\Acal<a_{\R{k}{l-1}}$ and $a_{\Acal}<a_{\R{l}{r}}$. 
Now, assume there exists $\Acal \in \Bcal'\setminus \cbr{\R{k}{l-1}}$, $z\in\Zcal$ such that $a_{\Acal,z}>a_{\R{l}{r},z}$. 
Since, by assumption, we have that $a_{\R{k}{l-1},z}\leq a_{\R{l}{r},z}$, it should hold that $a_{\R{k}{l-1},z}<a_{\Acal,z}$, which contradicts with the assumption that $\Bcal'\in\Bscr_{k,l-1}$, leading
to a contradiction. 
% or is within-group monotone on $\cup_{i\leq l-1}\Xcal_i$. 
%
This proves that $\Bcal\in\Bscr_{l,r}$.

% \begin{proposition}\label{prop:necessary-recursive-approach}
% If the classifier $f_{\Bcal\cup\cbr{\cbr{l, \ldots, r}}}$ such that $\Bcal\in\Bscr_{k,l-1}$ is within-group monotone on $\cup_{i\leq r}\Xcal_i$ then it holds that $a_{\cbr{k, \ldots, l-1},z}\leq a_{\cbr{l, \ldots, r},z}$ for all $z\in\Zcal$. 
% \end{proposition}
%
% \begin{proof}
\begin{lemma}\label{lem:lem-recursive-approach}
Let $\Bcal = \Bcal'\cup\cbr{\R{l}{r}} \in \Bcal_{l,r}$ and $\Bcal'\in\Bscr_{k,l-1}$ with $k<l$. Then, it must hold that $a_{\cbr{k, \ldots, l-1},z}\leq a_{\cbr{l, \ldots, r},z}$ $\forall z \in \Zcal$.
\end{lemma}
\begin{proof}
% We first prove the sufficient condition, \ie,  we prove that if the classifier $f_{\Bcal\cup\cbr{\cbr{l, \ldots, r}}}$ such that $\Bcal\in\Bscr_{k,l-1}$ is within-group monotone on $\cup_{i\leq r}\Xcal_i$ then it holds that $a_{\cbr{k, \ldots, l-1},z}\leq a_{\cbr{l, \ldots, r},z}$ for all $z\in\Zcal$. 
% Let $\Bcal = \Bcal'\cup\cbr{\cbr{l, \ldots, r}}$. 
%
Since $\Bcal'\in\Bscr_{k,l-1}$, we know that $\R{k}{l-1} \in \Bcal'$.
% and hence $\cbr{\R{l}{r}}\in\Bcal$. 
Moreover, it follows from Proposition~\ref{prop:monotone-contiguous} that $f_{\Bcal}$ is monotone with respect to $f$ and hence, since $k<l$ and $k \not \sim_{\Bcal}l$, we have that $a_{\R{k}{l-1}}< a_{\R{l}{r}}$. 
Further, since $\Bcal \in \Bcal_{l, r}$, we have that, for every $\Acal,\Acal'\in\Bcal$ such that $a_\Acal<a_{\Acal'}$, it holds that $a_{\Acal,z}\leq a_{\Acal',z}$ for all $z\in\Zcal$. 
Thus, it also holds that $a_{\R{k}{l-1},z}\leq a_{\R{l}{r},z}$ 
for all $z\in\Zcal$.
\end{proof}

\subsection{Proof of Theorem~\ref{thm:optimal}}
To prove that Algorithm~\ref{alg:optimal} returns the optimal partition $\Bcal^{*}$, we just
need to prove that, for each $l, r \in \{1, \ldots, n\}$, the partition $\Bcal_{l, r}$ the algorithm finds 
is optimal, \ie, $\Bcal_{l, r} = \Bcal^{*}_{l, r}$.
In what follows, we prove this by induction.

% Based on Lemma.~\ref{lem:thm-optimal} we have that Algorithm.~\ref{alg:optimal} finds all 
% optimal partitions $\Bcal_{l,r}$ for $l,r\in\R{1}{n}$ such that $l\leq r$. It is clear that the 
% optimal partition is the partition among $\cup_{l'\in\R{1}{n}}{\Bcal_{l,n}}$ with the maximum 
% size, \ie, $\Bcal^* = \Bcal_{l,n} \text{ if } l = \argmax_{l'\in\cbr{2, \ldots, n}} 
% \abr{\Bcal_{l',n}}$ as returned by the algorithm.

% \begin{lemma}\label{lem:thm-optimal}
%
% For every $l,r\in\R{1}{n}$ such that $l\leq r$, Algorithm~\ref{alg:optimal} stores the optimal 
% partition of $\Bscr_{l,r}$ in $\Bcal_{l,r}$.
% \end{lemma}
%
% \begin{proof}
%
For the base cases, we have that $\Bcal_{1, r} = \cbr{\R{1}{r}}$ are clearly optimal since $\Bscr_{1,r}$ only contains $\cbr{\R{1}{r}}$ for all $r \in \{1, \ldots, n\}$. 
As the induction hypothesis, assume that, for any $l' < l$ and $r' < r$, the partition $\Bcal_{l',r'}$ the algorithm finds is optimal. 
Moreover, let $\Scal_{l,r} = \cbr{k \given k<l, a_{\cbr{k, \ldots, l-1},z}\leq a_{\cbr{l, \ldots, r},z}~\forall{z\in\Zcal}}$.
Then, for $(l, r)$, we need to show that $\Bcal_{l, r} = \Bcal_{k^*,l-1} \cup \cbr{\R{l}{r}}$,
with $k^*= \argmax_{k\in\Scal_{l,r}}\abr{\Bcal_{k,l-1}}$, is optimal.

To this end, we first show that $f_{\Bcal_{l,r}}$ is within-group monotone on $\cup_{i\leq r}\Xcal_i$, \ie, $\Bcal_{l, r} \in \Bscr_{l,r}$. 
%
% Let $\Bcal = \Bcal_{k^*,l-1}\cup\cbr{\R{l}{r}}$. 
We have that, by the induction hypothesis, $\Bcal_{k^*,l-1}\in\Bscr_{k^*,l-1}$ and, 
by definition, $k^{*} \in \Scal_{l,r}$. Then, it follows directly from 
Lemma~\ref{lem:recursive-approach} that $f_{\Bcal} \in\Bscr_{l,r}$.
Next, we show that $\Bcal_{l, r} = \argmax_{\Bcal \in \Bscr_{l,r}} \abr{\Bcal}$. % is optimal. 
Using again Lemma~\ref{lem:recursive-approach}, we have that, for any $\Bcal \in \Bscr_{l,r}$,
it holds that $\Bcal = \Bcal'\cup\cbr{\R{l}{r}}$, with $\Bcal'\in\Bscr_{k,l-1}$, for some $k \in \Scal_{l, r}$.
As a result, since $\abr{\Bcal'\cup\cbr{\R{l}{r}}} = \abr{\Bcal'}+1$, it suffices to find $\Bcal' = \argmax_{\Bcal''\in\cup_{k\in\Scal_{l,r}}\Bscr_{k,l-1}}\abr{\Bcal''}$.
Now, by the induction hypothesis, we know that, for each $\Bscr_{k,l-1}$, $\Bcal_{k, l-1}$ is the optimal partition. 
% it suffices to consider the optimal partition we only need to consider the optimal partition 
% the optimal partition has the most number of 
% cells, it is clear that we do not need to consider those partitions which are not optimal, \ie, 
% we need to consider those partitions in $\cup_{k\in\Scal_{l,r}}\Bcal_{k,l-1}$. 
%
% Now, as for every $k \in \Scal_{l,r}$, we have that $k < l$, then by the induction hypothesis, 
% we know that $\Bcal_{k,l-1}$ is the optimal partition in $\Bscr_{k,l-1}$. 
%
Then, since $k^*= \argmax_{k\in\Scal_{l,r}}\abr{\Bcal_{k,l-1}}$, we can conclude that $\Bcal_{l, r}$ is optimal.
% \end{proof}
%

\subsection{Proof of Lemma~\ref{lem:multical-recursive}}
We first prove the sufficient condition, \ie, we prove that, given any $\Bcal \in \Bscr_r$, if it holds that $f_{\Bcal}$ is within-group calibrated on $\cup_{i \leq r}\Xcal_i$ then $\exists l < r$ such that $\Bcal \backslash \cbr{\R{l}{r}} \in \Bscr_{l-1}$ and $f_{\Bcal \backslash \cbr{\R{l}{r}}}$ is within-group calibrated on $\cup_{i \leq l-1} \Xcal_i$ and $a_{\R{l}{r},z} = a_{\R{l}{r}}$ for all $z\in\Zcal$.
Let $\Bcal' = \Bcal\setminus \cbr{\R{l}{r}}$. Since $\Bcal$ covers $\R{1}{r}$, then it holds that $\Bcal'$ covers $\R{1}{l-1}$ and hence $\Bcal'\in\Bscr_{l-1}$.
Since $\Bcal'\subset\Bcal$ and $f_\Bcal$ is within-group calibrated on $\cup_{i \leq r}\Xcal_i$, then it holds that $f_{\Bcal'}$ is within-group calibrated on $\cup_{i \leq l-1}\Xcal_i$. Finally, since $\R{l}{r} \in \Bcal$, it also holds that $a_{\R{l}{r},z} = a_{\R{l}{r}}$.

% We first prove the sufficient condition, \ie, we prove that If the classifier $f_{\Bcal\cup\cbr{\R{l}{r}}}$ with $\Bcal\in\Bscr_{l-1}$ such that $f_\Bcal$ is within-group calibrated on $\cup_{i<l}\Xcal_i$ is within-group calibrated on $\cup_{i\leq r}\Xcal_i$, then it holds that $a_{\R{l}{r},z} = a_{\R{l}{r}}$ for all $z\in\Zcal$. Since $f_{\Bcal\cup\cbr{\R{l}{r}}}$ is within-group calibrated, based on Definition.~\ref{def:within-group-calibration}, it should hold for every $\Acal\in\Bcal\cup\cbr{\R{l}{r}}$ that $a_{\Acal,z} = a_\Acal$ for all $z\in\Zcal$. As a result for the cell $\R{l}{r}$ it should hold that $a_{\R{l}{r},z} = a_{\R{l}{r}}$.

Next, we prove the necessary condition, \ie, given any $\Bcal\in\Bscr_r$, if $\exists l < r$ such that $\Bcal \backslash \cbr{\R{l}{r}} \in \Bscr_{l-1}$ and $f_{\Bcal \backslash \cbr{\R{l}{r}}}$ is within-group calibrated on $\cup_{i \leq l-1} \Xcal_i$ and $a_{\R{l}{r},z} = a_{\R{l}{r}}$ $\forall z\in\Zcal$ then $f_{\Bcal}$ is within-group calibrated on $\cup_{i \leq r}\Xcal_i$.
We need to show that, for every $\Acal\in\Bcal$, it holds that $a_{\Acal,z} = a_\Acal$. 
Let $\Bcal' = \Bcal\setminus \cbr{\R{l}{r}}$. For every $z \in \Zcal$,  it holds by assumption that $a_{\Acal,z} = a_{\Acal}$ $\forall \Acal\in\Bcal'$ and $a_{\R{l}{r},z} = a_{\R{l}{r}}$. 
As a result, $f_{\Bcal}$ is within-group calibrated on $\cup_{i\leq r}\Xcal_i$.

% We next prove the necessary condition, \ie, we prove that If it holds that $a_{\R{l}{r},z} = a_{\R{l}{r}}$ for all $z\in\Zcal$, then the classifier $f_{\Bcal\cup\cbr{\R{l}{r}}}$ for $\Bcal\in\Bscr_{l-1}$ such that $f_\Bcal$ is within-group calibrated on $\cup_{i<l}\Xcal_i$, is within-group calibrated on $\cup_{i\leq r}\Xcal_i$. We show that for every $\Acal\in\Bcal\cup\cbr{\R{l}{r}}$ it holds that $a_{\Acal,z} = a_\Acal$. If $\Acal\in\Bcal$ it holds from the fact that $f_\Bcal$ is within-group calibrated on $\cup_{i<l}\Xcal_i$ that $a_{\Acal,z} = a_{\Acal}$ and by assumption we have that $a_{\R{l}{r},z} = a_{\R{l}{r}}$. As a result, $f_{\Bcal\cup\cbr{\R{l}{r}}}$ is within-group calibrated on $\cup_{i\leq r}\Xcal_i$.

\subsection{Proof of Theorem~\ref{thm:multicalibrated}}
To prove that Algorithm~\ref{alg:multicalibrated} returns the optimal $\Bcal_{\text{cal}}^{*}$, if a solution exists, we just need to prove that, for every $r \in \{1, \ldots, n\}$, the partition $\Bcal_{\text{cal}, r}$ the algorithm finds is optimal, \ie, $\Bcal_{\text{cal}, r} = \Bcal_{\text{cal}, r}^{*}$. 
In what follows, we prove this by induction.

% Based on Lemma.~\ref{lem:thm-multicalibrated} we have that 
% Algorithm.~\ref{alg:multicalibrated} finds all optimal partitions 
% $\Bcal_{cal}^r$ that result in within-group calibrated classifiers on 
% $\cup_{i\leq r}\Xcal_i$. The optimal such partition on $\Xcal$ is hence % $\Bcal_{cal}^n$ as returned by the algorithm.
% \begin{lemma}\label{lem:thm-multicalibrated}
% For any $r\in\R{1}{n}$, Algorithm.~\ref{alg:multicalibrated} finds 
% $\Bcal_{cal}^r$ which is the optimal partition in $\Bscr_{r}$ that 
% results in a within-group calibrated classifier.
% \end{lemma}
%
% \begin{proof}
% The proof is by induction. 
For the base case ($r=1$), we have that $\Bcal_{\text{cal}, 1} = \cbr{\cbr{a_1}}$ iff, for all $z\in\Zcal$ with $\rho_{z\given 1}>0$, it holds that $a_{1,z}= a_{1}$. This is clearly optimal since $\Bscr_1$ only contains $\cbr{\cbr{a_1}}$.
Otherwise, it holds that $\Bcal_{\text{cal},1} = \emptyset$. 
%
% as set also by the algorithm. 
%
As the induction hypothesis, assume that, for any $r' < r$, the partition $\Bcal_{\text{cal}, r'}$ the algorithm finds is either 
the optimal partition or, if there is no solution, an empty partition.
Moreover, let $\Scal_r = \cbr{i\in\R{2}{r}\given a_{\R{i}{r},z} = a_{\R{i}{r}}~\forall z\in\Zcal}$.
Then, for $r$, we distinguish between two cases.
If $\Bcal_{\text{cal},r'}$ is empty for all $r'<r$, we again 
distinguish between two cases. 
If $a_{\{1, \ldots, r\}} \neq a_{\{1, \ldots, r\}, z}$ $\forall z \in \Zcal$, it means that $\Bcal_{\text{cal}, r} = \{\{1, \ldots, r\}\}$ is the only partition in $\Bscr_r$ that is within-group calibrated and thus it is optimal.
Otherwise, we can conclude that no partition $\Bcal \in \Bscr_r$ is within-group calibrated and thus $\Bcal_{\text{cal}, r} = \emptyset$. 
Now, if $\Bcal_{\text{cal},r'}$ is not empty for some $r' < r$, we need to show that $\Bcal_{\text{cal}, r} = \Bcal_{\text{cal}, k^*-1} \cup \cbr{\R{k^*}{r}}$, with $k^{*} = \argmax_{k \in \Scal_r} |\Bcal_{\text{cal}, k-1}|$, is optimal. 

To this end, we first show that $f_{\Bcal_{\text{cal}, r}}$ is within-group calibrated on $\cup_{i\leq r}\Xcal_i$. 
Using the induction hypothesis and the fact that $k^*\leq r$, we have that $\Bcal_{\text{cal}, k^*-1}$ is the optimal partition in $\Bscr_{k^*-1}$.
As a result, it follows from Lemma~\ref{lem:multical-recursive} that $f_{\Bcal_{\text{cal}, r}}$ is within-group calibrated on $\cup_{i\leq r}\Xcal_i$.
Next, we show that $\Bcal_{\text{cal}, r} = \argmax_{\Bcal \in \Bscr_{r}} |\Bcal|$ among those partitions $\Bcal$ such that $f_\Bcal$ is within-group calibrated. 
Using again Lemma~\ref{lem:multical-recursive}, we have that, for
any $\Bcal$ such that $f_{\Bcal}$ is within-group calibrated, it holds
that $\Bcal = \Bcal'\cup\cbr{\R{k}{r}}$, with $\Bcal'\in\Bscr_{k-1}$, for some $k \in \Scal_r$.
%
% only for those $k\in\Scal_r$ we have that for a $\Bcal'\in\Bscr_{k-
% 1}$ such that $f_{\Bcal'}$ is within-group calibrated, we have that 
% $f_{\Bcal'\cup\cbr{\R{k^*}{r}}}$ is within-group calibrated. 
%
As a result, since $\abr{\Bcal} = \abr{\Bcal'} + 1$, it suffices 
to find $\Bcal' = \argmax_{\Bcal'' \in \cup_{k \in \Scal_r} \Bscr_{k-1}} \abr{\Bcal''}$ such that $f_{\Bcal''}$ is within-group calibrated. 
%
% As the optimal partitions are those with maximum number of bins, we 
% need to only consider the optimal partitions among such partitions 
% which means that we can reduce our search to only
Now, by the induction hypothesis, we know that, for each $\Bscr_{k-1}$,
$\Bcal_{k-1}$ is the optimal partition. Then, since $k^* = \argmax_{k\in\Scal_r}\abr{\Bcal_{\text{cal},k-1}}$, we can conclude
that $\Bcal_{\text{cal}, r}$ is optimal.

\subsection{Proof of Proposition~\ref{prop:multicalibrated-bin-diff}}
We prove by contradiction. Assume there exists a $\Bcal\in\Bscr$ such that $f_\Bcal$ is within-group calibrated.
Then, for every $\Acal\in\Bcal$, it must hold that $a_{\Acal,z} = a_{\Acal,z'} = a_\Acal$. 
Consider an arbitrary cell $\Acal \in \Bcal$.
We have that 
\begin{align*}
    & a_{\Acal,z} = \frac{\sum_{j \in \Acal} \rho_j \rho_{z \given j} a_{j,z}}{\sum_{j \in \Acal} \rho_j \rho_{z \given j}}
    \overset{(i)}{=}\frac{\sum_{j \in \Acal} \rho_j \rho_{z' \given j} a_{j,z}}{\sum_{j \in \Acal} \rho_j \rho_{z' \given j}}
    \overset{(ii)}{<} \frac{\sum_{j \in \Acal} \rho_j \rho_{z' \given j} a_{j,z'}}{\sum_{j \in \Acal} \rho_j \rho_{z' \given j}} = a_{\Acal,z'}
    % \overset{ii}{<} \frac{\sum_{j \in\Acal} \rho_j \rho_{z_1 \given m} a_{j}}{\sum_{j \in \Acal} \rho_j \rho_{z_1 \given j}} 
\end{align*}
% and 
% \begin{align*}
%     & a_{\Acal,z'} = \frac{\sum_{j \in \Acal} \rho_j \rho_{z' \given j} a_{j,z'}}{\sum_{j \in \Acal} \rho_j \rho_{z' \given j}}
%     \overset{(ii)}{>} \frac{\sum_{j \in \Acal} \rho_j \rho_{z' \given j} a_{m,z}}{\sum_{j \in \Acal} \rho_j \rho_{z' \given j}}=a_{m,z},
%     % \overset{ii}{<} \frac{\sum_{j \in\Acal} \rho_j \rho_{z_1 \given m} a_{j}}{\sum_{j \in \Acal} \rho_j \rho_{z_1 \given j}} 
% \end{align*}
%
where $(i)$ follows from the fact that $\rho_{z\given i} = \rho_{z'\given i}$ for all $i\in\Rf$ and $(ii)$ follows
from the fact that, by assumption, $a_{i,z}<a_{i,z'}$ for all $i\in\R{1}{n}$.
As an immediate consequence, we have that $a_{\Acal,z}<a_{\Acal}<a_{\Acal,z'}$,
contradicting the within-group calibration property.

\clearpage
\newpage

\section{Additional Experiments}

\subsection{Screening Classifiers Induced by the Partitions Found by Algorithms~\ref{alg:pav},~\ref{alg:optimal} and \ref{alg:multicalibrated}} \label{app:exp-partitions}
%
% \xhdr{A closer look at the partitions found by each algorithm.}
In this section, we take a closer look at all the quality score values $a = \Pr(Y = 1 \given f(X) = a)$ and group 
conditional score values $a_z = \Pr(Y = 1 \given f(X) = a, Z = z)$ of both the original classifier $f$ and the
modified classifiers $f_{\Bcal}$ induced by the partitions $\Bcal$ found by Algorithms~\ref{alg:pav},~\ref{alg:optimal} and \ref{alg:multicalibrated}.
Figure~\ref{fig:exp_violations_1} summarizes the results for one experiment with a classifier $f$ with $n = 15$,
which reveal several interesting findings.
\begin{figure*}[h]
\centering
    \subfloat[Citizenship status $(Z)$]{
	\includegraphics[width=0.5\textwidth]{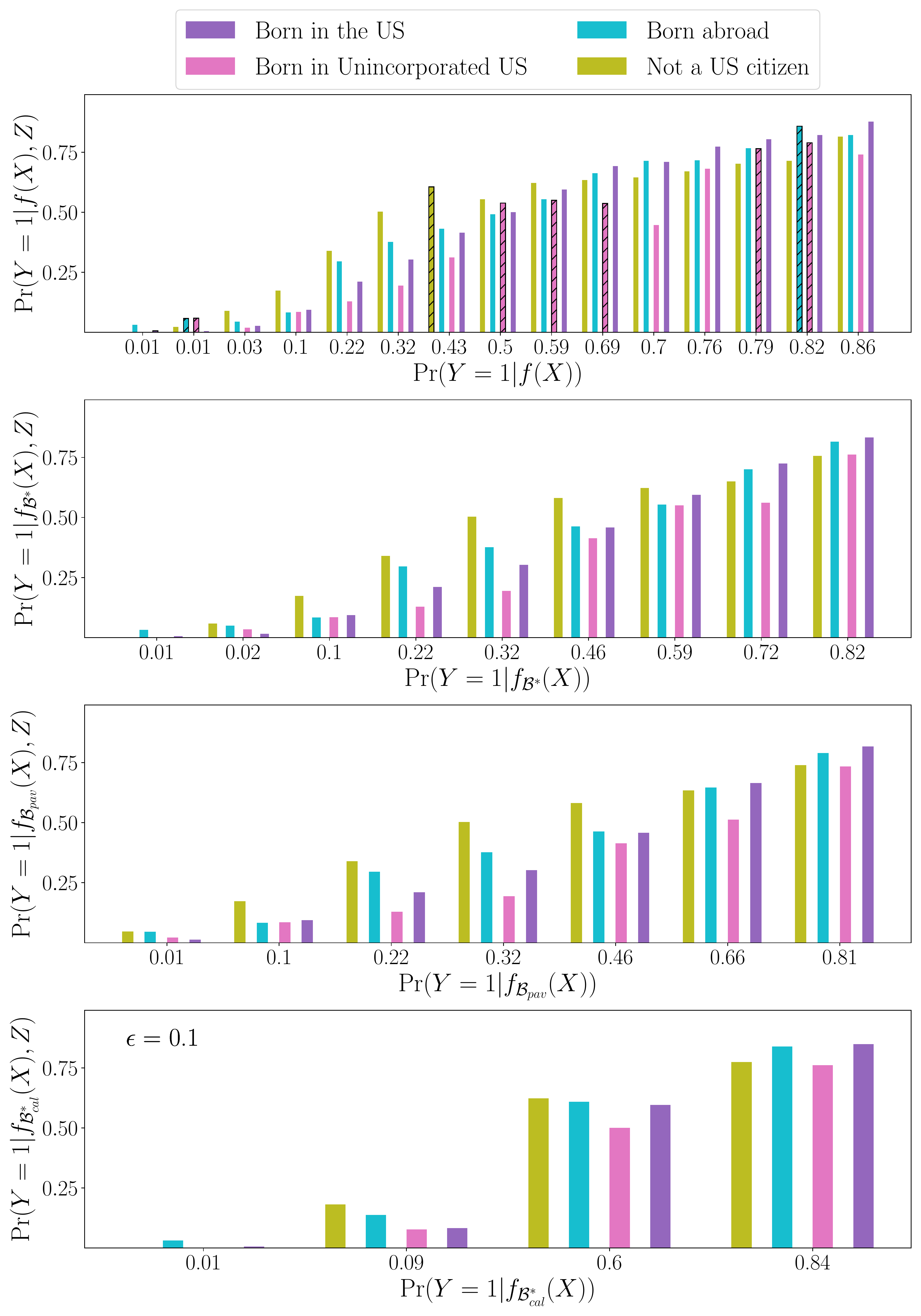}}
    \subfloat[Race code $(Z)$]{
	\includegraphics[width=0.5\textwidth]{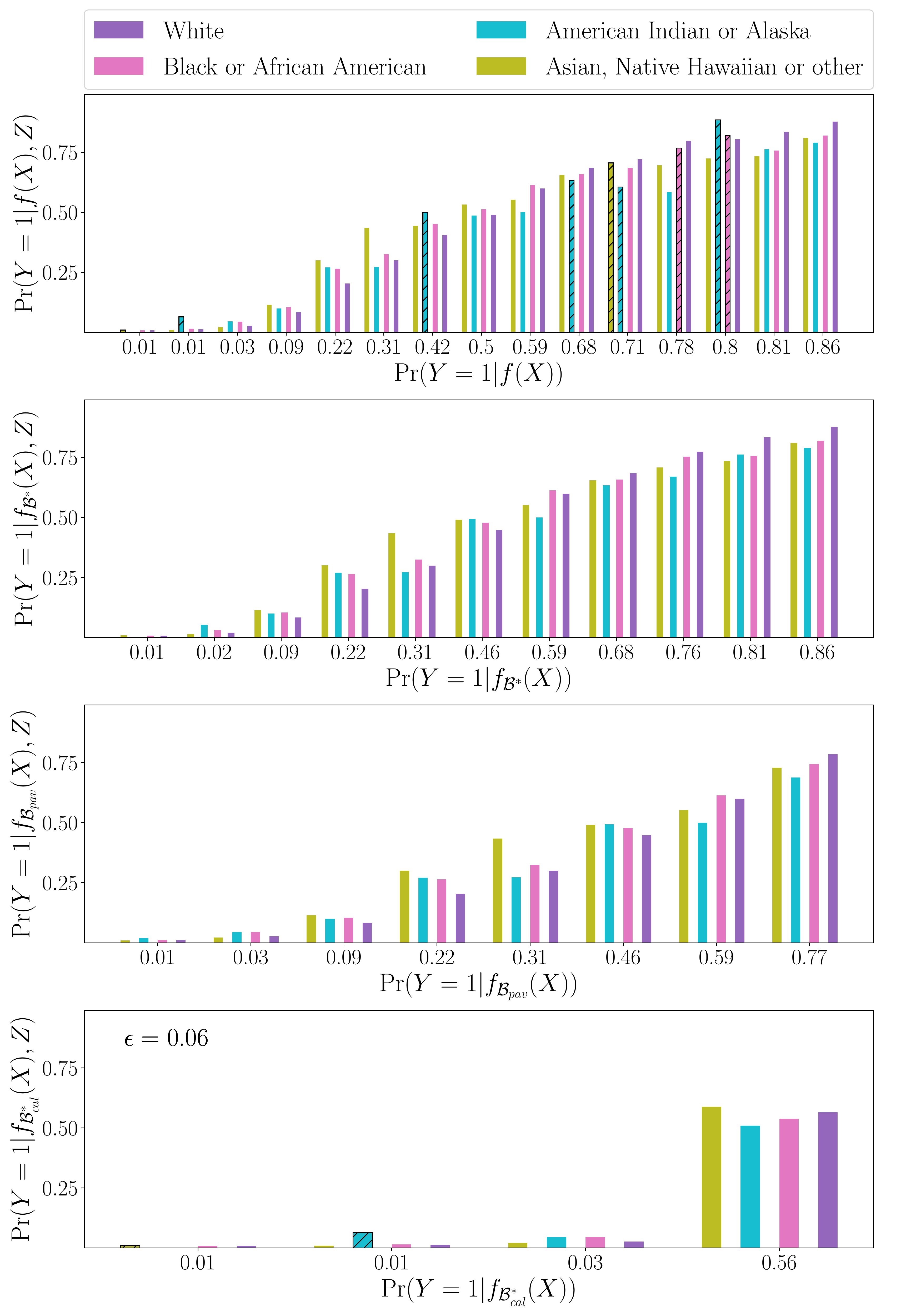}}
\caption{Quality score values $a = P(Y = 1 \given f(X) = a)$ and group conditional quality score values $a_z = P(Y = 1 \given f(X) = a, Z = z)$ of the screening classifier $f$ and the modified classifiers $f_{\Bcal_{\text{pav}}}$, $f_{\Bcal^{*}}$, and $f_{\Bcal^{*}_{\text{cal}}}$ induced by the partitions found by Algorithms~\ref{alg:pav},~\ref{alg:optimal} and \ref{alg:multicalibrated}, respectively. 
In the first and last rows, the hatched bars indicate within-group monotonicity violations and, in the last row, 
we report the smallest $\epsilon$ value such that a within-group $\epsilon$-calibrated classifier $f_{\Bcal^{*}_{\text{cal}}}$ exists.
}
\label{fig:exp_violations_1}\label{app:exp}
\end{figure*}

As expected, $f_{\Bcal^*}$ and $f_{\Bcal_{\text{pav}}}$ are within-group monotone and $f_{{\Bcal}^*}$ is more fine-grained than $f_{\Bcal_{\text{pav}}}$, \ie, $|\Bcal^*| \geq |\Bcal_{\text{pav}}|$.
However, the minimum value of $\epsilon$ such that $f_{\Bcal^{*}_{\text{cal}}}$ exists is not always low 
enough for $f_{\Bcal^{*}_{\text{cal}}}$ to be within-group monotone.
Moreover, we find that, for $f$, $f_{\Bcal^*}$ and $f_{\Bcal_{\text{pav}}}$, the difference among group conditional score values $a_z$ for a given quality score values $a$ is often significant. 
As a result, one should be cautious about comparing candidates from different groups $z$ and instead utilize 
group-dependent decision thresholds~\cite{Wang2022ImprovingSP} to implement more equitable hiring 
practices such as the Rooney rule~\cite{collins2007tackling}, which requires that, when hiring for 
a given position, at least one (or more) candidate(s) from each minority group should be interviewed.
In this context, it is also worth noting that, while using $f_{\Bcal^{*}_{\text{cal}}}$ would mitigate 
such differences, our results show that this would reduce dramatically the granularity of the predictions.
We found qualitatively similar results for different $n$ values.

% \clearpage

\subsection{Additional Experiments On Within-Group $\epsilon$-Calibration}
\label{app:exp-epsilon-calibration}
%
% the smallest $\epsilon \in (0, 1)$ such that 
% $f_{\Bcal^{*}_{\text{cal}}}$ exists
%
% This suggest that, by using employment as a proxy of 
% qualification, one may be perpetuating historical biases against 
% disabled candidates because they have typically had difficulties 
% to find employment, even if qualified.
%
In this section, we investigate how the smallest $\epsilon$ such that a within-group $\epsilon$-calibrated classifier $f_{\Bcal^{*}_{\text{cal}}}$ exists varies against the number of bins $n$ of the screening classifier $f$.
Figure~\ref{fig:epsilon} shows that, for each set of groups $\Zcal$,
$\epsilon$ remains relatively constant with respect to $n$, however, the greater the difference 
across group conditional quality scores $a_z = P(Y = 1 \given f(X) = a, Z = z)$, the greater the 
value of $\epsilon$ that is needed to obtain a within-group $\epsilon$-calibrated classifier, as
one may have perhaps expected.
\begin{figure*}[h]
\centering
\hspace{1cm}\includegraphics[width=0.6\textwidth]{plots/legend_Z.pdf}\\[-0.2ex]
\includegraphics[width=0.35\textwidth]{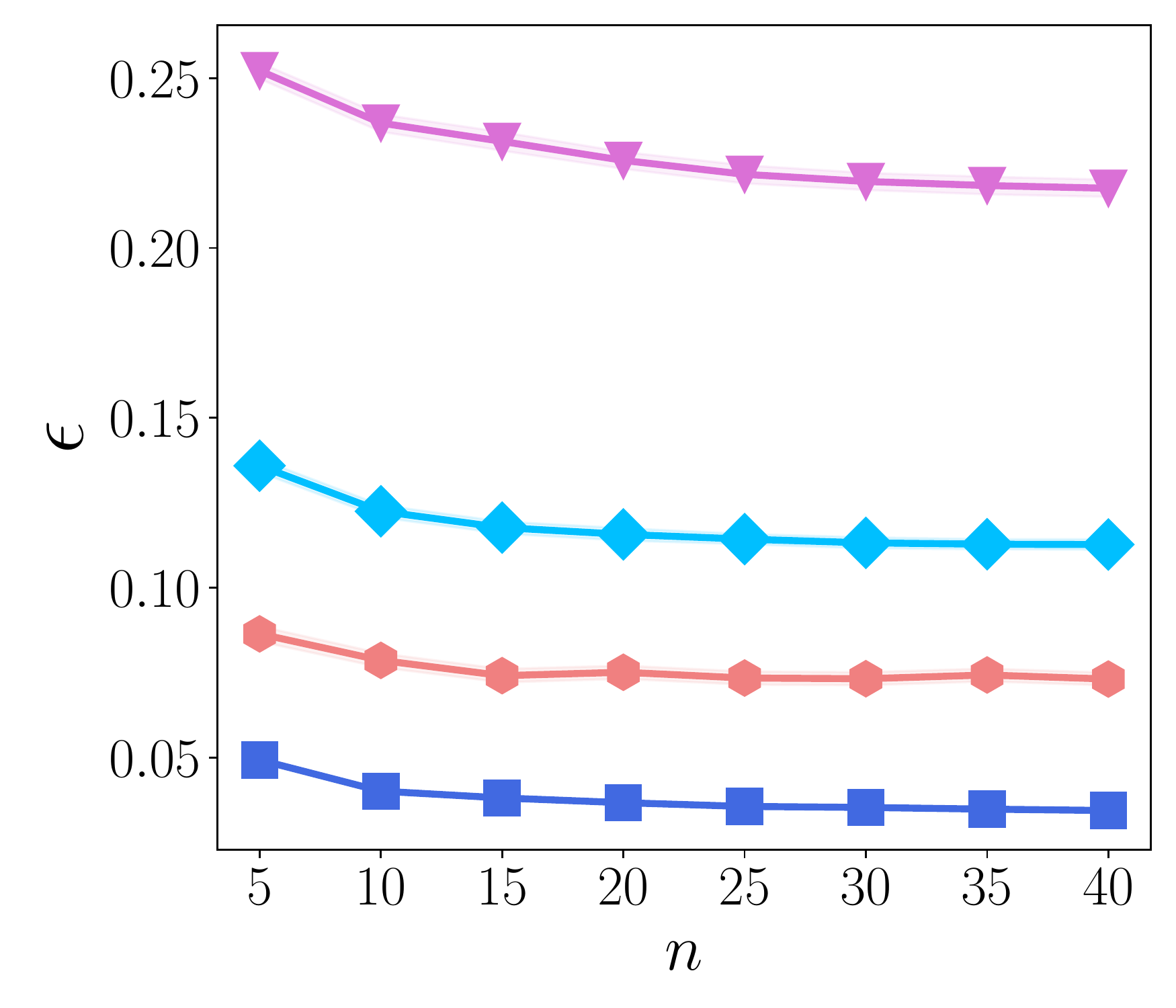}
\caption{Minimum value of $\epsilon$ such that a within-group $\epsilon$-calibrated $f_{\Bcal^{*}_{\text{cal}}}$ exists against 
the number of bins $n$ of the screening classifier $f$.}
\label{fig:epsilon}
\end{figure*}

\subsection{Experimental Results for Other Groups $\Zcal$} \label{app:exp-other-z}
\begin{figure*}[h]
\centering
\includegraphics[width=0.7\textwidth]{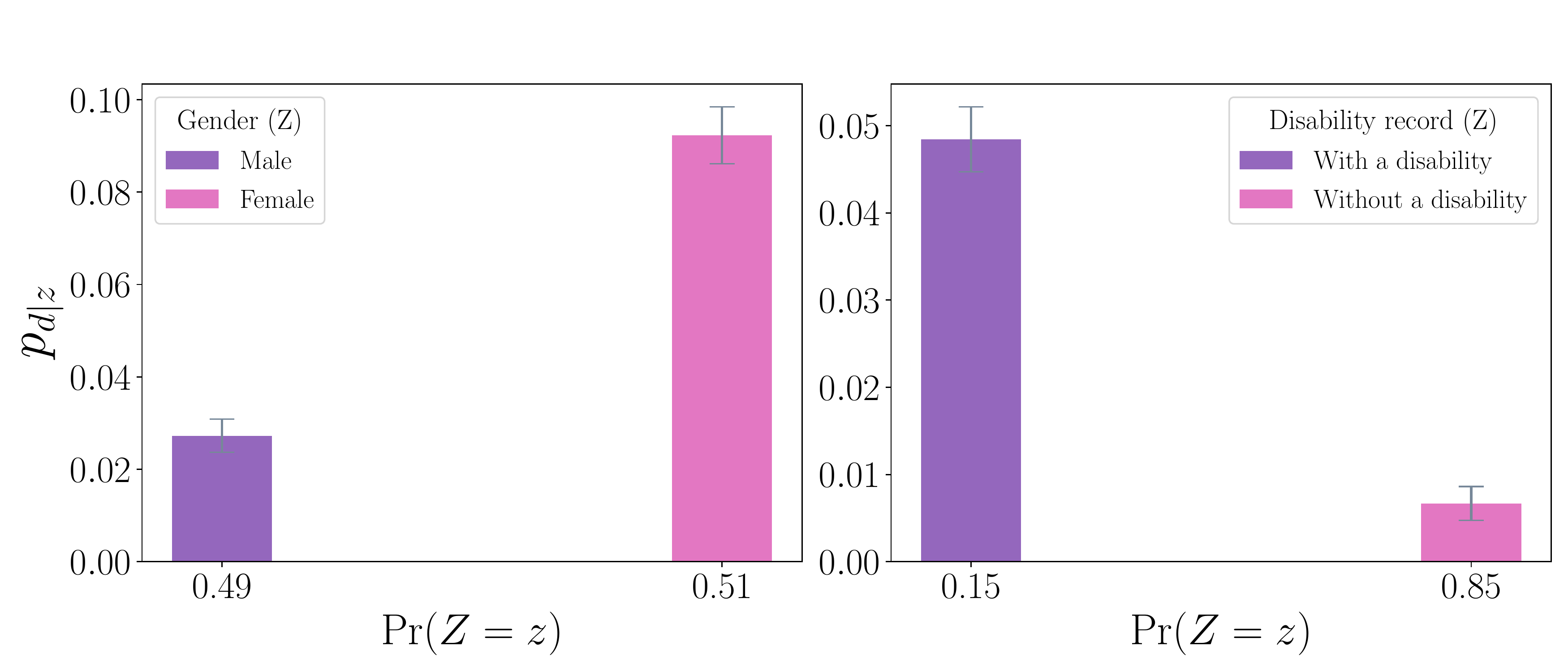}
\caption{Probability $p_{d \given z}$ that an individual from group $z$ may suffer from within-group unfairness against $\Pr(Z = z)$ for $n = 15$.}
\label{fig:exp_discriminations_2}
\end{figure*}

\begin{figure*}[h]
\centering
    \hspace*{0.5cm}\includegraphics[width=0.48\textwidth,left]{plots/legend_n_bins.pdf}\\[-3.5ex]
    \hspace*{0cm}\includegraphics[width=0.48\textwidth,right]{plots/legend.pdf}\\[-2.5ex]
    \subfloat[$|\Bcal|$ vs. n]{ % Gender $(Z)$]{
	\includegraphics[width=0.24\textwidth]{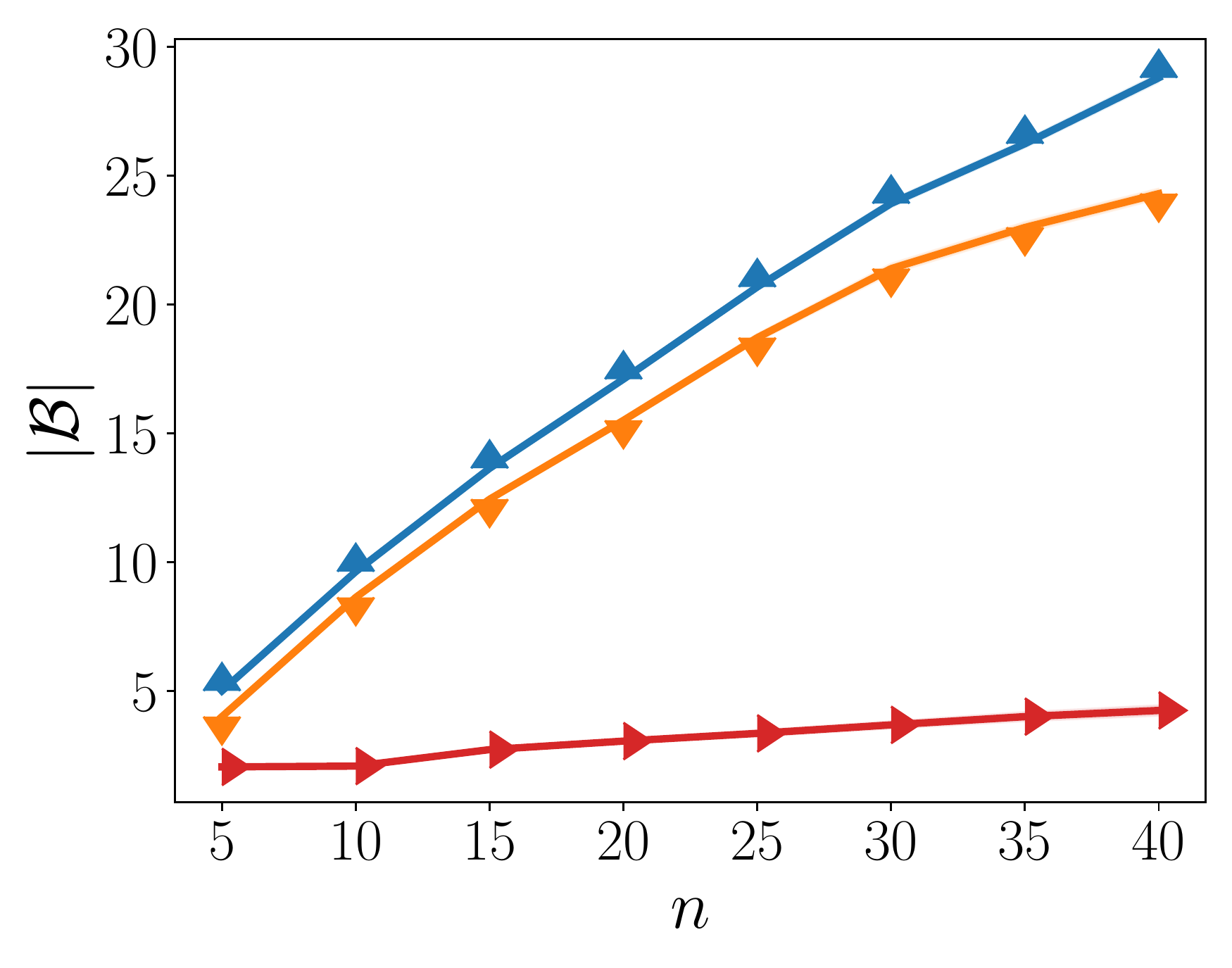}
        \includegraphics[width=0.24\textwidth]{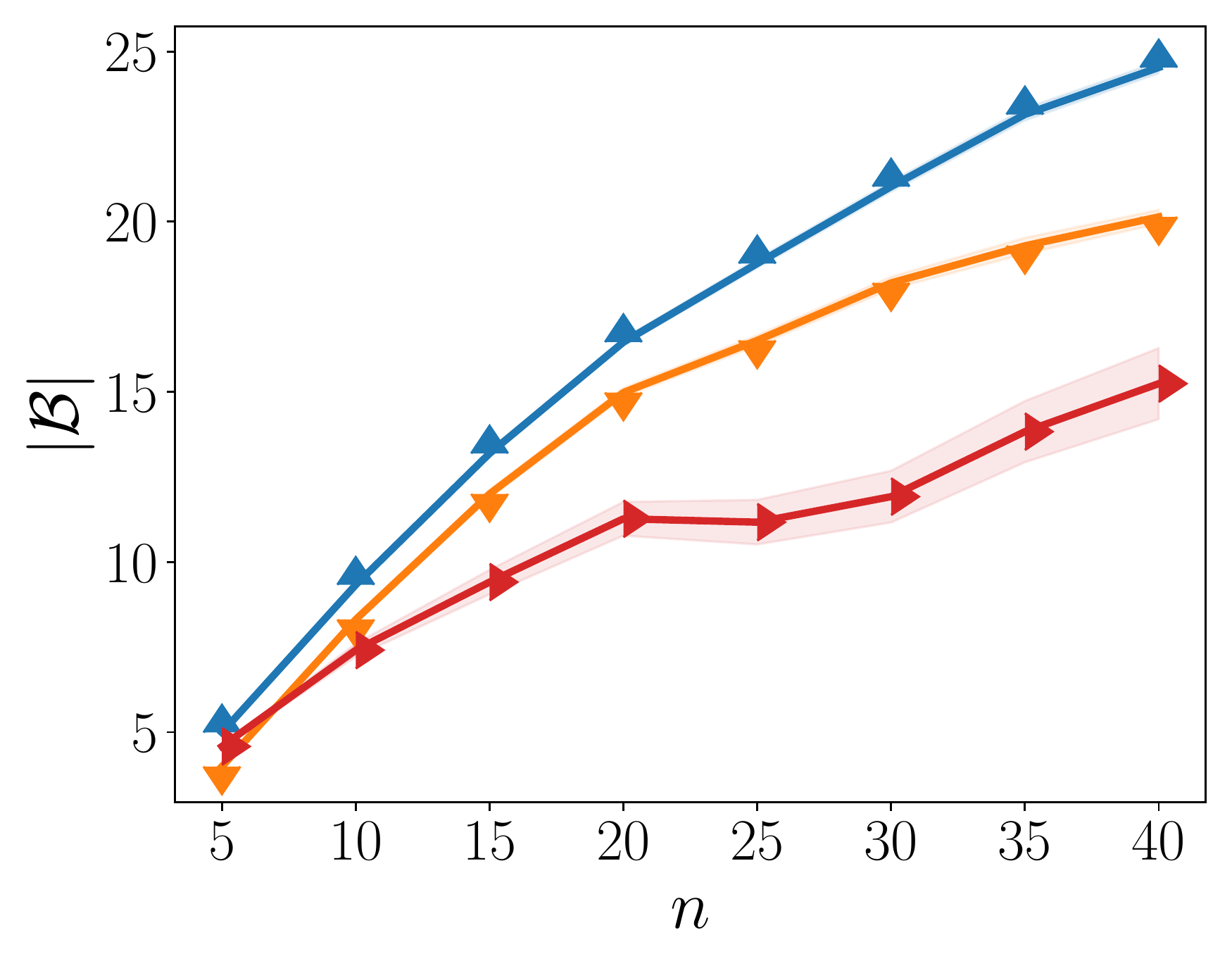}}
    \subfloat[Shortlist size vs. n]{ % Disability Status $(Z)$]{
	\includegraphics[width=0.24\textwidth]{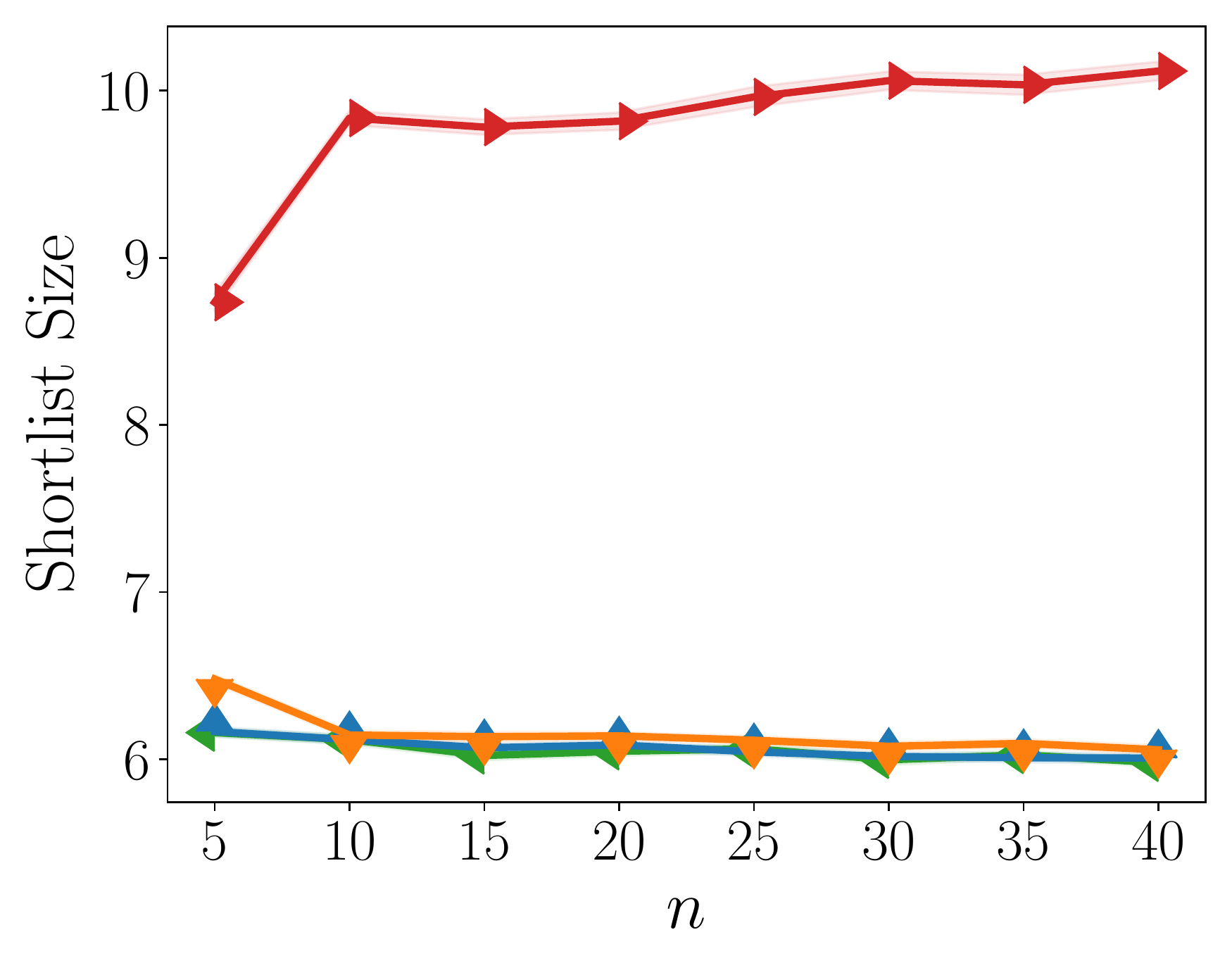}
        \includegraphics[width=0.24\textwidth]{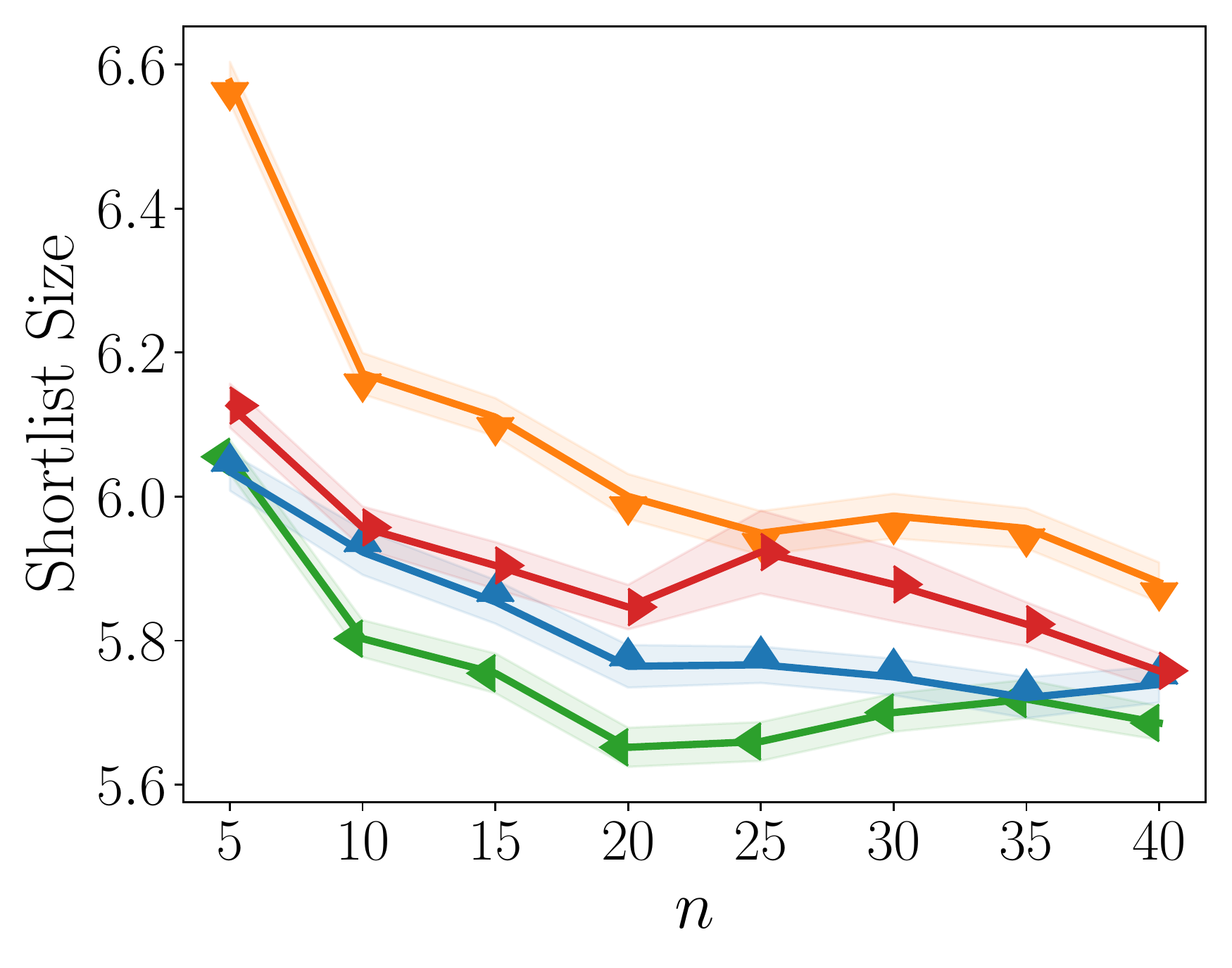}} 
\caption{Quality of the partitions $\Bcal_{\text{pav}}$, $\Bcal^{*}$, and $\Bcal^{*}_{\text{cal}}$ 
returned by Algorithms~\ref{alg:pav},~\ref{alg:optimal} and \ref{alg:multicalibrated}, respectively, 
for screening classifiers $f$ with an increasing number of bins $n$. 
Panel (a) shows the size $|\Bcal|$ of the partitions provided by each algorithm (higher is better). Panel (b) shows 
the size of the shortlists created using the classifiers $f_{\Bcal}$ induced by each partition $\Bcal$ (lower is better).}
\label{fig:partition-short-list-size-gender-disability}
\end{figure*}

\begin{figure*}[h]
\centering
    \subfloat[Gender $(Z)$]{
	\includegraphics[width=0.5\textwidth]{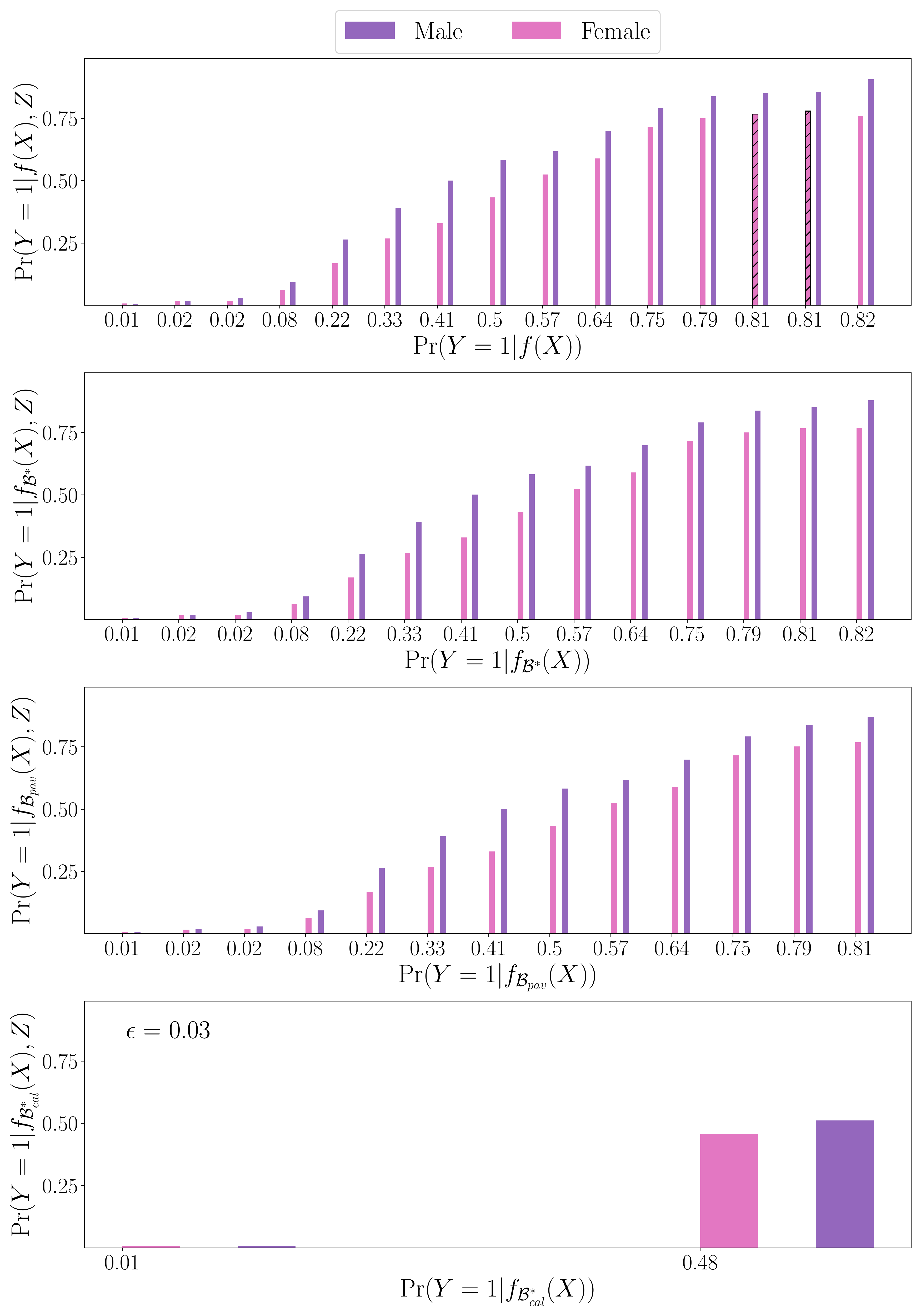}}
    \subfloat[Disability record $(Z)$]{
	\includegraphics[width=0.5\textwidth]{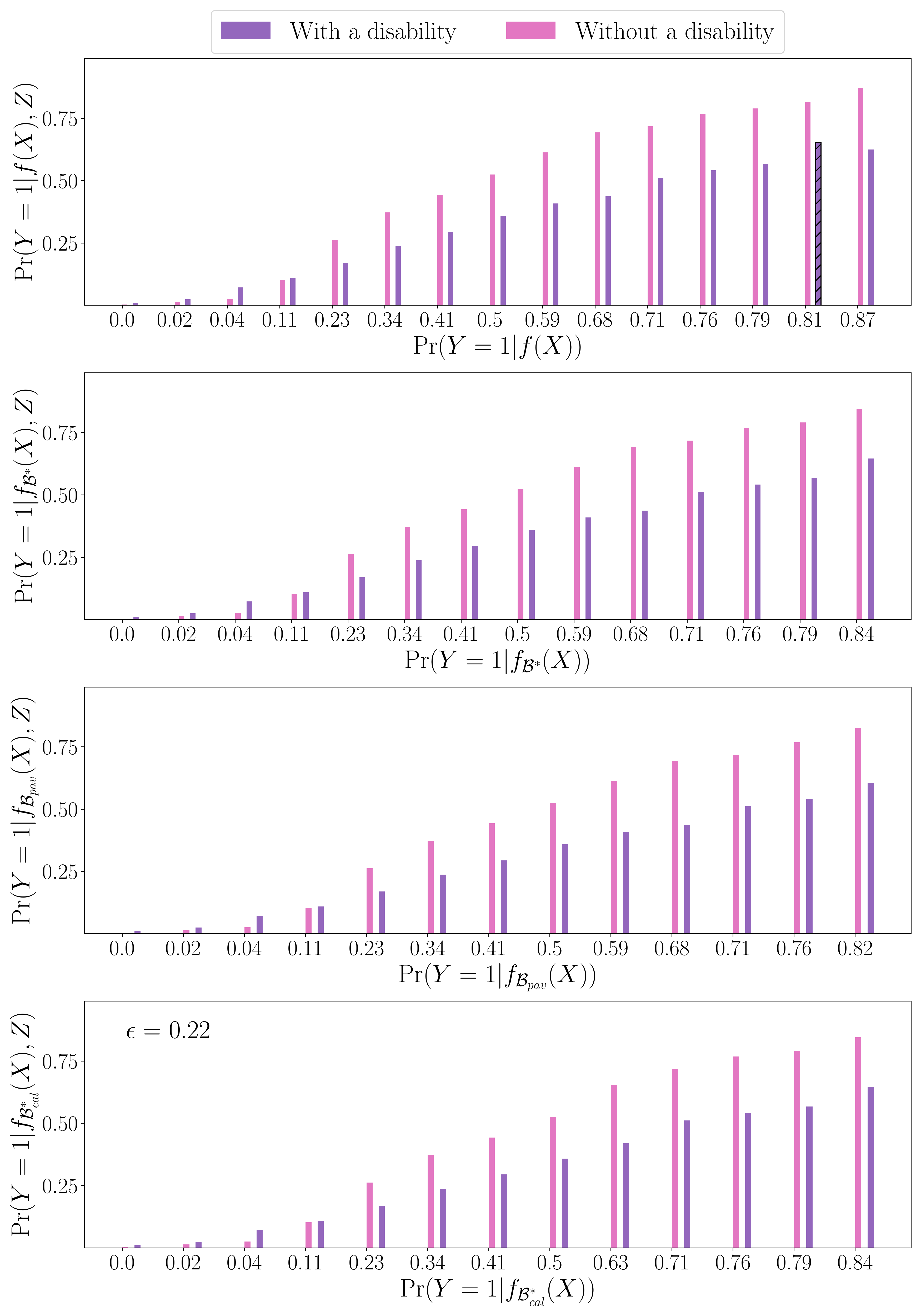}} 
\caption{Quality score values $a = P(Y = 1 \given f(X) = a)$ and group conditional quality score values $a_z = P(Y = 1 \given f(X) = a, Z = z)$ of the screening classifier $f$ and the modified classifiers $f_{\Bcal_{\text{pav}}}$, $f_{\Bcal^{*}}$, and $f_{\Bcal^{*}_{\text{cal}}}$ induced by the partitions found by Algorithms~\ref{alg:pav},~\ref{alg:optimal} and \ref{alg:multicalibrated}, respectively. 
In the first row, the hatched bars indicate within-group monotonicity violations and, in the last row, 
we report the smallest $\epsilon$ value such that a within-group $\epsilon$-calibrated classifier $f_{\Bcal^{*}_{\text{cal}}}$ exists.}
\label{fig:exp_violations_2}
\end{figure*}

\clearpage

\end{document}